%% file: main.tex
\documentclass{article}

\usepackage[main, final]{neurips_2026}

\usepackage{natbib}           %
\usepackage{hyperref}
\usepackage{url}
\usepackage{booktabs}
\usepackage{amsfonts}
\usepackage{amsmath}
\usepackage{amssymb}
\usepackage{amsthm}
\usepackage{nicefrac}
\usepackage{microtype}
\usepackage{xcolor}
\usepackage{graphicx}
\usepackage{multirow}
\usepackage{algorithm}
\usepackage{algpseudocode}
\usepackage{enumitem}
\usepackage{float}            %

\newtheorem{proposition}{Proposition}
\newtheorem{remark}{Remark}
\newtheorem{definition}{Definition}
\newtheorem{corollary}{Corollary}

\newtheorem{lemma}{Lemma}
\title{Unlocking Compositional Generalization in Continual Few-Shot Learning}

\newcommand{\authorcell}[3]{%
  \begin{tabular}[t]{@{}c@{}}\rule{0pt}{22pt}{\bfseries #1}$^{#2}$\\[1pt]
    {\normalfont\footnotesize\ttfamily #3}\end{tabular}}
\author{%
  \makebox[0pt][c]{%
    \authorcell{Phu-Quy Nguyen-Lam}{1,2,4,\dagger}{23122048@student.hcmus.edu.vn}\hspace{1.2em}%
    \authorcell{Phu-Hoa Pham}{1,2,4,\dagger}{23122030@student.hcmus.edu.vn}\hspace{1.2em}%
    \authorcell{Duy-Minh Dao-Sy}{1,2,4}{23122041@student.hcmus.edu.vn}}\\
  \authorcell{Chi-Nguyen Tran}{1,2,4}{23122044@student.hcmus.edu.vn}\hspace{1.2em}%
  \authorcell{Trung-Kiet Huynh}{1,2,4}{23122039@student.hcmus.edu.vn}\\
  \authorcell{Long Tran-Thanh}{3,*}{\normalsize long.tran-thanh@warwick.ac.uk}
}

\begin{document}
\maketitle
\begingroup
\renewcommand{\thefootnote}{}%
\footnotetext{\noindent\hspace*{-1.8em}$^{1}$Faculty of Information Technology, University of Science, Ho Chi Minh City, Vietnam\\
  $^{2}$Vietnam National University, Ho Chi Minh City, Vietnam\\
  $^{3}$Department of Computer Science, University of Warwick, Coventry, United Kingdom\\
  $^{4}$CompactAI\\
  $^{\dagger}$Equal contribution.\quad $^{*}$Corresponding author.}%
\endgroup

\input{sections/abstract}

\input{sections/introduction}

\input{sections/related}

\input{sections/preliminaries}

\input{sections/motivation}

\input{sections/method}

\input{sections/experiments}

\input{sections/conclusion}

\begin{ack}
We thank the Faculty of Information Technology, University of Science, Vietnam National University, Ho Chi Minh City, for providing the computing resources and the research environment that supported this work.
\end{ack}

\bibliographystyle{plainnat}
\bibliography{references}

\newpage
\appendix

\input{sections/appendix}

\newpage
\input{sections/checklist}

\end{document}

%% file: sections/abstract.tex
\begin{abstract}

Object-centric representations promise a key property for few-shot learning: Rather
than treating a scene as a single unit, a model can decompose it into individual
object-level parts that can be matched and compared across different concepts.
In practice, this potential is rarely realized. Continual learners either pool
scenes into one vector before any decomposition exists, or train with part-level
matching objectives that tie representations too closely to seen patterns, leaving
them unable to generalize to truly novel concepts. In this paper, we identify this fundamental structural conflict
and pioneer a new paradigm that strictly decouples representation learning from
compositional inference. Leveraging the inherent patch-level semantic geometry of
self-supervised Vision Transformers (ViTs), our framework employs a dual-phase strategy.
During training, the slot set is scored as a whole against class identity, which
preserves highly generalizable, object-level geometries; at inference, those same
preserved slots are composed part by part to match novel scenes. We demonstrate 
that this paradigm offers dual structural benefits: The frozen backbone naturally 
prevents representation drift, while our lightweight, holistic optimization 
preserves the features' capacity for novel-concept transfer. Extensive experiments 
validate this approach, achieving state-of-the-art unseen-concept generalization 
and competitive accuracy with low forgetting on standard continual learning benchmarks.
Code is available at \url{https://github.com/NLPQuy/COMPOSE}.
\end{abstract}

%% file: sections/introduction.tex
\section{Introduction}
\label{sec:intro}

Visual recognition in the real world is rarely a closed-book exam. A learner must continuously acquire new concepts from limited examples, with minimal ability to revisit the past, and crucially, without degrading what was previously learned. In the Continual Few-Shot Learning (CFSL) setting, this challenge is pushed to its extreme: The model must generalize to novel concepts that bear no resemblance to its training distribution, relying on only a handful of examples.

Object-centric representation is a natural candidate for this regime. If a scene can be decomposed into object-level representations, a model no longer needs to have seen a whole concept to recognize it; recognizing its parts suffices. In principle, a model built on such compositional structure should generalize robustly to novel concepts ~\citep{compositional_generalization}, and consequently, resist catastrophic forgetting ~\citep{cfst}. This structural potential, however, remains largely unrealized. Existing continual learners ~\citep{simpleshot, iCARL, ranpac, ease, cofima} discard object structure entirely, collapsing each scene into a single \emph{monolithic global embedding} before any decomposition exists, and fail precisely on novel concepts. Recent attempts to introduce slot attention~\citep{slot_attention} into continual learning similarly fail to generalize to truly novel concepts~\citep{compslot}. We trace this failure to two intertwining problems: (i) how object representations are extracted; and (ii) how they are optimized.

\paragraph{Self-supervised geometry as a structural foundation.}
A promising direction is to build recognition around individual objects within a scene, rather than a single descriptor of the whole image. If a scene can be decomposed into object-level representations general enough to describe any concept, then generalizing to novel ones follows naturally. Self-supervised vision transformers are what make this decomposition structurally feasible~\citep{byol_paper, dinov1}. Through local self-distillation objectives, models such as the DINO series~\citep{dinov1,dinov2, dinov3} and iBOT~\citep{ibot} learn patch-level geometry in which patches from the same object naturally cluster together in feature space, a property that slot attention~\citep{slot_attention, dinosaur} can directly exploit to segment a scene into coherent object-level slots. Vision-language models such as CLIP~\citep{clip} and OpenCLIP~\citep{openclip}, on the other hand, use contrastive alignment with language to build ground patch features in coherent regions. However, they do not reproduce the local geometry that self-distillation induces, and their performance settles at a lower ceiling that even a larger language-aligned model trained on far more images cannot lift (Appendix~\ref{sec:patch_source}). In addition, supervised Vision Transformers~\citep{supervised_vit}, are trained with global cross-entropy for inter-class discrimination and never develop this local geometric structure, leaving their patch tokens too noisy for reliable object decomposition. Thus, it is essential to achieve high-quality patch-level geometry, which is the aim of our work.

\textbf{Our contributions}. The observations above point to a single underlying tension: Object-level decomposition requires the right geometric foundation to extract meaningful slots, and the right training objective to keep them general. In this paper we develop \textbf{COMPOSE}, a dual-phase framework built on a single insight that \emph{training and inference require fundamentally different objectives}. 
In particular, using compositional matching during training corrupts the very representations needed to generalize at inference. COMPOSE resolves this tension without degrading decomposition quality by a strict two-phase separation. In both phases an image stays the same set of slots, and we only change the way that set is read. During training the slots are scored \emph{as a set} against a class prototype, a router-weighted sum over them rather than the single pre-decomposition vector criticised above, so no gradient ever drives an individual slot toward a training-specific match and the object-level structure inherited from the backbone survives. At inference the same slots are matched part by part against novel scenes, where part-level comparison is free because no gradient flows. Keeping the slots and merely summing them is already worth $+11.53$\,pp on novel concepts over a meta-learner on the CLS token with no slots (Section~\ref{sec:main_results}); matching them part by part on top adds more again. Because representations are never forced toward training-specific patterns, they transfer naturally to unseen concepts, and a model that generalizes to unseen concepts may have learned representations general enough not to forget.

\noindent
Against this background, the contributions of this paper are three-fold:
\begin{itemize}
  \item \textbf{Rethinking Slot Representations.} Across seven patch sources we identify what object-centric decomposition requires: not self-supervision as a category, but semantic content in the patch tokens and the block they are read at ($\rho = 1.000$ between COCO slot purity and every CFST metric). We also replace recurrent GRU states with attention-weighted aggregates that improve slot purity.
  \item \textbf{Diagnosing and Escaping the Compositional Trap.} We analytically establish that part-level matching imposes divergent slot-specific gradients, structurally driving the over-specialization that degrades novel-concept transfer. We escape this trap via a strict two-phase separation: the slot set is scored as a whole while gradients flow, and matched part by part once they do not. A cross-correlation penalty prevents slot collapse without the embedding conflicts of variance-hinge alternatives, and a train~$\times$~inference cross-ablation isolates the damage to training.
  \item \textbf{Novel-Concept Transfer and Continual Robustness.} Our
  framework achieves state-of-the-art performance on the unseen-concept split of CGQA ($93.46\%$ noc). On standard
  continual learning benchmarks it is competitive with dedicated methods, leading on ImageNet-A and forgetting far
  less than Comp-FSCIL on FSCIL, and it extends to domain-incremental DN4IL at minimal training overhead. A closed-form rule removes the only hand-set constant, the blend coefficient, and matches or improves the fixed value.
\end{itemize}

%% file: sections/related.tex
\section{Related Work}
\label{sec:related}

\paragraph{Pre-trained global models for continual learning.}
A dominant family of class-incremental methods exploits frozen or lightly adapted vision transformers. RanPAC~\citep{ranpac} inserts random projections, EASE~\citep{ease} expands representation dimensionality via task-specific adapters, and CoFiMA~\citep{cofima} fine-tunes the backbone using Fisher-weighted parameter merging. All are efficient on closed-world benchmarks but share one vulnerability: every scene becomes a monolithic global vector, so concepts with zero overlap with training hit a hard ceiling. COMPOSE avoids this by construction, segmenting scenes into object-centric slots.

\paragraph{Self-supervised geometry and representation pollution.}
Object-centric routing, popularized by Slot Attention~\citep{slot_attention} and scaled by DINOSAUR~\citep{dinosaur}, breaks that global coupling by grouping patches into $K$ slots, but its success is highly sensitive to the underlying feature geometry. Self-supervised models like DINOv2~\citep{dinov2} and iBOT~\citep{ibot} induce patch-level semantic clustering. Vision-language models such as CLIP~\citep{clip} and its open-source variants~\citep{openclip} develop patch-level structure that is
sufficient for slot routing but that plateaus below self-distilled geometry, a ceiling scale does not lift. Supervised ViTs~\citep{supervised_vit} lack this local structure, and so do label-free objectives that reconstruct \emph{pixels}~\citep{mae} or supervise \emph{masks}~\citep{sam}, so the operative property is semantic content in the patch tokens rather than the absence of labels (Appendix~\ref{sec:patch_source}). Prior slot-centric methods nonetheless treat the backbone as an interchangeable black box; we make the requirement explicit and measurable, replacing the recurrent state with a direct attention-weighted aggregate that restores slot purity and preserves the
geometries necessary for downstream matching.

\paragraph{Compositional generalization in continual learning.}
The CFST benchmark~\citep{cfst} formalizes the evaluation of compositionality in continual learning, particularly penalizing models that fail on novel concepts. Compositional FSCIL~\citep{Comp-FSCIL} and CZSL ~\citep{cczsl} attempt to model primitive reuse, but either operate without explicit spatial decomposition or rely on closed-world assumptions. %

%% file: sections/preliminaries.tex
\section{Preliminaries}
\label{sec:prelim}

\subsection{Compositional Few-Shot Transfer (CFST)}
\label{sec:cfst}

CFST~\citep{cfst} evaluates compositional generalization in a continual
few-shot setting. A model trains sequentially on $T$ sessions
$\{\mathcal{T}_1,\ldots,\mathcal{T}_T\}$, each introducing $C_t$ novel classes
over a shared concept pool $\mathcal{C}$. After continual training, the feature
extractor is \emph{frozen} and evaluated on five disjoint few-shot splits.
Let $\mathcal{Y}_\text{tr}$ denote training combinations and $\mathcal{C}$ the
training concept pool.

\begin{itemize}
    \item \textbf{Systematicity (sys):} Novel concept combinations not seen
    in $\mathcal{Y}_\text{tr}$, with the same concept count as training.
    All concepts $\mathcal{C}_\text{sys} \subseteq \mathcal{C}$ have been seen.

    \item \textbf{Productivity (pro):} Novel combinations requiring \emph{more}
    concepts than training. Tests whether knowledge scales to more complex
    compositions.

    \item \textbf{Substitutivity (sub):} Same concept distribution as sys but
    different conditional appearance, testing attribute-level generalization
    (e.g.\ a green shirt given red shirts and green grass in training).

    \item \textbf{Non-novel (non):} Few-shot evaluation on \emph{trained
    combinations}, an upper reference on retained knowledge.

    \item \textbf{Non-compositional (noc):} Few-shot evaluation on combinations of
    \emph{entirely unseen concepts}, $\mathcal{C}_\text{noc} \cap \mathcal{C} = \emptyset$,
    measuring transferability independent of compositional priors. \textit{This is our focal metric.}
\end{itemize}
The primary metric is $H_a$, the harmonic mean over all available split
accuracies, which penalizes models that improve some axes at the expense of others. In addition to the five few-shot splits, CFST reports $A_\text{con}$, the average test accuracy across all $T$ continual tasks after the final session, serving as a measure of catastrophic forgetting.

\textbf{Datasets.} CGQA (Compositional GQA; 100 classes, 10 sessions of 10 classes) uses synthetic composite images: two objects randomly placed in a $2{\times}2$ grid. Clear spatial structure makes slot--object alignment reliable. COBJ (Compositional Objects365; 30 classes, 3 sessions) uses natural photographs with complex scene statistics that make slot decomposition noisier.

\subsection{Slot Attention}
\label{sec:slot_attn}
Given image $x$, a frozen Vision Transformer (ViT) extracts $N$ patch tokens
$F = \{F_n\}_{n=1}^{N} \subset \mathbb{R}^{D}$ from a fixed intermediate
block. Slot Attention~\citep{slot_attention} maintains $K$ slot vectors
$\{s_k\}_{k=1}^{K} \subset \mathbb{R}^{D}$, initialized from a learned
Gaussian prior. Each slot attends to patch features via competitive softmax:
\begin{equation}
    \alpha_{kn} = \frac{
        \exp\!\bigl(\tfrac{1}{\sqrt{D}}\langle W_q s_k, W_k F_n\rangle\bigr)
    }{
        \sum_{k'}\exp\!\bigl(\tfrac{1}{\sqrt{D}}\langle W_q s_{k'},
        W_k F_n\rangle\bigr)
    },
    \label{eq:slot_attn}
\end{equation}
where $\alpha_{kn}$ is the normalized attention weight of slot $k$ on patch $n$. The competitive softmax makes each slot specialize to a distinct spatial region, so $\{\alpha_{kn}\}_{n=1}^{N}$ acts as a soft spatial mask over the region slot $k$ binds to. When the patch features carry coherent object-level semantics, these masks align with object boundaries. The slot is then updated via a GRU using a patch-wise weighted mean of value-projected tokens: $s_k \leftarrow \text{GRU}\!\bigl(s_k,\, \textstyle\sum_n \hat\alpha_{kn} W_v F_n\bigr)$, where $\hat\alpha_{kn} = \alpha_{kn} / \sum_{n'} \alpha_{kn'}$ renormalizes the slot-wise weights across spatial tokens, following the original Slot Attention formulation~\citep{slot_attention}. We use the publicly released DINOSAUR checkpoint~\citep{dinosaur} trained on COCO 2017.

\begin{remark}[Self-supervised geometry enables slot decomposition]
\label{rem:selfsup_geometry}
The competitive softmax in Eq.~\eqref{eq:slot_attn} can only identify coherent
object regions when nearby patches carry consistent semantic content.
Self-supervised ViTs (DINOv2, DINOv1, iBOT) learn such patch-level semantic consistency
through local self-distillation objectives, enabling reliable slot decomposition.
Supervised ViTs, trained with global cross-entropy, do not---leading to noisy
slot assignments and poor slot purity. %
\end{remark}

\subsection{Problem Setup}
\label{sec:setup}
The model observes $T$ sessions sequentially. At session $t$, it accesses
support set $\mathcal{S}_t$ with $C_t$ classes and $M$ images per class; at
meta-test time, $N$-way $K$-shot episodes require classifying $Q$ query images
given $N$ classes with $K$ labeled examples each. No prior-session images are
available at test time. The goal is to maximize $H_a$ over all CFST evaluation
splits, with noc as the primary focal metric.

%% file: sections/motivation.tex
\section{Rethinking Slot-Centric Continual Learning}
\label{sec:rethinking}

To realize the structural potential of object-centric learning in CFSL, we must identify why existing slot-based methods collapse when confronted with novel concepts. Our analysis reveals that this failure is not single-sourced but a result of two coupled flaws. In particular, it is due to representation pollution at the architectural level, and the compositional trap at the optimization level.

\subsection{Representation Pollution: The Necessity of Preserved Geometry}
\label{subsec:pollution}
For slot routing to isolate objects without supervision, the feature space must possess patch-level semantic geometry. That is, patches of the same object must naturally cluster together. As established in Section~\ref{sec:related}, supervised ViTs lack this property, as do pixel-reconstruction and mask-supervised objectives, so continual slot decomposition requires semantic content in the patch tokens at the block being read (Appendix~\ref{sec:patch_source}).

Correct geometry is insufficient if the slot architecture pollutes it. Standard slot attention~\citep{slot_attention, dinosaur} takes the GRU hidden state $s_k$ as the slot representation, but $s_k$ encodes routing dynamics, intertwining spatial logic with visual content and corrupting the backbone's semantic structure.

To preserve the backbone's semantic geometry, we discard $s_k$ entirely and instead compute an attention-weighted aggregate directly from the frozen backbone features:
\begin{equation}
\label{eq:phi}
\tilde{\phi}_k = \text{L2-norm}\!\left(\sum_{n=1}^{N} \alpha_{kn} F_n\right) \in \mathcal{S}^{D-1},
\end{equation}
where $F_n$ are ViT patch tokens and $\alpha_{kn}$ are slot $k$'s attention weights as defined in Eq.~\eqref{eq:slot_attn} (the L2-normalization makes the result invariant to whether the weights are taken raw or patch-wise renormalized). Unlike the latent GRU state $s_k$, $\tilde{\phi}_k$ is a direct attention-weighted readout of the patch features. By bypassing the recurrent bottleneck, it resides strictly within the backbone's pristine semantic space, yielding a substantial gain in slot purity on COCO val (Appendix~\ref{app:purity_extended}).

\subsection{The Compositional Optimization Trap}
\label{subsec:trap}
Even with preserved slot representations, training the model requires an objective function. When comparing a query scene to a support class, the foundational intuition in compositional learning is to match them part-by-part. Mathematically, this is typically realized by finding the nearest support slot for each query slot, a set-to-set nearest-neighbor assignment commonly formalized as the Chamfer distance~\citep{chamfer_distance}. We define the reliance on this mechanism during training as the \emph{compositional trap}: optimizing directly for part-level matching paradoxically destroys representational generality.
 
Consider a query scene with extracted slots $\{z_k\}$ and a support class prototype composed of support slots $\{z_{k'}^{(c)}\}$. A part-level objective computes the best local match $k^*(k)$ for each query slot $k$. By analyzing the gradient of this operation (formal statement in Appendix~\ref{app:gradient_divergence}; dynamical consequences in Appendix~\ref{app:dynamics}), we establish the following structural property, illustrated in Figure~\ref{fig:trap}:

\begin{figure}[t]
    \centering
    \includegraphics[width=\textwidth]{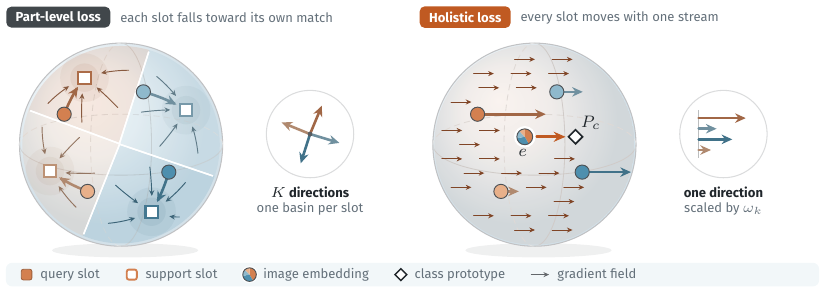}
    \caption{\textbf{The compositional optimization trap.} Gradient field each training objective induces on slot embeddings. \emph{Left:} a part-level (Chamfer) loss pulls every point toward its nearest support slot, so the sphere splits into $K$ basins and each query slot falls toward its own match. \emph{Right:} a holistic loss sends all slots one shared direction, toward the class prototype $P_c$, scaled only by the router weight $\omega_k$ (exact at the projections $\check{y}_k$ before normalization; Proposition~\ref{prop:gradient}). Point positions and weights are schematic.}
    \label{fig:trap}
\end{figure}
 
\textbf{Takeaway 1 (Gradient Directional Variance).} \emph{Under any part-level nearest-neighbor objective, the gradient direction applied to slot $z_k$ is governed by slot-specific factors, namely: the projection anchor at $z_k$ itself and its matched target $z_{k^*(k)}^{(c)}$. Consequently, different slots receive structurally divergent gradient directions.}
 
Each slot then optimizes its own local similarity independently, and over many episodes slots over-specialize to recurring local patterns rather than reusable whole-object identities. On a novel concept ($noc$) they fail to find meaningful matches, their features being too narrowly tied to the training distribution.

In contrast, a holistic objective applied to the averaged slot embedding propagates a shared gradient direction across all slots. Crucially, because the router computes importance weights $\omega_k$ from the frozen backbone aggregates $\tilde{\phi}_k$ via a disjoint parameter branch (Section~\ref{subsec:phase1}), these weights act as constants with respect to the projection head $W_2$. Consequently, the gradient on each projected slot $\check{y}_k = W_2 \tilde{\phi}_k$ differs only by the non-negative scalar $\omega_k$. This aligns the entire set toward a shared class identity without forcing local, part-level specialization.

\subsection{The Role of Redundancy Reduction in Shaping the Embedding Space}
\label{subsec:redundancy}
A purely holistic objective introduces the complementary failure mode of \emph{slot collapse}, where slots converge to redundant directions~\citep{barlow_twins,vicreg,slot_attention,dinosaur}. Redundancy reduction is required, but not all decorrelation objectives are compatible with prototype learning. VICReg~\cite{vicreg} enforces a per-dimension variance-hinge, which conflicts with the prototype CE objective: neural collapse~\cite{neural_collapse} drives within-class variability to zero, triggering the hinge.

We find a cross-correlation penalty (inspired by Barlow Twins~\citep{barlow_twins}) harmonizes with prototype learning. Let $C \in [-1, 1]^{D \times D}$ be the correlation matrix of the raw projections $\check{y}_k = W_2 \tilde\phi_k$, standardized per-dimension and computed globally across all $B \times K$ slots in the episode batch. Penalizing only its off-diagonal entries,
\begin{equation}
\mathcal{L}_{decorr} = \frac{\lambda_d}{D} \|C - I\|_F^2,
\end{equation}
encourages distributed feature utilization without rigid per-dimension variance constraints. Appendix~\ref{app:operator_distinction} makes this precise: CE is invariant under orthogonal rebasings of $W_2$, so every CE-optimum rotates at no CE cost to one with zero CC penalty; spectral whitening has no such optimum on L2-normalized embeddings, and the variance hinge, switched on by CE's own contraction, then raises CE. Empirical comparison is in Appendix~\ref{app:decor_mechanism}.

%% file: sections/method.tex
\section{The COMPOSE Framework}
\label{sec:compose}

We now turn to the description of COMPOSE, which strictly separates representation learning (Phase I: updating the router and projection head alongside a 20-exemplar/class replay buffer) from backprop-free compositional inference (Phase II). By freezing the backbone and slot attention, we prevent feature drift and strictly bound catastrophic forgetting. Figure ~\ref{fig:pipeline} illustrates the full pipeline.

\begin{figure*}[t]
    \centering
    \includegraphics[width=\textwidth]{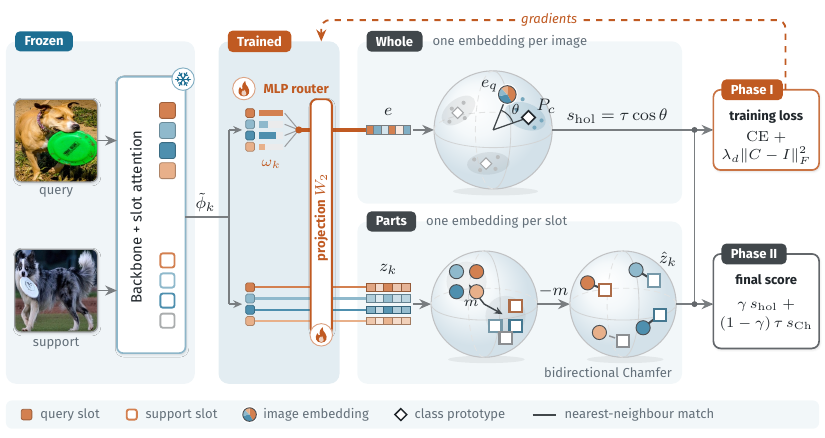}
    \caption{\textbf{One slot set, read two ways.} A frozen ViT with slot attention turns each image into $K$ slots $\tilde{\phi}_k$ (four shown). The trainable MLP router weights the slots by $\omega_k$, and a shared projection $W_2$ maps either their weighted sum to a single image embedding $e$ (\emph{Whole}) or each slot to its own embedding $z_k$ (\emph{Parts}). \textbf{Phase~I} trains the router and projection through the Whole branch only, with cross-entropy on $\cos(e_q, P_c)$ and a cross-correlation penalty; no part-level matching is used. \textbf{Phase~II} is gradient-free: the holistic score is blended with a bidirectional Chamfer score over the per-image centered slots $\hat{z}_k$. Point positions and weights are schematic; the inputs are a query and a support image of one COBJ class.}
    \label{fig:pipeline}
\end{figure*}
\subsection{Phase I: Holistic Training}
\label{subsec:phase1}

\paragraph{Slot Extraction.}
Given an input image, we extract patch features $F_n \in \mathbb{R}^{D}$
using a frozen self-supervised Vision Transformer backbone. A frozen slot
attention module computes attention weights $\alpha_{kn}$ for $K$ slots.
To preserve the backbone's semantic geometry (as analysed in
Section~\ref{sec:rethinking}), we bypass recurrent states and directly
compute the attention-weighted aggregate as defined in Eq.~\eqref{eq:phi}.

\paragraph{MLP Router.}
With the backbone and slot attention frozen, the trainable component in
Phase~I comprises a lightweight \emph{router} and a \emph{projection head}.
We frame slot aggregation as a multiple instance learning (MIL) problem:
each image is a bag of $K$ slots with only class-level supervision. Under this
view, the router instantiates the attention-based MIL pooling of
\citet{mil_attention}, which provides permutation-invariant,
interpretable aggregation under weak supervision. Concretely, the router
assigns a scalar importance weight to each slot via a one-hidden-layer MLP:
\begin{equation}
    \omega_k = \frac{\exp\!\left(v^\top \sigma(W_1\tilde{\phi}_k)\right)}
                    {\sum_{k'=1}^{K}
                     \exp\!\left(v^\top \sigma(W_1\tilde{\phi}_{k'})\right)},
    \label{eq:router}
\end{equation}
where $W_1 \in \mathbb{R}^{h \times D}$, $v \in \mathbb{R}^{h}$, and $\sigma$
is ReLU. The router operates directly on $\tilde{\phi}_k \in \mathcal{S}^{D-1}$,
the L2-normalized slot aggregates that reside in the backbone's semantic space,
providing a stable and geometry-preserving signal for slot selection at both
training and inference.

The projection head $W_2 \in \mathbb{R}^{D \times D}$ (initialised as the
identity) is shared across two uses. Applied per-slot independently, it
produces unit-norm slot embeddings:
\begin{equation}
    z_k = \mathrm{L2\text{-}norm}\!\left(W_2\,\tilde{\phi}_k\right)
          \in \mathcal{S}^{D-1}.
    \label{eq:slot_proj}
\end{equation}
Applied to the router-weighted aggregate, it produces the holistic
embedding:
\begin{equation}
    e = \mathrm{L2\text{-}norm}\!\left(
            W_2 \sum_{k=1}^{K} \omega_k\, \tilde{\phi}_k
        \right) \in \mathcal{S}^{D-1}.
    \label{eq:holistic}
\end{equation}
Shared weights in Eq.~\eqref{eq:slot_proj} preserve permutation equivariance, and
the linearity of $W_2$ makes $e$ an aggregate of the per-slot projections.

\paragraph{Training Objective.}
The router and projection head are jointly optimised using a holistic
prototype cross-entropy loss coupled with a cross-correlation redundancy
reduction penalty:
\begin{equation}
    \mathcal{L} = \mathrm{CE}\!\left(\cos(e_q, P_c) \cdot \tau,\; y\right)
                + \frac{\lambda_d}{D} \|C - I\|_F^2,
\end{equation}
where $P_c = \mathrm{L2\text{-}norm}\!\left(\frac{1}{|S_c|}\sum_{i \in S_c} e_i\right)$ is the prototype dynamically computed from the active support set $S_c$ (recomputed per episode during few-shot inference, whereas standard CIL uses the stored memory $\mathcal{M}$), $\tau$ is a learned temperature, and $C$ is the batch-wide cross-correlation of the unnormalized slot projections
$\check{y}_k = W_2\tilde{\phi}_k \in \mathbb{R}^D$ and the penalty is averaged over the $D$ embedding dimensions (Appendix~\ref{app:operator_distinction}). \textit{Crucially, no part-level Chamfer matching is applied during Phase~I.}

\subsection{Phase II: Compositional Inference}
\label{subsec:phase2}
With the router trained holistically, Phase~II exploits the preserved slot
geometry to compose novel scenes. All network weights are frozen; no gradient
reaches any component.

\paragraph{Per-image Slot Centering.}
Before matching, we convert absolute slot features into relative structural
contrasts within each image by subtracting the per-image slot mean:
\begin{equation}
    \hat{z}^{(i)}_k = \mathrm{L2\text{-}norm}\!\left(
        z^{(i)}_k - \frac{1}{K}\sum_{k'} z^{(i)}_{k'}
    \right).
    \label{eq:centering}
\end{equation}
This removes shared background and illumination biases while preserving the
inter-slot differences that carry the scene's compositional structure.

\paragraph{Bidirectional Matching.}
The holistic score is obtained directly from the Phase~I embeddings of
Eq.~\eqref{eq:holistic}, which sum \emph{over the slots} of an image rather than pooling it before
decomposition, and are read before centering:
\begin{equation}
    s_\text{hol}(q,c) = \tau \cdot \cos\!\left(e_q,\; P_c\right),
    \label{eq:holistic_score}
\end{equation}
where $P_c = \mathrm{L2\text{-}norm}\!\left(\frac{1}{|S_c|}\sum_{i \in S_c}
e_i\right)$ is the class prototype in embedding space.  The compositional score
exploits part-level composability via bidirectional Chamfer matching on the
centered per-slot projections:
\begin{equation}
    s_\text{Ch}(q,c) =
        \frac{1}{|\mathcal{F}_q|}\sum_{k\in\mathcal{F}_q}
            \max_{k'\in\mathcal{S}_c}\cos(\hat{z}^{(q)}_k, \hat{z}^{(c)}_{k'})
        +\frac{1}{|\mathcal{S}_c|}\sum_{k'\in\mathcal{S}_c}
            \max_{k\in\mathcal{F}_q}\cos(\hat{z}^{(c)}_{k'}, \hat{z}^{(q)}_k),
    \label{eq:chamfer_inf}
\end{equation}
where $\mathcal{F}_q$ denotes the top-$\kappa$ ($\kappa=4$ in all experiments) query slots ranked by router importance weight $\omega_k$ (Eq.~\eqref{eq:router}), and $\mathcal{S}_c$ is the union of top-$\kappa$ slots across all support images of class $c$. To prevent scale mismatch, the raw Chamfer similarity is identically scaled by the learned temperature $\tau$. The final prediction blends the two scores:
\begin{equation}
    s(q,c) = \gamma \cdot s_\text{hol}(q,c) + (1-\gamma) \cdot \tau \cdot s_\text{Ch}(q,c),
    \label{eq:blend}
\end{equation}
where $\gamma \in [0,1]$ balances holistic class identity against part-level compositional matching.
Figure~\ref{fig:example} (Appendix~\ref{app:worked_example}) traces these steps on a real query.

\paragraph{Setting $\gamma$.}
All main results use one constant, $\gamma = 0.3$, set early on DINOv2/CGQA and kept for every dataset and backbone
without search. A closed-form rule removes even this constant. With $m$ the mean projected slot and $R^2$ the slot energy surviving centering,
unit-norm slots give $R^2 + \lVert m\rVert^2 = 1$, splitting the image as the branches read it: $R^2$ is what
Chamfer works with, $\lVert m\rVert^2$ what only the holistic branch sees. We therefore set $\gamma = 1 - R^2$,
read from the support set once per episode, with no labels, gradients or sweep. It matches the constant on CGQA
($-0.29$\,pp) and improves COBJ ($+0.87$\,pp noc; Appendix~\ref{sec:auto_gamma}).

%% file: sections/experiments.tex
\section{Experiments}
\label{sec:experiments}

\subsection{Setup}
\label{sec:setup_exp}

\paragraph{Datasets and protocol.}
We evaluate on \textit{CFST}~\citep{cfst}, comprising CGQA (100 classes, 10 sessions, five splits) and COBJ (built from Objects365~\citep{objects365}; 30 classes, 3 sessions, four splits) over three seeds (42, 123, 7) with 300 episodes of 10-way 10-shot per run, on \textit{standard CL} protocols (CIL on CIFAR-100, ImageNet-R, ImageNet-A; FSCIL on CIFAR-100, miniImageNet, CUB-200; Appendix~\ref{app:standard_cil}), and on \textit{DN4IL}~\citep{duca}, a six-session Domain-IL split of DomainNet testing domain rather than class shift (Appendix~\ref{sec:domain_shift}).

\paragraph{Baselines and implementation.}
All CFST baselines share a frozen DINOv2 ViT-B/14~\citep{dinov2} backbone with DINOSAUR-style slot attention~\citep{dinosaur} trained on COCO~2017~\citep{coco}, isolating learning strategy from backbone quality. Baselines are zero-shot (DINOv2-Proto, SimpleShot~\citep{simpleshot}), offline meta-trained (TransSlot, MetaProtoNet, both self-implemented), PTM-based CIL (RanPAC, EASE~\citep{ease}, CoFiMA~\citep{cofima}), and Comp-FSCIL~\citep{Comp-FSCIL} on FSCIL; for the patch-source analysis we swap in six other backbones~\citep{dinov1,ibot,openclip,eva02,mae,sam}, keeping all other components fixed. COMPOSE uses $K{=}7$ slots, router hidden dimension $h{=}64$, a $D{\times}D$ projection head, Adam ($\text{lr}{=}10^{-3}$, $\lambda_d{=}0.02$) and $\gamma{=}0.3$ for all backbones and benchmarks (derived rule: Appendix~\ref{sec:auto_gamma}).

\begin{table}[t]
\caption{CFST benchmark on \textbf{CGQA}, 10-way 10-shot (mean over 3 seeds: 42, 123, 7).
$\pm$ci $= 95\%$ CI. $\dagger$: frozen DINOv2 ViT-B/14; $\ddagger$: CoFiMA also fine-tunes DINOv2 (${\approx}4$\,h/dataset);
$\star$: episodically meta-trained; $\S$: single seed (42).
\textbf{Bold} = best overall; \underline{Underline} = best frozen-backbone (excl.\ COMPOSE-CT).
Per-split and COBJ results in Appendix~\ref{app:full_cfst_results}.}
\label{tab:main_reduce}
\centering
\footnotesize
\setlength{\tabcolsep}{3pt}
\resizebox{\textwidth}{!}{%
\begin{tabular}{lccccccccc}
\toprule
& \multicolumn{2}{c}{\textit{Zero-shot}} 
& \multicolumn{2}{c}{\textit{Offline meta}} 
& \multicolumn{3}{c}{\textit{PTM-based CIL}} 
& \multicolumn{2}{c}{\textit{Ours}} \\
\cmidrule(lr){2-3} \cmidrule(lr){4-5} \cmidrule(lr){6-8} \cmidrule(lr){9-10}
& DINOv2-Proto$^\dagger$
& SimpleShot$^\dagger$
& TransSlot$^{\dagger\star}$
& MetaProto$^{\dagger\star}$
& RanPAC$^{\dagger\S}$
& EASE$^{\dagger\S}$
& CoFiMA$^{\ddagger\S}$
& COMPOSE$^{\dagger\star}$
& COMPOSE-CT$^{\dagger\star}$ \\
\midrule
noc
  & $83.28_{\pm.46}$ & $86.79_{\pm1.29}$
  & $69.68_{\pm.37}$ & $79.32_{\pm.35}$
  & 88.60 & 91.21 & 91.97
  & \underline{\textbf{93.46}$_{\pm.33}$} & $90.77_{\pm.53}$ \\
$H_a\uparrow$
  & $67.41_{\pm.49}$ & $75.13_{\pm.91}$
  & $74.60_{\pm1.03}$ & $83.63_{\pm.04}$
  & 87.28 & 92.06 & \textbf{94.74}
  & \underline{$92.42_{\pm.25}$} & $93.15_{\pm.11}$ \\
\bottomrule
\end{tabular}%
}
\end{table}

\begin{table}[t]
\caption{Effect of the patch source on COMPOSE (\textit{CGQA}, 10-way 10-shot, 3 seeds; all other components fixed).
Second header row: parameters / pretraining images. A supervised ViT-H/14 (630\,M / 14\,M) reaches noc $82.75$; three further sources with a COCO purity column are in Appendix~\ref{sec:patch_source}.
$\ddagger$: from~\citep{compslot}, ImageNet-21K supervised ViT-B/16 backbone. \textbf{Bold} = best noc.}
\label{tab:backbone_type_reduce}
\centering
\scriptsize
\resizebox{\textwidth}{!}{%
\begin{tabular}{lcccccccc}
\toprule
\multicolumn{9}{c}{\textit{ViT-B/16 (IN-21K)$^\ddagger$}} \\
\cmidrule(lr){2-9}
& CPrompt & ADAM & RanPAC & EASE & CoFiMA & FOSTER* & DER* & MEMO* \\
\midrule
noc
  & $86.53_{\pm.60}$ & $84.97_{\pm.19}$ & $79.13_{\pm1.59}$ & $85.87_{\pm.54}$
  & $89.23_{\pm.39}$ & $68.77_{\pm2.20}$ & $74.87_{\pm1.54}$ & $77.70_{\pm2.08}$ \\
$H_a\uparrow$
  & $80.07$ & $75.58$ & $78.81$ & $82.84$
  & $87.93$ & $84.66$ & $85.68$ & $82.35$ \\
\bottomrule
\end{tabular}%
}

\vspace{0.5em}

\begin{tabular}{lcccc}
\toprule
\multicolumn{5}{c}{\textit{COMPOSE (ours)}} \\
\cmidrule(lr){2-5}
& DINOv1 ViT-B/16 & OpenCLIP ViT-L/14 & iBOT ViT-B/16 & DINOv2 ViT-B/14 \\
& 86\,M / 1.28\,M & 304\,M / 2\,B & 86\,M / 1.28\,M & 86\,M / 142\,M \\
\midrule
noc
  & $88.69_{\pm0.27}$
  & $89.71_{\pm0.59}$
  & $90.09_{\pm0.62}$
  & \textbf{$93.46_{\pm0.33}$} \\
$H_a\uparrow$
  & $87.56_{\pm0.58}$
  & $87.34_{\pm0.44}$
  & $88.68_{\pm1.07}$
  & \textbf{$92.42_{\pm0.25}$} \\
\bottomrule
\end{tabular}

\end{table}

\subsection{Main Results}
\label{sec:main_results}
Table~\ref{tab:main_reduce} reports CGQA results, with the per-split breakdown and COBJ deferred to Appendix~\ref{app:full_cfst_results}. On CGQA, other meta-trained methods degrade noc (TransSlot: 69.68\%, MetaProtoNet: 79.32\%) and even CoFiMA (full DINOv2 fine-tuning, ${\approx}4$\,h) trails by 1.49\,pp, so fine-tuning does not recover the invariances novel-concept generalization needs. Most of that margin is the decomposition itself: MetaProtoNet is the monolithic global embedding of Section~\ref{sec:intro} (CLS token, no slots, $79.32$), while COMPOSE read with a purely holistic score ($\gamma{=}1.0$) keeps the slots and merely sums them, reaching $90.85$, so \textbf{$+11.53$\,pp comes from the decomposition alone}.

On COBJ, COMPOSE is \emph{not} competitive at the reported blend: noc $88.32$ is below zero-shot SimpleShot ($89.19$), EASE ($89.03$) and DINOv2-Proto ($88.80$), and $H_a$ $77.31$ is below EASE ($80.85$); Appendix~\ref{app:sensitivity_alpha} shows the blend there prefers a higher holistic weight. Table~\ref{tab:backbone_type_reduce} shows \textit{noc rising monotonically with the semantic density of the patch source}: 88.69\% $\to$ 90.09\% $\to$ 93.46\% (DINOv1, iBOT, DINOv2). Scale does not order it: OpenCLIP ViT-L/14 still misses iBOT ViT-B/16 ($-0.38$ noc, $-1.34$ $H_a$) at ${\sim}3.5\times$ the parameters and ${\sim}1500\times$ the data, so language-aligned geometry is usable but capped. Appendices~\ref{sec:disentangle}--\ref{sec:auto_gamma} report the train~$\times$~inference cross-ablation, three further patch sources, DN4IL, and the derived blend coefficient. The backbone stays frozen throughout; full per-split results, the COMPOSE vs.\ COMPOSE-CT gap across backbones and the noc--sys trade-off are in Appendix~\ref{app:full_backbone_ablation_results}.

\subsection{Continual Learning Benchmarks}
\label{sec:cl_results}
Tables~\ref{tab:cil_fscil} and~\ref{tab:cil_vit} (Appendix~\ref{app:standard_cil}) report accuracy and FF. On FSCIL, COMPOSE forgets far less than Comp-FSCIL (e.g.\ FF $27.91\% \to 3.54\%$ on CIFAR-100) and leads on CIFAR-100. On CIL (single seed) it leads on ImageNet-A ($A_T$ $70.31\%$) and FeCAM on ImageNet-R; CIL forgetting is within the baselines' range. Freezing the backbone across incremental sessions removes representation drift as the main source of forgetting, and COMPOSE-CT's higher FF shows Chamfer training trading compositional retention for router stability.

\subsection{Ablation Study}
\label{sec:ablation}

Table~\ref{tab:ablation_reduce} reports ablations ($\Delta$noc = CGQA / COBJ).
On the \emph{training objective}, COMPOSE-CT drops noc by 2.69\,pp, so Chamfer
must stay inference-only, while removing redundancy reduction drops it further
($-5.31$\,pp), making decorrelation the primary noc driver. On \emph{inference matching}, both branches are needed on CGQA: dropping holistic
scoring costs $H_a$ 5.87\,pp, dropping Chamfer 1.70\,pp. On COBJ, dropping Chamfer
marginally helps ($+0.60$\,pp noc), the same tilt towards the holistic branch that
Appendix~\ref{app:sensitivity_alpha} reports. On \emph{slot
representations}, replacing the attention-weighted aggregates $\tilde{\phi}_k$ with raw GRU
states drops noc by 2.96\,pp (CGQA) and 4.98\,pp (COBJ) and $H_a$ by 3.44 and 7.42\,pp, a statement
about how slots are pooled into $\tilde{\phi}_k$: the recurrent state mixes routing
dynamics into the representation, and that costs more on natural images.

\begin{table}[t]
\caption{Ablations, 10-way 10-shot, 3 seeds; the COMPOSE column is Table~\ref{tab:main_reduce}. $\Delta$ vs.\ COMPOSE (CGQA/COBJ); per-split results in Appendix~\ref{app:ablation_full}.}
\label{tab:ablation_reduce}
\centering
\footnotesize
\setlength{\tabcolsep}{4pt}
\resizebox{\textwidth}{!}{%
\begin{tabular}{lcccccccc}
\toprule
& \textbf{COMPOSE}
& \multicolumn{2}{c}{\textit{Training objective}}
& \multicolumn{3}{c}{\textit{Inference matching}}
& \multicolumn{2}{c}{\textit{Slot representations}} \\
\cmidrule(lr){2-2} \cmidrule(lr){3-4} \cmidrule(lr){5-7} \cmidrule(lr){8-9}
& {\scriptsize (baseline)}
& {\scriptsize +Chamfer (CT)}
& {\scriptsize $\lambda_d{=}0$}
& {\scriptsize w/o centering}
& {\scriptsize w/o Chamfer}
& {\scriptsize w/o holistic}
& {\scriptsize GRU state}
& {\scriptsize mean-pool} \\
\midrule
\multicolumn{9}{l}{\textbf{CGQA}} \\
noc     & 93.46 & 90.77 & 88.15 & 92.52 & 90.85 & 90.76 & 90.50 & 92.44 \\
$H_a$   & 92.42 & 93.15 & 91.50 & 91.48 & 90.72 & 86.55 & 88.98 & 90.79 \\
\midrule
\multicolumn{9}{l}{\textbf{COBJ}} \\
noc     & 88.32 & 87.82 & 87.32 & 88.23 & 88.91 & 84.81 & 83.34 & 88.80 \\
$H_a$   & 77.31 & 77.63 & 77.51 & 76.84 & 77.73 & 68.98 & 69.88 & 75.48 \\
\midrule
\multicolumn{9}{l}{\textit{$\Delta$ relative to COMPOSE (CGQA / COBJ):}} \\
$\Delta$noc   & --- & $-2.69/-0.50$ & $-5.31/-1.00$ & $-0.93/-0.09$ & $-2.61/+0.60$ & $-2.70/-3.51$ & $-2.96/-4.98$ & $-1.01/+0.48$ \\
$\Delta H_a$  & --- & $+0.73/+0.33$ & $-0.92/+0.20$ & $-0.94/-0.46$ & $-1.70/+0.42$ & $-5.87/-8.33$ & $-3.44/-7.42$ & $-1.63/-1.82$ \\
\bottomrule
\end{tabular}%
}
\end{table}

%% file: sections/conclusion.tex
\section{Conclusion}
\label{sec:conclusion}

We introduced COMPOSE, which resolves two coupled structural failures in continual few-shot learning models. At the representation layer, reliable slot decomposition needs semantic content in the patch tokens and the right block to read it from, which across seven sources predicts novel-concept transfer better than self-supervision as a category; attention-weighted aggregates ($\tilde{\phi}_k$) replace recurrent bottleneck states and preserve that geometry. At the optimization layer, the \emph{compositional trap} is that part-level matching in training induces slot-specific gradients forcing over-specialization; we break it by \emph{training holistically and inferring compositionally}, with a cross-correlation penalty that prevents redundancy without conflicting variance constraints. COMPOSE reaches $93.46 \pm 0.33\%$ noc on CGQA ($H_a = 92.42$), beating CoFiMA by $1.49$\,pp with ${\sim}5$ minutes of training rather than ${\sim}4$ hours.

\paragraph{Limitations and Future Work.}
COMPOSE depends on semantic content in the patch tokens and the block they are read at, so supervised
ViTs and pixel- or mask-supervised objectives degrade slot decomposition regardless of capacity, and the
CLIP encoders we ran route slots successfully but plateau below the self-distilled backbones.

%% file: sections/appendix.tex
\section{Full CFST Main Results}
\label{app:full_cfst_results}
\begin{table}[ht]
\caption{CFST benchmark, 10-way 10-shot (mean over 3 seeds: 42, 123, 7).
$\pm$ci $= 95\%$ CI. \textbf{†}: frozen DINOv2 ViT-B/14.
\textbf{‡}: CoFiMA also fine-tunes DINOv2 (${\approx}4$\,h/dataset).
$\star$: episodically meta-trained.
$\S$: single-seed result (seed 42 only).
Marked for noc and $H_a$ only: \textbf{Bold} = best overall; \underline{Underline} = best frozen-backbone, excluding our COMPOSE-CT ablation.}
\label{tab:main}
\centering
\resizebox{\textwidth}{!}{%
\begin{tabular}{lccccccccccc}
\toprule
& \multicolumn{6}{c}{\textbf{CGQA}} & \multicolumn{5}{c}{\textbf{COBJ}} \\
\cmidrule(lr){2-7}\cmidrule(lr){8-12}
\textbf{Method} & sys & pro & sub & non & noc & $H_a\uparrow$
    & sys & pro & non & noc & $H_a\uparrow$ \\
\midrule
\multicolumn{12}{l}{\textit{Zero-shot inference (no training):}} \\
DINOv2-Proto\textsuperscript{†}
    & $60.82_{\pm1.30}$ & $65.12_{\pm.88}$ & $70.36_{\pm.96}$ & $61.88_{\pm.94}$ & $83.28_{\pm.46}$ & $67.41_{\pm.49}$
    & $66.93_{\pm.95}$ & $52.90_{\pm.47}$ & $56.09_{\pm.36}$ & $88.80_{\pm1.02}$ & $63.56_{\pm.09}$ \\
SimpleShot\textsuperscript{†}
    & $66.50_{\pm1.78}$ & $83.11_{\pm.72}$ & $75.98_{\pm1.50}$ & $67.59_{\pm.57}$ & $86.79_{\pm1.29}$ & $75.13_{\pm.91}$
    & $72.72_{\pm1.12}$ & $55.99_{\pm.55}$ & $61.31_{\pm.40}$ & \underline{\textbf{89.19}}$_{\pm.62}$ & $67.64_{\pm.07}$ \\
\midrule
\multicolumn{12}{l}{\textit{Offline meta-trained (not continual):}} \\
TransSlot\textsuperscript{†$\star$}
    & $75.58_{\pm.63}$ & $68.96_{\pm4.39}$ & $81.12_{\pm.48}$ & $79.35_{\pm.91}$ & $69.68_{\pm.37}$ & $74.60_{\pm1.03}$
    & $72.95_{\pm.22}$ & $58.25_{\pm1.04}$ & $66.64_{\pm.79}$ & $84.34_{\pm1.03}$ & $69.27_{\pm.43}$ \\
MetaProtoNet\textsuperscript{†$\star$}
    & $82.23_{\pm.52}$ & $85.76_{\pm.46}$ & $86.62_{\pm.62}$ & $84.64_{\pm.22}$ & $79.32_{\pm.35}$ & $83.63_{\pm.04}$
    & $76.90_{\pm.53}$ & $66.55_{\pm.47}$ & $75.01_{\pm.61}$ & $81.76_{\pm.59}$ & $74.64_{\pm.46}$ \\
\midrule
\multicolumn{12}{l}{\textit{PTM-based CIL (continual, standard training):}} \\
RanPAC\textsuperscript{†$\S$}
    & 84.02 & 89.33 & 88.78 & 85.94 & 88.60 & 87.28
    & 80.60 & 67.78 & 77.72 & 88.27 & 77.89 \\
EASE\textsuperscript{†$\S$}
    & 90.12 & 93.77 & 93.55 & 91.78 & 91.21 & 92.06
    & 83.66 & 72.11 & 80.51 & 89.03 & \underline{\textbf{80.85}} \\
CoFiMA\textsuperscript{‡$\S$}
    & 94.04 & 97.03 & 95.49 & 95.41 & 91.97 & \textbf{94.74}
    & 77.12 & 63.08 & 79.57 & 87.25 & 75.69 \\
\midrule
\textbf{COMPOSE}\textsuperscript{†$\star$} (ours)
    & $90.17_{\pm.80}$ & $94.17_{\pm.43}$ & $93.23_{\pm.07}$
    & $91.20_{\pm.82}$ & \underline{\textbf{93.46}$_{\pm.33}$} & \underline{92.42}$_{\pm.25}$
    & $81.38_{\pm.66}$ & $66.73_{\pm.80}$ & $76.07_{\pm.05}$
    & $88.32_{\pm.49}$ & $77.31_{\pm.47}$ \\
COMPOSE-CT (ablation)\textsuperscript{†$\star$}
    & $92.20_{\pm.79}$ & $95.32_{\pm.37}$ & $94.42_{\pm.46}$
    & $93.20_{\pm.69}$ & $90.77_{\pm.53}$ & $93.15_{\pm.11}$
    & $81.00_{\pm.57}$ & $67.56_{\pm.78}$ & $76.99_{\pm.77}$
    & $87.82_{\pm.96}$ & $77.63_{\pm.40}$ \\
\bottomrule
\end{tabular}%
}
\end{table}
\paragraph{On noc outperforming sys.}
COMPOSE achieves higher noc (93.46\%) than sys (90.17\%) because solvability
is violated with respect to the \emph{CGQA training curriculum}, not with
respect to DINOv2 pretraining. DINOv2's frozen features already encode rich
representations of noc objects from its LVD-142M pretraining corpus. sys
generalization is in contrast sensitive to interference from continual training
on a shared concept pool---making the frozen DINOv2 geometry the key asset.
\section{Slot Purity: Formal Definition and Extended Analysis}
\label{app:purity_extended}

\subsection{Formal Definition}
\label{app:purity_def}

Let $\mathcal{D} = \{(I_i, M_i)\}$ be a set of images paired with instance
segmentation masks, where $M_i = \{(m_{ij}, c_{ij})\}_{j=1}^{J_i}$ provides
binary foreground masks $m_{ij} \in \{0,1\}^{H \times W}$ and category labels
$c_{ij}$ for $J_i$ object instances in image $i$.
For a frozen slot attention model producing $K$ slots per image with attention
maps $A \in [0,1]^{K \times N}$ over $N$ patch tokens (where $N = H_p \times W_p$
after spatial downsampling), define the \emph{pixel-level attention matrix}
$\hat{A} \in [0,1]^{K \times HW}$ by upsampling each column of $A$ to full
image resolution via bilinear interpolation.

\paragraph{Per-slot category assignment.}
Given image $i$ with attention maps $\hat{A}^{(i)}$ and ground-truth masks
$M_i$, the category assigned to slot $k$ is:
\begin{equation}
    \hat{c}_k^{(i)} = \operatorname{argmax}_{c} \sum_{p \in \mathcal{P}_c^{(i)}}\hat{A}_{kp}^{(i)},
    \label{eq:slot_assign}
\end{equation}
where $\mathcal{P}_c^{(i)} = \{p \,:\, \exists\, j \text{ s.t.\ }
m_{ij}^{(p)} = 1 \text{ and } c_{ij} = c\}$ is the set of pixels belonging to
category $c$ in image $i$, and a background category $c_0$ collects all
non-object pixels.
Equation~\eqref{eq:slot_assign} assigns each slot to the object category it
attends to most, by total attention mass---a soft analogue of the
\emph{dominant-object} rule used in evaluating attention-based segmentation.

\paragraph{Purity.}
\emph{Slot purity} $\rho$ is then defined as:
\begin{equation}
    \rho = \frac{1}{|\mathcal{D}|} \sum_{i=1}^{|\mathcal{D}|}
           \frac{1}{K} \sum_{k=1}^{K}
           \mathbf{1}\!\left[\hat{A}_{k\cdot}^{(i)} \cdot \mathbf{1}_{c^*_k(i)}^{(i)} > 0.5\right],
    \label{eq:purity}
\end{equation}
where $c^*_k(i) = \operatorname*{arg\,max}_{c \in \mathcal{C}_\text{GT}} \hat{A}_{k\cdot}^{(i)} \cdot \mathbf{1}_c^{(i)}$
identifies the dominant ground-truth category for slot $k$ in image $i$. The indicator
evaluates to $1$ if and only if this dominant category captures a \emph{strict majority}
of slot $k$'s total attention mass (${>}0.5$), ensuring the slot isolates a single object
category rather than mixing concepts. We evaluate on $|\mathcal{D}| = 5{,}000$ validation images from COCO 2017~\citep{coco}, with $K=7$.

\subsection{Why Slot Purity Predicts Novel-Concept Transfer}
\label{app:purity_theory}

Slot purity captures whether each slot is \emph{object-grounded}: does it
encode the feature statistics of a single coherent object, or does it mix
features from multiple semantically distinct objects?
This distinction matters directly for noc generalization.

Recall the compositional transfer setting: an entirely unseen concept $c_\text{noc}$ ($\mathcal{C}_\text{noc} \cap \mathcal{C} = \emptyset$) presents global visual semantics absent from training. For part-level transfer to succeed on these novel concepts, the following must hold:

\begin{enumerate}[label=(\roman*)]
    \item Each slot $k$ must consistently isolate a \emph{semantically pure} object primitive, remaining invariant to the unfamiliar holistic contexts encountered at inference.
    \item The support slot pool $\mathcal{S}_c$ (used for Chamfer matching) must capture these pure primitives faithfully. When $c_\text{noc}$ presents a novel global structure, these pristine slots allow the Chamfer distance to establish reliable part-to-part correspondences without relying on familiar global compositions.
\end{enumerate}

Low slot purity undermines (i): if slot $k$ mixes object $A$ and object $B$
in training, its prototype encodes a superposition of $A$- and $B$-statistics.
At inference on a novel concept containing $A$ but not $B$, the slot $k$
prototype fails to match---not because the compositional rule is wrong, but
because the individual slot representation is noisy.
High purity, by contrast, guarantees that slot $k$'s prototype is a clean
summary of one semantic part, enabling reliable Chamfer and holistic matching
on novel compositions.

This argument predicts a monotonic relationship between slot purity and noc.
Table~\ref{tab:backbone_type} is consistent: DINOv2 (purity~$= 0.9046$)
achieves the highest noc across all evaluated backbones; the GRU-based baseline
with purity~$= 0.7679$ (see below) suffers a $-2.96$\,pp noc drop on CGQA and
$-4.98$\,pp on COBJ (Table~\ref{tab:ablation}).

\subsection{Effect of Slot Computation on Purity}
\label{app:purity_sca}

We compare purity under two slot representations available in our pipeline:

\begin{itemize}
    \item \textbf{GRU state $s_k$}: the hidden state of the recurrent slot
          updater after the final refinement iteration (standard DINOSAUR output).
    \item \textbf{SCA-refined slot $\tilde{\phi}_k$}: the slot cross-attention (SCA) readout
          from the frozen patch features over the final slot attention map
          (Eq.~\eqref{eq:phi} in the main paper), which re-expresses the slot directly in
          the patch feature space.
\end{itemize}

\begin{table}[h]
\caption{Slot purity on COCO val2017 ($K=7$, 5000 images) under different
slot representations. Both use the same DINOv2 ViT-B/14 backbone and DINOSAUR
attention maps.}
\label{tab:purity_sca}
\centering
\begin{tabular}{lcc}
\toprule
\textbf{Slot representation} & \textbf{Purity $\rho$} & \textbf{noc (CGQA 10-shot)} \\
\midrule
GRU state $s_k$              & 0.7679 & 90.50 \\
SCA-refined $\tilde{\phi}_k$ & 0.9046 & 93.46 \\
\midrule
$\Delta$                     & $+0.1367$ & $+2.96$\,pp \\
\bottomrule
\end{tabular}
\end{table}

The GRU state $s_k$ lives in the slot attention model's latent space, which
is trained to minimise patch reconstruction error---an objective that rewards
covering all foreground pixels without explicit object-level separation.
The result is a representation that conflates multiple objects when they
co-occur in the same spatial region, yielding purity~$= 0.7679$.

The SCA-refined representation $\tilde{\phi}_k = \text{L2-norm}(\sum_n \alpha_{kn} f_n)$
re-expresses each slot as an L2-normalized weighted average of DINOv2 patch features, where
the weights $\alpha_{kn}$ are the (already-computed) slot attention maps.
Because DINOv2 features are themselves highly object-discriminative and encode
rich semantic geometry~\citep{dinov2}, this readout concentrates each slot's
representation on the semantic content of the patches it attends to most.
The result is a representation that, for high-purity slots, approximates the
\emph{object-centric mean feature} of the attended object---a clean, semantically
grounded summary compatible with prototype-based compositional matching.

Figure~\ref{fig:purity_maps} shows the two representations of Table~\ref{tab:purity_sca} for the same
$7{,}036$ objects, next to a single patch token of each object.

\begin{figure}[!ht]
    \centering
    \includegraphics[width=\textwidth]{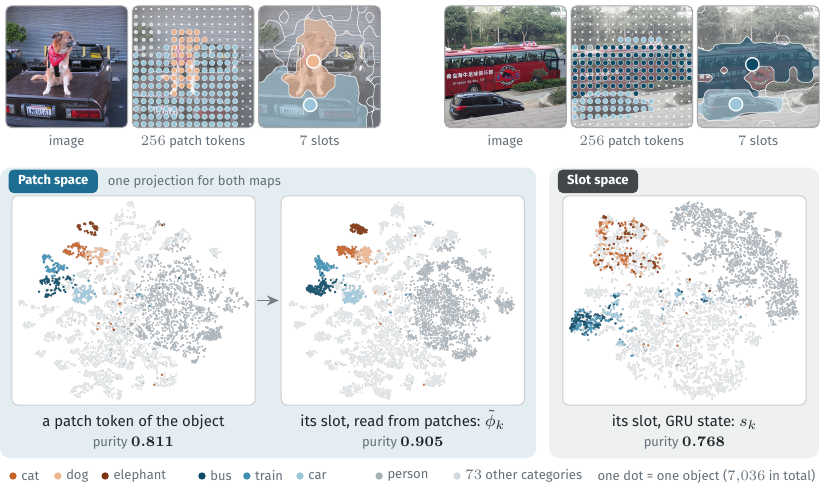}
    \caption{\textbf{What a slot keeps from its patch source.} \emph{Top:} two COCO val2017 images (photographs by Dan Phiffer and by bfishadow on Flickr, CC~BY~2.0); DINOSAUR groups the $256$ patch tokens of an image into $7$ slots, and a large dot marks each slot matched to an object. \emph{Bottom:} t-SNE maps with one dot per matched object ($7{,}036$ objects, the set behind Table~\ref{tab:purity_sca}), coloured by category; warm shades are animals and cool shades are vehicles. A single patch token of the object and the slot read back from patch tokens, $\tilde{\phi}_k$, share one projection; the GRU state $s_k$ of the same slot is projected separately. Purity is measured as in Table~\ref{tab:purity_sca}. Reading the slot back sharpens the clusters of the patch source ($0.811 \to 0.905$), whereas $s_k$ merges categories of the same kind: for the $983$ animal objects, the nearest neighbour of $s_k$ is another animal in $96\%$ of cases but the same species in only $65\%$, against $94\%$ for $\tilde{\phi}_k$. Map coordinates are for display only; all numbers are computed in the original spaces.}
    \label{fig:purity_maps}
\end{figure}

\subsection{Backbone and Purity}
\label{app:purity_backbone}

Table~\ref{tab:purity_extended} reports purity and NMI across all evaluated
backbones for both GRU-state and SCA-refined representations on COCO val2017.

\begin{table}[h]
\caption{Slot purity and NMI on COCO val2017 ($K=7$, 5000 images) across
different backbone and slot representations. $\Delta$ vs GRU is computed
relative to the GRU baseline of the corresponding backbone patch size.
Attn src denotes the attention maps used for SCA readout.}
\label{tab:purity_extended}
\centering
\begin{tabular}{lcccc}
\toprule
\textbf{Representation} & \textbf{Purity $\rho$} & \textbf{NMI} & \textbf{$\Delta$ vs GRU} & \textbf{Attn src} \\
\midrule
$s_k$ B/14 (DINOv2-GRU)                  & 0.7679 & 0.4215 & (baseline) & --- \\
$s_k$ B/16 (DINOv1-GRU)                  & 0.5146 & 0.2179 & (baseline) & --- \\
\midrule
$\tilde{\phi}_k$ DINOv2-B/14 (DINOSAUR)  & 0.9046 & 0.7141 & $+0.1367$  & B/14 attn \\
$\tilde{\phi}_k$ DINOv1-B/16 (DINOSAUR)  & 0.6949 & 0.4850 & $+0.1802$  & B/16 attn \\
$\tilde{\phi}_k$ EVA02-B/14              & 0.8408 & 0.6697 & $+0.0729$  & B/14 attn \\
$\tilde{\phi}_k$ CLIP-B/16              & 0.7814 & 0.6155 & $+0.2668$  & B/16 attn \\
$\tilde{\phi}_k$ iBOT-B/16              & 0.7640 & 0.5554 & $+0.2494$  & B/16 attn \\
\bottomrule
\end{tabular}
\end{table}

We provide qualitative interpretations grounded in each backbone's training
objective.

\textbf{DINOv1.}
DINOv1 self-distillation~\citep{dinov1} encourages global and local
feature consistency across views via a student--teacher distillation.
The resulting patch features carry significant spatial coherence (consistent with DINOv1's known foreground segmentation property) but at a smaller scale and with a less expressive pretraining corpus than DINOv2 (ImageNet-1K vs.\ LVD-142M),
yielding purity $= 0.6949$ for SCA-refined slots---below DINOv2 but substantially
above the DINOv1-GRU baseline ($0.5146$).

\textbf{iBOT.}
iBOT~\citep{ibot} adds masked image modeling to DINO's self-distillation,
explicitly encouraging patch tokens to predict masked patch content.
This provides a stronger local reconstruction signal than pure self-distillation.
The SCA-refined iBOT slots achieve purity $= 0.7640$, consistent with the
iBOT noc results (90.09\%) sitting between DINOv1 (88.69\%) and DINOv2 (93.46\%).

\textbf{EVA02 and CLIP.}
EVA02~\citep{eva02} and CLIP~\citep{clip} backbones achieve purity $0.8408$
and $0.7814$ respectively under SCA readout. Despite neither being trained
with an explicit object-segmentation objective, the strong visual--semantic
alignment in both enables the SCA readout to isolate coherent object regions,
with $\Delta$ vs.\ GRU gains of $+0.0729$ and $+0.2668$.

\textbf{Supervised ViT-B/16.}
Supervised training on ImageNet-1K with cross-entropy produces features that
encode \emph{global class discriminability} rather than \emph{patch-level
object identity}. The CLS token carries class information; patch tokens encode
local texture statistics that do not organise along object boundaries.
As a result, DINOSAUR slot attention applied to supervised ViT features
produces low-purity slots (purity~$< 0.3$), explaining the dramatic noc
degradation reported in Appendix~\ref{app:supervised_vit}.
\subsection{Full Backbone Ablation Results}
\label{app:full_backbone_ablation_results}
\subsubsection{CGQA full per-split breakdown}
\label{app:backbone_cgqa_full}

Table~\ref{tab:backbone_type} reports the complete per-split breakdown (sys, pro, sub, non, noc, $H_a$) for both COMPOSE and COMPOSE-CT across four self-supervised backbones on CGQA. Two patterns extend the main-text analysis. First, the \textbf{COMPOSE vs.\ COMPOSE-CT noc gap is backbone-invariant} ($+2.99$\,pp on DINOv1, $+2.18$\,pp on iBOT, $+2.69$\,pp on DINOv2, $+2.45$\,pp on OpenCLIP), confirming the noc advantage of holistic training is a property of the \emph{training objective} rather than an interaction with a specific backbone's feature structure. Second, \textbf{COMPOSE-CT consistently leads on sys and pro across all backbones}, establishing the noc--sys trade-off as a structural consequence of Chamfer gradients (Corollary~\ref{cor:assignment_gradient}) rather than a DINOv2-specific artifact.

\begin{table}[ht]
\caption{Effect of self-supervised backbone on COMPOSE and COMPOSE-CT
on the \textbf{CGQA} benchmark
(10-way 10-shot, mean over 3 seeds: 42, 123, 7). All other COMPOSE
components fixed (same DINOSAUR slot attention, same MLP router architecture).
\textsuperscript{$\ddagger$}Results taken from~\citep{compslot};
all methods in that block use an ImageNet-21K supervised ViT-B/16 backbone.
\textbf{Bold} = best per metric within each backbone group; \underline{Underline} = best $H_a$ within each backbone group (ViT-B/16\textsuperscript{$\ddagger$} block excluded from $H_a$ underline as results are taken from external work).}
\label{tab:backbone_type}
\centering
\resizebox{\textwidth}{!}{%
\begin{tabular}{ll ccccccc}
\toprule
\textbf{Backbone} & \textbf{Method}
    & sys & pro & sub & non & noc & $H_a\uparrow$ \\
\midrule
\multirow{8}{*}{\shortstack[l]{ViT-B/16\\(IN-21K)\textsuperscript{$\ddagger$}}}
  & CPrompt
    & $75.13_{\pm1.84}$ & $78.13_{\pm0.97}$ & $84.60_{\pm0.51}$
    & $77.17_{\pm0.68}$ & $86.53_{\pm0.60}$ & $80.07$ \\
  & ADAM + adapter
    & $68.53_{\pm0.96}$ & $75.03_{\pm0.53}$ & $80.40_{\pm0.09}$
    & $71.27_{\pm0.42}$ & $84.97_{\pm0.19}$ & $75.58$ \\
  & RanPAC
    & $75.83_{\pm1.76}$ & $80.60_{\pm0.80}$ & $83.43_{\pm1.78}$
    & $75.60_{\pm0.61}$ & $79.13_{\pm1.59}$ & $78.81$ \\
  & EASE
    & $78.27_{\pm0.51}$ & $84.63_{\pm0.51}$ & $86.20_{\pm0.48}$
    & $79.90_{\pm0.19}$ & $85.87_{\pm0.54}$ & $82.84$ \\
  & CoFiMA
    & $84.47_{\pm0.32}$ & $88.97_{\pm0.37}$ & \textbf{91.77}$_{\pm0.14}$
    & $85.60_{\pm0.73}$ & \textbf{89.23}$_{\pm0.39}$ & \underline{87.93} \\
  & FOSTER*
    & \textbf{87.60}$_{\pm0.61}$ & \textbf{91.73}$_{\pm0.98}$ & $90.50_{\pm0.73}$
    & \textbf{89.70}$_{\pm1.54}$ & $68.77_{\pm2.20}$ & 84.66 \\
  & DER*
    & $86.57_{\pm0.51}$ & $90.30_{\pm0.67}$ & $90.20_{\pm0.32}$
    & $88.60_{\pm0.37}$ & $74.87_{\pm1.54}$ & $85.68$ \\
  & MEMO*
    & $79.73_{\pm1.25}$ & $85.13_{\pm1.82}$ & $87.53_{\pm1.43}$
    & $82.43_{\pm2.21}$ & $77.70_{\pm2.08}$ & $82.35$ \\
\midrule
\multirow{2}{*}{DINOv1 ViT-B/16 (IN-1K)}
  & COMPOSE (ours)
    & $84.42_{\pm0.87}$ & $90.21_{\pm1.12}$ & $88.66_{\pm0.59}$
    & $86.08_{\pm0.74}$ & \textbf{88.69}$_{\pm0.27}$ & $87.56_{\pm0.58}$ \\
  & COMPOSE-CT (ours)
    & \textbf{88.13}$_{\pm0.94}$ & \textbf{92.81}$_{\pm0.65}$ & \textbf{91.04}$_{\pm0.40}$
    & \textbf{89.93}$_{\pm0.25}$ & $85.70_{\pm1.04}$ & \underline{89.46}$_{\pm0.56}$ \\
\midrule
\multirow{2}{*}{iBOT ViT-B/16 (IN-1K)}
  & COMPOSE (ours)
    & $85.44_{\pm1.41}$ & $91.24_{\pm1.43}$ & $89.84_{\pm0.86}$
    & $87.06_{\pm1.35}$ & \textbf{90.09}$_{\pm0.62}$ & $88.68_{\pm1.07}$ \\
  & COMPOSE-CT (ours)
    & \textbf{89.45}$_{\pm0.88}$ & \textbf{93.60}$_{\pm0.33}$ & \textbf{92.09}$_{\pm0.33}$
    & \textbf{90.95}$_{\pm0.84}$ & $87.91_{\pm0.29}$ & \underline{90.76}$_{\pm0.41}$ \\
\midrule
\multirow{2}{*}{DINOv2 ViT-B/14}
  & COMPOSE (ours)
    & 90.17$_{\pm0.80}$ & 94.17$_{\pm0.43}$ & 93.23$_{\pm0.07}$
    & 91.20$_{\pm0.82}$ & \textbf{93.46}$_{\pm0.33}$ & 92.42$_{\pm0.25}$ \\
  & COMPOSE-CT (ours)
    & \textbf{92.20}$_{\pm0.79}$ & \textbf{95.32}$_{\pm0.37}$ & \textbf{94.42}$_{\pm0.46}$
    & \textbf{93.20}$_{\pm0.69}$ & 90.77$_{\pm0.53}$ & \underline{93.15}$_{\pm0.11}$ \\
\midrule
\multirow{2}{*}{\shortstack[l]{OpenCLIP ViT-L/14\\(LAION-2B)}}
  & COMPOSE (ours)
    & $83.94_{\pm0.51}$ & $88.29_{\pm1.16}$ & $88.96_{\pm1.09}$
    & $86.05_{\pm0.41}$ & \textbf{89.71}$_{\pm0.59}$ & $87.34_{\pm0.44}$ \\
  & COMPOSE-CT (ours)
    & \textbf{86.75}$_{\pm0.48}$ & \textbf{91.08}$_{\pm0.72}$ & \textbf{89.97}$_{\pm0.34}$
    & \textbf{88.75}$_{\pm0.60}$ & $87.26_{\pm0.06}$ & \underline{88.73}$_{\pm0.25}$ \\
\bottomrule

\end{tabular}%
}
\end{table}
\subsubsection{COBJ full per-split breakdown}
\begin{table}[ht]
\caption{Effect of self-supervised backbone on COMPOSE and COMPOSE-CT on \textbf{COBJ}
(10-way 10-shot, mean over 3 seeds: 42, 123, 7). All other COMPOSE
components fixed (same DINOSAUR slot attention, same MLP router architecture). COBJ per-suite breakdown not reported in the original paper.
\textbf{Bold} = best per metric within each backbone group; \underline{Underline} = best $H_a$ per backbone group.}
\label{tab:backbone_type_cobj}
\centering
\resizebox{\textwidth}{!}{%
\begin{tabular}{ll cccccc}
\toprule
\textbf{Backbone} & \textbf{Method}
    & sys & pro & non & noc & $H_a\uparrow$ \\
\midrule
\multirow{2}{*}{DINOv1 ViT-B/16 (IN-1K)}
  & COMPOSE (ours)
    & \textbf{69.74}$_{\pm0.72}$ & $55.02_{\pm1.63}$ & $65.67_{\pm1.28}$
    & \textbf{76.49}$_{\pm1.05}$ & \underline{65.77}$_{\pm0.83}$ \\
  & COMPOSE-CT (ours)
    & $68.60_{\pm0.35}$ & \textbf{55.18}$_{\pm1.85}$ & \textbf{66.14}$_{\pm0.99}$
    & $75.68_{\pm1.31}$ & $65.54_{\pm0.70}$ \\
\midrule
\multirow{2}{*}{iBOT ViT-B/16 (IN-1K)}
  & COMPOSE (ours)
    & \textbf{73.35}$_{\pm0.96}$ & $57.73_{\pm2.67}$ & $68.34_{\pm1.35}$
    & \textbf{81.13}$_{\pm1.51}$ & $69.06_{\pm1.45}$ \\
  & COMPOSE-CT (ours)
    & $72.83_{\pm1.32}$ & \textbf{58.94}$_{\pm2.62}$ & \textbf{69.40}$_{\pm1.01}$
    & $80.52_{\pm1.36}$ & \underline{69.53}$_{\pm1.46}$ \\
\midrule
\multirow{2}{*}{DINOv2 ViT-B/14}
  & COMPOSE (ours)
    & \textbf{81.38}$_{\pm0.66}$ & 66.73$_{\pm0.80}$ & 76.07$_{\pm0.05}$
    & \textbf{88.32}$_{\pm0.49}$ & 77.31$_{\pm0.47}$ \\
  & COMPOSE-CT (ours)
    & 81.00$_{\pm0.57}$ & \textbf{67.56}$_{\pm0.78}$ & \textbf{76.99}$_{\pm0.77}$
    & 87.82$_{\pm0.96}$ & \underline{77.63}$_{\pm0.40}$ \\
\midrule
\multirow{2}{*}{\shortstack[l]{OpenCLIP ViT-L/14\\(LAION-2B)}}
  & COMPOSE (ours)
    & \textbf{77.35}$_{\pm1.41}$ & $62.25_{\pm2.59}$ & $72.08_{\pm1.30}$
    & $85.46_{\pm0.32}$ & $73.30_{\pm1.52}$ \\
  & COMPOSE-CT (ours)
    & $76.36_{\pm1.47}$ & \textbf{63.41}$_{\pm2.90}$ & \textbf{72.74}$_{\pm1.35}$
    & \textbf{85.48}$_{\pm0.74}$ & \underline{73.64}$_{\pm1.35}$ \\
\bottomrule

\end{tabular}%
}
\end{table}

Table~\ref{tab:backbone_type_cobj} reports the full per-split breakdown on COBJ
across all three self-supervised backbone variants. Three observations extend and sharpen the CGQA analysis in Section~\ref{sec:main_results}.

\subsubsection{noc–sys trade-off: consistency across backbones}
\paragraph{Larger absolute gaps on natural images.}
The noc gap between DINOv1 and DINOv2 is substantially wider on COBJ ($-11.83$ pp)
than on CGQA ($-4.77$ pp), and the same amplification holds for iBOT
(COBJ: $-7.19$ pp; CGQA: $-3.37$ pp).
CGQA uses synthetic composite images with clear spatial boundaries, where even
moderately structured patch features are sufficient for slot routing to recover coherent object regions.
COBJ uses natural photographs with cluttered backgrounds, partial occlusions,
and variable illumination—conditions under which slot decomposition is noisier and
more sensitive to the quality of the underlying patch geometry.
Formally, the competitive softmax in Eq.~(1) can only produce pure slots when
patches from the same object carry mutually consistent features.
Richer SSL pretraining (DINOv2, LVD-142M) produces denser patch-level semantic
clustering, yielding higher slot purity (Section~A.3) and thus more reliable
object-slot assignments even under the adverse statistics of natural scenes.
Weaker pretraining (DINOv1, ImageNet-1K) leaves patch features less organised
along object boundaries, amplifying routing errors—an effect that is largely hidden
on the forgiving geometry of synthetic CGQA images but surfaces prominently on COBJ.

\paragraph{Monotonic noc ordering is preserved.}
Despite the noisier decomposition, the backbone ordering DINOv1 $<$ iBOT $<$ DINOv2 is consistent across both datasets: noc on COBJ reaches $76.49\%$, $81.13\%$, and $88.32\%$ respectively. This monotonic ordering—mirroring the increasing richness of each backbone's pretraining objective and corpus—confirms that the noc advantage of stronger SSL geometry is not a dataset-specific artifact but a structural consequence of patch-levelsemantic quality, as argued in ~\ref{app:purity_theory}.

\paragraph{COMPOSE vs.\ COMPOSE-CT noc gap remains backbone-invariant.}
As on CGQA, holistic-only training (COMPOSE) consistently outperforms
Chamfer-augmented training (COMPOSE-CT) on noc across all backbone variants: the gaps are $+0.57$ pp (DINOv2), $+0.61$ pp (iBOT), and $+0.81$ pp (DINOv1). Although these margins are smaller than their CGQA counterparts in absolute terms, their consistency across three backbones of very different quality confirms that the noc advantage of holistic training is a property of the optimization objective, not of any interaction with a specific backbone's feature structure—and that this
property holds even when slot decomposition is made harder by natural-image statistics.

\section{Full Ablation Results}
\label{app:ablation_full}
\begin{table}[ht]
\caption{Ablation results, 10-way 10-shot, mean over 3 seeds (42, 123, 7).
$\Delta$ relative to COMPOSE (CGQA/COBJ).}
\label{tab:ablation}
\centering
\resizebox{\textwidth}{!}{%
\begin{tabular}{lccccccccccccc}
\toprule
& \multicolumn{6}{c}{\textbf{CGQA}} & \multicolumn{5}{c}{\textbf{COBJ}}
    & \multicolumn{2}{c}{\textbf{Delta}} \\
\cmidrule(lr){2-7} \cmidrule(lr){8-12} \cmidrule(lr){13-14}
\textbf{Variant} & sys & pro & sub & non & noc & $H_a$
    & sys & pro & non & noc & $H_a$
    & $\Delta$noc & $\Delta H_a$ \\
\midrule
\textbf{COMPOSE}
    & 90.17 & 94.17 & 93.23 & 91.20 & 93.46 & 92.42
    & 81.38 & 66.73 & 76.07 & 88.32 & 77.31 & --- & --- \\
\midrule
\multicolumn{14}{l}{\textit{Training objective:}} \\
+Chamfer training (COMPOSE-CT)
    & 92.20 & 95.32 & 94.42 & 93.20 & 90.77 & 93.15
    & 81.00 & 67.56 & 76.99 & 87.82 & 77.63
    & $-2.69/-0.50$ & $+0.73/+0.33$ \\
$\lambda_d{=}0$ (no redundancy reduction)
    & 90.51 & 94.23 & 93.18 & 91.70 & 88.15 & 91.50
    & 81.48 & 67.48 & 76.55 & 87.32 & 77.51
    & $-5.31/-1.00$ & $-0.92/+0.20$ \\
\midrule
\multicolumn{14}{l}{\textit{Inference matching:}} \\
w/o centering
    & 89.11 & 92.75 & 92.58 & 90.54 & 92.52 & 91.48
    & 80.02 & 66.73 & 75.55 & 88.23 & 76.84
    & $-0.93/-0.09$ & $-0.94/-0.46$ \\
w/o Chamfer ($\alpha{=}1.0$)
    & 88.65 & 91.62 & 92.29 & 90.27 & 90.85 & 90.72
    & 80.17 & 67.71 & 77.09 & 88.91 & 77.73
    & $-2.61/+0.60$ & $-1.70/+0.42$ \\
w/o holistic ($\alpha{=}0.0$)
    & 82.80 & 88.22 & 87.57 & 83.92 & 90.76 & 86.55
    & 75.79 & 57.08 & 64.59 & 84.81 & 68.98
    & $-2.70/-3.51$ & $-5.87/-8.33$ \\
\midrule
\multicolumn{14}{l}{\textit{Slot representations:}} \\
GRU state $s_k$ (no SCA)
    & 86.23 & 90.89 & 90.20 & 87.30 & 90.50 & 88.98
    & 73.99 & 59.19 & 67.44 & 83.34 & 69.88
    & $-2.96/-4.98$ & $-3.44/-7.42$ \\
mean-pool (no router)
    & 87.57 & 94.20 & 91.27 & 88.78 & 92.44 & 90.79
    & 80.25 & 64.08 & 73.18 & 88.80 & 75.48
    & $-1.01/+0.48$ & $-1.63/-1.82$ \\
\bottomrule
\end{tabular}
}
\end{table}
\section{Class-Incremental Learning on Standard Benchmarks}
\label{app:standard_cil}
Beyond the CFST compositional splits, we evaluate COMPOSE on two complementary
standard continual learning protocols to assess whether the frozen-backbone design
generalises to conventional class-incremental settings.

\subsection{Few-Shot Class-Incremental Learning (FSCIL)}

Table~\ref{tab:cil_fscil} compares COMPOSE against Comp-FSCIL~\citep{Comp-FSCIL}
on three standard FSCIL benchmarks (CIFAR-100, CUB-200, miniImageNet) following the standard Comp-FSCIL protocol. CIFAR-100 and miniImageNet use 60 base classes (60\%) followed by 8 incremental sessions of 5-way 5-shot (100 total classes each); CUB-200 uses 100 base classes (50\%, the standard FSCIL split) followed by 10 incremental sessions of 10-way 5-shot (200 total classes). Both methods use a frozen DINOv2 ViT-B/14 backbone.

\paragraph{Forgetting.}
COMPOSE achieves substantially lower FF than Comp-FSCIL on CIFAR-100
($3.54\%$ vs.\ $27.91\%$, a $7.9\times$ reduction) and on CUB-200
($2.25\%$ vs.\ $14.69\%$). On miniImageNet the gap narrows
($2.22\%$ vs.\ $4.40\%$).

\paragraph{Final-session accuracy.}
COMPOSE achieves the best $A_T$ on CIFAR-100 ($84.02\%$ vs.\ $74.97\%$,
$+9.05$ pp), while trailing marginally on CUB-200 ($79.20\%$ vs.\ $79.54\%$,
$-0.34$ pp) and miniImageNet ($94.04\%$ vs.\ $94.22\%$, $-0.18$ pp).

\paragraph{COMPOSE vs.\ COMPOSE-CT.}
COMPOSE-CT achieves lower FF than COMPOSE on all three datasets
(e.g., $1.72\%$ vs.\ $2.22\%$ on miniImageNet) and higher $A_T$ on CUB-200
($80.26\%$ vs.\ $79.20\%$). As discussed, this reflects router stability under Chamfer training rather than better compositional retention: Chamfer gradients encourage
more discriminative per-class slot assignments that are individually more stable across sessions, even as they overspecialise to training classes. COMPOSE-CT's lower FF thus does not contradict the noc finding on CFST—it reflects a different axis of evaluation.

\begin{table}[ht]
\caption{Comparison with FSCIL methods
(mean $\pm$ 95\% CI, 3 seeds: 42, 123, 7).
$A_T$: accuracy over all seen classes at the final session ($\uparrow$), as in Table~\ref{tab:cil_vit}.
FF: forgetting factor ($\downarrow$).
CIFAR-100: 60 base classes (60\%) $+$ 8 incremental sessions $\times$ 5-way 5-shot (100 total).
CUB-200: 100 base classes (50\%) $+$ 10 incremental sessions $\times$ 10-way 5-shot (200 total).
miniImageNet: 60 base classes (60\%) $+$ 8 incremental sessions $\times$ 5-way 5-shot (100 total).
All images resized to 224$\times$224.
\textbf{Bold} = best $A_T$ per column; \underline{Underline} = lowest FF per column.}
\label{tab:cil_fscil}
\centering
\resizebox{\textwidth}{!}{%
\begin{tabular}{lcccccc}
\toprule
& \multicolumn{2}{c}{\textbf{CIFAR-100}}
& \multicolumn{2}{c}{\textbf{CUB-200}}
& \multicolumn{2}{c}{\textbf{miniImageNet}} \\
\cmidrule(lr){2-3}\cmidrule(lr){4-5}\cmidrule(lr){6-7}
\textbf{Method}
    & $A_T$$\uparrow$ & FF$\downarrow$
    & $A_T$$\uparrow$ & FF$\downarrow$
    & $A_T$$\uparrow$ & FF$\downarrow$ \\
\midrule
Comp-FSCIL~\cite{Comp-FSCIL}
    & $74.97_{\pm0.02}$ & $27.91_{\pm0.15}$
    & $79.54_{\pm0.03}$ & $14.69_{\pm0.14}$
    & $\mathbf{94.22}_{\pm0.01}$ & $4.40_{\pm0.00}$ \\
\midrule
\textbf{COMPOSE} (ours)
    & $\mathbf{84.02}_{\pm0.86}$ & $3.54_{\pm0.22}$
    & $79.20_{\pm0.41}$ & $2.25_{\pm0.20}$
    & $94.04_{\pm0.07}$ & $2.22_{\pm0.82}$ \\
COMPOSE-CT (ablation)
    & $83.32_{\pm0.62}$ & $\underline{3.11}_{\pm0.47}$
    & $\mathbf{80.26}_{\pm0.66}$ & $\underline{1.73}_{\pm0.28}$
    & $92.97_{\pm0.27}$ & $\underline{1.72}_{\pm0.40}$ \\
\bottomrule
\end{tabular}%
}
\end{table}

\subsection{Comparison with other Pretrained Model Based CIL Methods}

Table~\ref{tab:cil_vit} compares COMPOSE against RanPAC~\citep{ranpac},
FeCAM~\citep{fecam}, and F-OAL~\citep{foal} on CIFAR-100, ImageNet-R, and ImageNet-A under a standard class-incremental learning protocol~\citep{zhou2024class,zhou2024continual,sun2025pilot}.
CIFAR-100 uses 10 base classes followed by 9 incremental sessions of 10 classes
each (total 100 classes); ImageNet-R and ImageNet-A each use 20 base classes
followed by 9 incremental sessions of 20 classes each (total 200 classes,
resized to $224\times224$).
All methods, including COMPOSE and COMPOSE-CT, start from the same DINOv2 ViT-B/14 checkpoint and fine-tune
the backbone on the base classes at the base session, following the LAMDA-PILOT training
procedure~\citep{sun2025pilot}, so that no method gets a better-adapted backbone than another.
Subsequent incremental sessions follow each method's own update rule. COMPOSE and COMPOSE-CT retain
20 exemplars per class as a replay buffer throughout incremental sessions.

\paragraph{Final-session accuracy ($A_T$).}
On CIFAR-100, RanPAC achieves the highest $A_T$ ($86.61\%$), with COMPOSE
close behind ($86.52\%$, $-0.09$ pp).
On ImageNet-R, FeCAM leads at $83.12\%$, followed by COMPOSE ($82.30\%$,
$-0.82$ pp).
On ImageNet-A, COMPOSE achieves the highest $A_T$ ($70.31\%$), outperforming
F-OAL ($67.22\%$, $+3.09$ pp), FeCAM ($66.95\%$), and RanPAC ($63.99\%$).

\paragraph{Average accuracy ($\overline{\mathrm{AA}}$).}
COMPOSE achieves the highest $\overline{\mathrm{AA}}$ on CIFAR-100 ($91.36\%$)
and ImageNet-A ($79.55\%$).
On ImageNet-R, FeCAM leads ($87.78\%$) with F-OAL ($87.04\%$) and
RanPAC ($86.31\%$) ahead of COMPOSE ($86.13\%$).

\paragraph{Forgetting (FF).}
RanPAC exhibits the lowest FF on CIFAR-100 ($7.30\%$); FeCAM achieves the
lowest on ImageNet-R ($4.40\%$); F-OAL achieves the lowest on ImageNet-A
($0.65\%$), with COMPOSE at $4.56\%$.
COMPOSE-CT consistently shows higher FF than COMPOSE across all datasets,
consistent with Chamfer training producing more rigid per-class router
assignments that are individually stable but less robust to distributional
shift across sessions.

\paragraph{Relation to CFST results.}
The CIL results suggest that COMPOSE's frozen-backbone design generalises
beyond the compositional setting of CFST: it remains competitive on standard
benchmarks where compositionality is not explicitly tested.
The shared mechanism is that the backbone (fine-tuned once at the base session on CIL) and the
DINOSAUR slot attention stay frozen across incremental sessions, which prevents prototype geometry from
drifting regardless of the evaluation protocol.

\begin{table}[ht]
\caption{Comparison with ViT-based CIL methods on standard benchmarks (single seed 42).
  $A_T = \Pr[\hat{y}=y \mid (x,y)\sim\mathcal{D}_{0:T}^{\mathrm{te}}]$:
  accuracy on the cumulative test set after the final session $T$ ($\uparrow$).
  $\overline{\mathrm{AA}} = \frac{1}{T+1}\sum_{t=0}^{T} A_t$:
  average accuracy across all sessions ($\uparrow$).
  FF: drop in base-class accuracy from session~0 to session~$T$ ($\downarrow$).
  CIFAR-100\textsuperscript{$\dagger$}: 10 sessions $\times$ 10-way (100 total classes; session~1 is the base session on the first 10 classes, sessions~2--10 are incremental; all sessions update each method's trainable components),
  images resized to $224\times224$.
  ImageNet-R/A: 10 sessions $\times$ 20-way (200 total classes; same convention).
  All methods use DINOv2 ViT-B/14 as backbone.
  \textbf{Bold} = best per column; \underline{Underline} = lowest FF per column.}
\label{tab:cil_vit}
\centering
\resizebox{\textwidth}{!}{%
\begin{tabular}{lcccccccccc}
\toprule
  & \multicolumn{3}{c}{\textbf{CIFAR-100\textsuperscript{$\dagger$}}}
  & \multicolumn{3}{c}{\textbf{ImageNet-R}}
  & \multicolumn{3}{c}{\textbf{ImageNet-A}} \\
\cmidrule(lr){2-4}\cmidrule(lr){5-7}\cmidrule(lr){8-10}
\textbf{Method}
  & $A_T$$\uparrow$ & $\overline{\mathrm{AA}}$$\uparrow$ & FF$\downarrow$
  & $A_T$$\uparrow$ & $\overline{\mathrm{AA}}$$\uparrow$ & FF$\downarrow$
  & $A_T$$\uparrow$ & $\overline{\mathrm{AA}}$$\uparrow$ & FF$\downarrow$ \\
\midrule
RanPAC~\cite{ranpac}
  & $\mathbf{86.61}$ & $91.21$ & \underline{$7.30$}
  & $81.60$ & $86.31$ & $8.18$
  & $63.99$ & $76.29$ & $2.28$ \\
FeCAM~\cite{fecam}
  & $86.09$ & $90.44$ & $9.20$
  & $\mathbf{83.12}$ & $\mathbf{87.78}$ & \underline{$4.40$}
  & $66.95$ & $77.70$ & $2.28$ \\
F-OAL~\cite{foal}
  & $86.44$ & $91.33$ & $7.40$
  & $81.70$ & $87.04$ & $6.45$
  & $67.22$ & $77.35$ & \underline{$0.65$} \\
\midrule
\textbf{COMPOSE} (ours)
  & $86.52$ & $\mathbf{91.36}$ & $7.70$
  & $82.30$ & $86.13$ & $5.97$
  & $\mathbf{70.31}$ & $\mathbf{79.55}$ & $4.56$ \\
COMPOSE-CT (ablation)
  & $86.02$ & $90.71$ & $7.80$
  & $80.57$ & $84.92$ & $6.13$
  & $69.72$ & $79.49$ & $6.51$ \\
\bottomrule
\end{tabular}%
}
\end{table}

\section{Continual Training Phase Analysis (CFST)}
\begin{table}[ht]
\caption{Continual learning phase metrics across backbone variants on CGQA
($T{=}10$ sessions, single seed 42).
$A_\mathrm{con}$: accuracy on each task's own test set averaged across tasks
using final model $P_T$ ($\uparrow$).
$\overline{\mathrm{AA}}$: average accuracy over all seen classes across
sessions ($\uparrow$).
$H_a$: harmonic mean of all 5 CFST modes at 10-shot ($\uparrow$).
FF: forgetting factor ($\downarrow$).
\textbf{Bold} = best per column; \underline{Underline} = lowest FF.}
\label{tab:cl_phase_backbone}
\centering
\begin{tabular}{lcccc}
\toprule
\textbf{Backbone} & $A_\mathrm{con}$$\uparrow$ & $\overline{\mathrm{AA}}$$\uparrow$ & $H_a$$\uparrow$ & FF$\downarrow$ \\
\midrule
DINOv1 ViT-B/16 (IN-1K)
  & 56.87 & 68.45 & 87.94 & 13.70 \\
CLIP ViT-B/16 (OpenAI)
  & 60.98 & 70.95 & 89.22 & 11.65 \\
iBOT ViT-B/16 (IN-1K)
  & 64.04 & 74.22 & 91.15 & 10.84 \\
\textbf{DINOv2 ViT-B/14}
  & \textbf{76.23} & \textbf{83.38} & \textbf{92.31} & \underline{7.56} \\
\bottomrule
\end{tabular}%
\end{table}

\begin{table}[ht]
\caption{Continual learning phase metrics on CGQA ($T{=}10$ sessions).
$\overline{\mathrm{AA}}$: average accuracy over all seen classes across sessions ($\uparrow$).
FF: forgetting factor ($\downarrow$).
\textsuperscript{†}~frozen backbone; \textsuperscript{‡}~backbone fine-tuned each session.}
\label{tab:cl_phase}
\centering
\resizebox{0.35\textwidth}{!}{%
\begin{tabular}{lcc}
\toprule
\textbf{Method} & $\overline{\mathrm{AA}}$$\uparrow$ & FF$\downarrow$ \\
\midrule
EASE\textsuperscript{†}
  & 79.50 & 11.39 \\
\textbf{CoFiMA}\textsuperscript{‡}
  & \textbf{85.71} & 11.53 \\
\midrule
COMPOSE\textsuperscript{†} (ours)
  & 83.38 & \underline{7.56} \\
\bottomrule
\end{tabular}%
}
\end{table}

CFST evaluation proceeds in two phases: a continual learning (CL) phase in
which the model trains sequentially on $T$ sessions, followed by a
frozen-feature few-shot evaluation over the compositional splits reported in
Table~\ref{tab:main}. The few-shot phase is task-incremental (TIL): each episode provides
an $N$-way support set as a task oracle. Here we examine Phase~I.
\textbf{Average accuracy} ($\overline{\mathrm{AA}}$) is the mean over all
seen classes across sessions.
\textbf{Forgetting factor} (FF) is the base-class accuracy drop from
session~0 to session~$T$.

\paragraph{Backbone quality and CL metrics.}
Table~\ref{tab:cl_phase_backbone} reports CL metrics for COMPOSE across four
backbone variants, each with an adapter trained on the base session only and
frozen thereafter.
DINOv2 achieves the highest $\overline{\mathrm{AA}}$ ($83.38\%$), $A_\mathrm{con}$
($76.23\%$), and $H_a$ ($92.31\%$) with the lowest FF ($7.56\%$), consistent
with its stronger patch-level representations from local self-distillation.
The ordering DINOv1 $<$ CLIP $<$ iBOT $<$ DINOv2 holds across all four
metrics, suggesting that SSL pretraining quality is the dominant factor in
both CL retention and few-shot generalisation.
CLIP ViT-B/16, despite broader pretraining data, underperforms iBOT
($89.22\%$ vs.\ $91.15\%$ $H_a$), consistent with the noc analysis in
Section~6.2.

\paragraph{Comparison with backbone fine-tuning methods.}
Table~\ref{tab:cl_phase} compares COMPOSE against EASE, which trains adapters
at every session with a frozen backbone, and CoFiMA, which fine-tunes the full
backbone at every session via Fisher-weighted parameter merging.
COMPOSE adapts to the target domain in a single base session and achieves
$\overline{\mathrm{AA}} = 83.38\%$, within $2.33$ pp of CoFiMA ($85.71\%$),
while exhibiting substantially lower FF ($7.56\%$ vs.\ $11.53\%$).
This suggests that one-time adaptation at the base session is sufficient to
capture domain-specific structure, and that updating the backbone at every
incremental session introduces unnecessary representation drift.

\section{Gradient Alignment: Theory, Metric, and Training Dynamics}
\label{app:gradient_extended}

\subsection{Formal Definition of the Gradient Alignment Metric}
\label{app:salign_def}

Let $s(q, c)$ denote the query-to-class scoring function (holistic,
Chamfer, or blended) and let $z_k^{(q)}$ denote the $k$-th projected slot
embedding of query image $q$.
Define the \emph{inter-slot gradient alignment} as:
\begin{equation}
    S_{kk'} = \mathbb{E}_{(q,c)}\left[
        \cos\!\left(
            \frac{\partial s(q,c)}{\partial z_k^{(q)}},\;
            \frac{\partial s(q,c)}{\partial z_{k'}^{(q)}}
        \right)
    \right],
    \quad k \neq k',
    \label{eq:salign}
\end{equation}
where the expectation is taken over query--class pairs $(q, c)$ from a held-out
evaluation set.
$S_{kk'} = +1$ means both slots receive identical gradient directions for every
episode, encoding no slot-specific information into the score.
$S_{kk'} < +1$ means at least some episodes produce divergent gradient
directions across slots, encoding slot-specific information.

In practice we compute $S_{kk'}$ by averaging over all $\binom{K}{2} = 21$
slot pairs ($K=7$) and over 300 evaluation episodes with a fixed checkpoint.
Table~\ref{tab:soft_chamfer} in the main appendix reports this quantity for COMPOSE, COMPOSE-CT, and the soft Chamfer variants.

\begin{table}[ht]
\centering
\caption{Gradient alignment $S_{kk'}$ and accuracy for soft Chamfer variants ($\beta$ sweep)
vs.\ \textbf{COMPOSE} and \textbf{COMPOSE-CT} (hard), on CGQA\@.
$S_{kk'}\uparrow$ means more uniform slot gradients (holistic);
$S_{kk'}\downarrow$ means more slot-specific gradients.
Accuracy is from a 3-session mini-benchmark (shorter than the 10-session protocol of Table~\ref{tab:main_reduce}), so it is comparable across rows of this table only.}
\label{tab:soft_chamfer}
\resizebox{\textwidth}{!}{%
\begin{tabular}{llcccc}
\toprule
\textbf{Variant} & \textbf{Training}
    & $S_{kk'}$ (full)
    & $S_{kk'}$ (hol.)
    & \textbf{noc (\%)}
    & \textbf{sys (\%)} \\
\midrule
\textbf{COMPOSE}            & Holistic only             & $+0.186$ & $+0.186$ & \textbf{92.02} & 86.81 \\
\midrule
\textbf{COMPOSE-CT} (hard)  & $+$Chamfer ($\beta\to\infty$) & $+0.048$ & $+0.081$ & 87.85 & 88.89 \\
Soft ($\beta{=}50$)         & $+$Chamfer                & $+0.045$ & $+0.087$ & 87.98 & 88.86 \\
Soft ($\beta{=}20$)         & $+$Chamfer                & $+0.045$ & $+0.092$ & 88.25 & 88.94 \\
Soft ($\beta{=}10$)         & $+$Chamfer                & $+0.047$ & $+0.079$ & 88.07 & 89.01 \\
Soft ($\beta{=}5$)          & $+$Chamfer                & $+0.081$ & $+0.088$ & 86.76 & 89.03 \\
Soft ($\beta{=}1$)          & $+$Chamfer                & $+0.376$ & $+0.118$ & 82.57 & 87.04 \\
\bottomrule
\end{tabular}%
}
\end{table}

\subsection{Connection Between \texorpdfstring{$S_{kk'}$}{S\_\{kk'\}} and Novel-Concept Generalization}
\label{app:salign_noc}

\paragraph{High $S_{kk'}$ promotes compositional invariance.}
When $S_{kk'}$ is high (close to $+1$), the gradient signal pushes all slots
in the same direction for a given episode.
This means the \emph{per-slot contribution to the loss is indistinguishable}:
no slot is systematically favoured or penalised relative to others.
Across many episodes, the router therefore learns to project slots into a
shared representational geometry where \emph{which particular combination of
slots fires} is less important than their aggregate statistics.
This is precisely the invariance required for noc generalisation: a novel
concept presents a new \emph{combination} of semantic parts, and the model
must recognise it without having seen that specific combination.
If the score function is insensitive to which slots contribute (high $S_{kk'}$),
the model is more tolerant of novel slot-to-part assignments.

\paragraph{Low $S_{kk'}$ promotes slot overspecialization.}
When $S_{kk'}$ is low, the gradient signal is slot-specific.
Through many training episodes, the router learns to assign \emph{fixed roles}
to specific slots: slot $k$ is expected to fire on the dog's head, slot $k'$
on its body.
This slot specialization improves seen-concept discriminability (sys, pro)
because the structured matching is consistent within the training distribution.
However, it breaks at test time for noc concepts: a novel combination of a
dog's head on a cat's body activates slots in an order not seen during training,
and the specialized slot representations fail to match.

\paragraph{Empirical validation.}
Table~\ref{tab:soft_chamfer} (main appendix) reports the following ordering by noc
(descending):
COMPOSE ($S_{kk'} = +0.186$, noc $= 92.02$) $>$
Soft-$\beta{=}20$ ($+0.045$, noc $= 88.25$) $\approx$
COMPOSE-CT ($+0.048$, noc $= 87.85$) $>$
Soft-$\beta{=}5$ ($+0.081$, noc $= 86.76$) $>$
Soft-$\beta{=}1$ ($+0.376$, noc $= 82.57$).

\noindent
$S_{kk'}$ is \emph{not} a monotonic function of noc across the full table. Instead,
the variants split into two regimes determined by whether the gradient signal is
\emph{structured} or \emph{confused}, and within each regime $S_{kk'}$ aligns with
noc as predicted:
\begin{itemize}[leftmargin=1.5em, itemsep=0.2em]
    \item \textit{Structured-gradient regime} (COMPOSE and Chamfer-trained variants
    with $\beta \in \{5, 20, \infty\}$): higher $S_{kk'}$ corresponds to higher noc.
    COMPOSE's holistic-only training reaches the highest values on both axes
    ($S_{kk'} = +0.186$, noc $= 92.02$), while the Chamfer-trained variants cluster
    at low $S_{kk'}$ (${\sim}0.045$--$0.081$) and lower noc, consistent with
    Lemma~\ref{lem:sinkhorn_low_eps} and Corollary~\ref{cor:assignment_gradient}:
    nearest-neighbor matching during training lowers $S_{kk'}$ by inducing
    slot-specific gradients, and this slot specialisation degrades noc.
    \item \textit{Confused-gradient regime} (Soft-$\beta{=}1$): $S_{kk'}$ is
    anomalously high ($+0.376$) yet noc is the lowest in the table ($82.57$).
    At $\beta{=}1$, soft weights are nearly uniform across all support slots, so
    each slot's gradient contains an \emph{interference term} that simultaneously
    pulls it toward all support slots, including irrelevant ones. This uniform
    interference inflates $S_{kk'}$ not by aligning gradients toward a
    compositionally invariant signal (as in holistic training) but by making them
    maximally confused. Proposition~\ref{prop:high_eps} characterises this as the
    high-$\varepsilon$ degeneracy in which the matcher ceases to function as a
    matcher.
\end{itemize}
The empirical pattern therefore validates the theoretical prediction within each
regime, while the cross-regime non-monotonicity reflects the qualitative distinction
between structured holistic gradients and unstructured confused ones; $S_{kk'}$ must
be interpreted jointly with whether the gradient signal is meaningfully structured.

\subsection{Formal Statement and Proof of Gradient Divergence}
\label{app:gradient_divergence}

\begin{proposition}
\label{prop:gradient}
Let $\tilde\phi_k^{(q)} \in \mathbb{R}^{D}$ denote the $k$-th attention-weighted backbone
aggregate of query image $q$ ($D = 768$ for DINOv2 ViT-B/14). Let
$W_2 \in \mathbb{R}^{D \times D}$ be the shared projection head of
Section~\ref{subsec:phase1}, and define the \emph{unnormalized projection}
\begin{equation}
\check y_k^{(q)} \;=\; W_2 \, \tilde\phi_k^{(q)} \;\in\; \mathbb{R}^{D},
\label{eq:ycheck_def}
\end{equation}
together with the \emph{unit-norm slot embedding}
$z_k^{(q)} = \check y_k^{(q)} / \|\check y_k^{(q)}\| \in \mathcal{S}^{D-1}$.
Let $\omega_k^{(q)} \in [0,1]$ with $\sum_k \omega_k^{(q)} = 1$ be the simplex importance
weights produced by the router MLP on the raw aggregates $\tilde\phi_k^{(q)}$
(independent of $\check y_k^{(q)}$). The router-weighted aggregate and the holistic
embedding are
\begin{equation}
u^{(q)} \;=\; \sum_{k=1}^{K} \omega_k^{(q)} \, \check y_k^{(q)}, \qquad
e^{(q)} \;=\; \frac{u^{(q)}}{\|u^{(q)}\|} \;\in\; \mathcal{S}^{D-1}.
\label{eq:u_e_def}
\end{equation}

\begin{enumerate}[label=(\roman*)]
    \item \textbf{Holistic score.} For $s_{\mathrm{hol}}(q,c) = \langle e^{(q)}, P_c \rangle$,
    the gradient with respect to each unnormalized projection factorizes as
    \begin{equation}
    \frac{\partial s_{\mathrm{hol}}}{\partial \check y_k^{(q)}}
    \;=\; \frac{\omega_k^{(q)}}{\|u^{(q)}\|}\, \Pi^{\perp}_{e^{(q)}}(P_c).
    \label{eq:grad_holistic}
    \end{equation}
    The direction $\Pi^{\perp}_{e^{(q)}}(P_c)$ is a \emph{single signal vector shared
    across all slots}; the slot index $k$ enters only through the non-negative scalar
    $\omega_k^{(q)} / \|u^{(q)}\|$. Consequently, the per-slot gradients
    $\{\partial s_{\mathrm{hol}}/\partial \check y_k^{(q)}\}_{k=1}^{K}$ are collinear,
    so the gradient field has rank at most one.

    \item \textbf{Chamfer score.} For the forward Chamfer score
    $s_{\mathrm{Ch}}(q,c) = \tfrac{1}{K}\sum_{k} \max_{k'} \cos(z_k^{(q)}, z^{(c)}_{k'})$
    with matched index $k^{*}(k) = \arg\max_{k'} \cos(z_k^{(q)}, z^{(c)}_{k'})$ held
    locally constant, the gradient admits two equivalent forms. At the unit-norm slot
    embedding $z_k^{(q)} \in \mathcal{S}^{D-1}$,
    \begin{equation}
    \frac{\partial s_{\mathrm{Ch}}}{\partial z_k^{(q)}}
    \;=\; \frac{1}{K}\, \Pi^{\perp}_{z_k^{(q)}}\!\left(z^{(c)}_{k^{*}(k)}\right),
    \label{eq:grad_chamfer_z}
    \end{equation}
    and at the unnormalized projection $\check y_k^{(q)}$,
    \begin{equation}
    \frac{\partial s_{\mathrm{Ch}}}{\partial \check y_k^{(q)}}
    \;=\; \frac{1}{K\,\|\check y_k^{(q)}\|}\, \Pi^{\perp}_{z_k^{(q)}}\!\left(z^{(c)}_{k^{*}(k)}\right).
    \label{eq:grad_chamfer}
    \end{equation}
    In either form, $k$ enters through the slot-specific projection anchor
    $\Pi^{\perp}_{z_k^{(q)}}$ and the matched target $z^{(c)}_{k^{*}(k)}$ (and, in
    Eq.~\eqref{eq:grad_chamfer}, additionally through the inverse magnitude
    $1/\|\check y_k^{(q)}\|$), yielding $K$ generically independent per-slot directions.
    The gradient field generically has rank $K$.
\end{enumerate}
\end{proposition}

\begin{proof}
All gradients are ambient Euclidean gradients in $\mathbb{R}^{D}$.

\noindent\textbf{Part (i).} By the chain rule,
\begin{equation}
\frac{\partial s_{\mathrm{hol}}}{\partial \check y_k^{(q)}}
\;=\; \left(\frac{\partial u^{(q)}}{\partial \check y_k^{(q)}}\right)^{\!\top}\!
      \left(\frac{\partial e^{(q)}}{\partial u^{(q)}}\right)^{\!\top}\!
      \frac{\partial \langle e^{(q)}, P_c\rangle}{\partial e^{(q)}}.
\end{equation}
The three Jacobians are
\begin{align}
\frac{\partial \langle e^{(q)}, P_c\rangle}{\partial e^{(q)}} &= P_c, \\
\frac{\partial e^{(q)}}{\partial u^{(q)}} &= \frac{1}{\|u^{(q)}\|}\bigl(I - e^{(q)}(e^{(q)})^{\!\top}\bigr) = \frac{1}{\|u^{(q)}\|}\,\Pi^{\perp}_{e^{(q)}}, \\
\frac{\partial u^{(q)}}{\partial \check y_k^{(q)}} &= \omega_k^{(q)}\, I,
\end{align}
where the last identity uses that $\omega_k^{(q)}$ is produced by the router MLP with
parameters $(W_1, v)$ disjoint from the projection head $W_2$, acting on the frozen
backbone aggregate $\tilde\phi_k^{(q)}$; consequently $\omega_j^{(q)}$ has no functional
dependence on $\check y_k^{(q)} = W_2\tilde\phi_k^{(q)}$ for any $j,k$. Composing and using symmetry of $\Pi^{\perp}_{e^{(q)}}$,
\begin{equation}
\frac{\partial s_{\mathrm{hol}}}{\partial \check y_k^{(q)}}
\;=\; \omega_k^{(q)} \cdot \frac{1}{\|u^{(q)}\|}\,\Pi^{\perp}_{e^{(q)}}(P_c)
\;=\; \frac{\omega_k^{(q)}}{\|u^{(q)}\|}\,\Pi^{\perp}_{e^{(q)}}(P_c).
\end{equation}
The vector $\Pi^{\perp}_{e^{(q)}}(P_c)$ is independent of $k$: it is determined entirely
by the aggregate $e^{(q)}$ and the class prototype $P_c$. All $K$ per-slot gradients are
therefore non-negative scalar multiples of this single vector, so the gradient field
has rank at most one.

\noindent\textbf{Part (ii).} Each summand in $s_{\mathrm{Ch}}$ depends on exactly one
query slot:
\begin{equation}
s_{\mathrm{Ch}}(q,c) \;=\; \frac{1}{K}\sum_{j=1}^{K}\cos\bigl(z_j^{(q)}, z^{(c)}_{k^{*}(j)}\bigr).
\end{equation}
Treat $k^{*}(k)$ as locally constant (the standard subgradient for argmax,
corresponding to the gradient computed by autodifferentiation). We first establish the
$z_k^{(q)}$-form, then chain through L2-normalization to obtain the
$\check y_k^{(q)}$-form.

\emph{Gradient at $z_k^{(q)}$.} Only the $j = k$ summand depends on $z_k^{(q)}$, and the
Euclidean gradient of cosine similarity evaluated at the unit vector $z_k^{(q)}$ is
$\partial \cos(z_k^{(q)}, v)/\partial z_k^{(q)} = \Pi^{\perp}_{z_k^{(q)}}(v)$. Thus
\begin{equation}
\frac{\partial s_{\mathrm{Ch}}}{\partial z_k^{(q)}}
\;=\; \frac{1}{K}\,\Pi^{\perp}_{z_k^{(q)}}\!\left(z^{(c)}_{k^{*}(k)}\right),
\end{equation}
which establishes Eq.~\eqref{eq:grad_chamfer_z}.

\emph{Gradient at $\check y_k^{(q)}$.} By the chain rule through
$z_k^{(q)} = \check y_k^{(q)}/\|\check y_k^{(q)}\|$ and the L2-normalization Jacobian
$\partial z_k^{(q)}/\partial \check y_k^{(q)} = \|\check y_k^{(q)}\|^{-1}\,
\Pi^{\perp}_{z_k^{(q)}}$ at $\check y_k^{(q)} \neq 0$,
\begin{equation}
\frac{\partial s_{\mathrm{Ch}}}{\partial \check y_k^{(q)}}
\;=\; \left(\frac{\partial z_k^{(q)}}{\partial \check y_k^{(q)}}\right)^{\!\top}\!
      \frac{\partial s_{\mathrm{Ch}}}{\partial z_k^{(q)}}
\;=\; \frac{1}{K\,\|\check y_k^{(q)}\|}\, \Pi^{\perp}_{z_k^{(q)}} \Pi^{\perp}_{z_k^{(q)}}\!\left(z^{(c)}_{k^{*}(k)}\right),
\end{equation}
and since $\Pi^{\perp}_{z_k^{(q)}}$ is symmetric and idempotent,
\begin{equation}
\frac{\partial s_{\mathrm{Ch}}}{\partial \check y_k^{(q)}}
\;=\; \frac{1}{K\,\|\check y_k^{(q)}\|}\, \Pi^{\perp}_{z_k^{(q)}}\!\left(z^{(c)}_{k^{*}(k)}\right).
\end{equation}
Distinct slots $k \neq k'$ have generically distinct projection anchors
$\Pi^{\perp}_{z_k^{(q)}} \neq \Pi^{\perp}_{z_{k'}^{(q)}}$ and distinct magnitudes
$\|\check y_k^{(q)}\|$. Their matched targets need \emph{not} differ: two query slots
share a nearest support slot on an open set of configurations of positive measure.
Slot-specificity does not rely on distinct targets, however. The distinct anchors alone
make the per-slot gradients span a $K$-dimensional subspace on a full-measure subset of
configurations, including those where targets are shared
(Proposition~\ref{prop:genericity}).
\end{proof}

\begin{remark}[Gradient reference variable]
\label{rem:grad_reference_variable}
Part (i) is stated at $\check y_k^{(q)} = W_2 \tilde\phi_k^{(q)}$ because the holistic
objective aggregates $\check y_k^{(q)}$ directly via Eq.~\eqref{eq:u_e_def}, so
$\check y_k^{(q)}$ is the natural reference for the exact rank-one claim. Part (ii) is
stated at both the unit-norm slot embedding $z_k^{(q)}$ (the variable on which
assignment-based matchers act directly; used by Corollary~\ref{cor:assignment_gradient}
and Appendix~\ref{app:sinkhorn_rigorous}) and at $\check y_k^{(q)}$ (the common upstream
variable shared with Part (i)).

The structural distinction --- rank-one upstream for holistic versus rank-$K$ for
Chamfer at either reference --- is robust to the choice of variable. The \emph{exact}
rank-one structure of the holistic gradient field, however, is established at
$\check y_k^{(q)}$ and is not pointwise preserved when transported to $z_k^{(q)}$: the
chain rule through L2-normalization applies a slot-specific tangent projection
$\Pi^{\perp}_{z_k^{(q)}}$ to the shared upstream signal $\Pi^{\perp}_{e^{(q)}}(P_c)$,
yielding
\begin{equation}
\frac{\partial s_{\mathrm{hol}}}{\partial z_k^{(q)}}
\;=\; \frac{\omega_k^{(q)}\,\|\check y_k^{(q)}\|}{\|u^{(q)}\|}\,
      \Pi^{\perp}_{z_k^{(q)}}\!\bigl(\Pi^{\perp}_{e^{(q)}}(P_c)\bigr).
\label{eq:grad_holistic_z}
\end{equation}
Because the tangent hyperplanes $\{T_{z_k^{(q)}}\mathcal{S}^{D-1}\}_{k=1}^{K}$ differ
across slots, the projected vectors are not pointwise collinear in $\mathbb{R}^{D}$.
They nevertheless retain a shared upstream component along $\Pi^{\perp}_{e^{(q)}}(P_c)$,
yielding pairwise cosine similarities that are strictly positive yet bounded away from
$+1$. This is precisely what the empirical alignment metric $S_{kk'}$
(Eq.~\eqref{eq:salign}, Appendix~\ref{app:salign_def}) measures at $z_k^{(q)}$:
\textbf{Compose} reports $S_{kk'} = +0.186$ in Table~\ref{tab:soft_chamfer} ---
strictly positive, reflecting the shared upstream signal, but well below $+1$,
reflecting the projection-induced spread. By contrast, Chamfer training drives
$S_{kk'}$ toward zero ($+0.048$ for \textbf{Compose-CT}), consistent with the
structurally distinct slot-specific targets $z^{(c)}_{k^*(k)}$ in
Part~(ii).\footnote{The $\check y_k^{(q)}$-form of Part~(i) remains pointwise
rank-one: the additional projection step is introduced only when the gradient is
re-expressed at $z_k^{(q)}$ to match the reference frame of assignment-based matchers.
The contrast between the two parts is therefore stated most cleanly at
$\check y_k^{(q)}$, while empirical alignment is reported at $z_k^{(q)}$ for direct
comparability with the Chamfer family.}
\end{remark}

\paragraph{Bidirectional Chamfer.} For the bidirectional score of
Eq.~\eqref{eq:chamfer_inf}, the backward term contributes an additional slot-specific
summand whenever query slot $k$ is the nearest neighbor of some support slot. This
strengthens rather than weakens the slot-specificity established above: the gradient
field remains generically rank $K$, with the same structural dependence on both
projection anchor and matched target.

\begin{corollary}[Slot-specific gradients under assignment-based matching]
\label{cor:assignment_gradient}
Let $T \in \mathbb{R}^{K \times K}$ denote a coupling matrix with $T_{k,k'} \geq 0$ and 
$\sum_{k'} T_{k,k'} = 1$ for each $k$. The family of assignment-based scores
\begin{equation}
s_T(q, c) = \frac{1}{K} \sum_{k=1}^{K} \sum_{k'=1}^{K} T_{k,k'} \cos\bigl(z_k^{(q)}, z_{k'}^{(c)}\bigr)
\label{eq:assignment_score}
\end{equation}
subsumes the following matchers as special cases:
\begin{itemize}
    \item \textbf{Hard Chamfer:} $T_{k,k'} = \mathbf{1}\bigl[k' = \arg\max_{j} \cos(z_k^{(q)}, z_j^{(c)})\bigr]$.
    \item \textbf{Soft Chamfer} (temperature $\beta$): $T_{k,k'} = \mathrm{softmax}_{k'}\bigl(\beta \cdot \cos(z_k^{(q)}, z_{k'}^{(c)})\bigr)$.
    \item \textbf{Mutual nearest-neighbor:} $T_{k,k'} = \mathbf{1}[k' = \mathrm{NN}(k) \wedge k = \mathrm{NN}(k')]$, row-normalized.
    \item \textbf{Entropic optimal transport} (Sinkhorn, regularization $\varepsilon$): $T = \arg\min_{T} \langle T, -S \rangle - \varepsilon H(T)$, with $S_{k,k'} = \cos(z_k^{(q)}, z_{k'}^{(c)})$.
    \item \textbf{Differentiable Hungarian:} $T$ is a doubly-stochastic relaxation of a permutation matrix.
\end{itemize}
The gradient of $s_T$ with respect to query slot $z_k^{(q)}$, treating $T$ as locally constant, is
\begin{equation}
\frac{\partial s_T}{\partial z_k^{(q)}} = \frac{1}{K} \, \Pi_{z_k^{(q)}}^{\perp} \!\left( \sum_{k'=1}^{K} T_{k,k'} \, z_{k'}^{(c)} \right).
\label{eq:assignment_gradient}
\end{equation}
The $k$-dependence is two-fold: (i) the projection anchor $\Pi_{z_k^{(q)}}^{\perp}$ is 
anchored at the query slot, and (ii) the coupling row $T_{k,:}$ specifies a 
slot-specific convex combination of support targets. Consequently,
$\partial s_T / \partial z_k^{(q)} \neq \partial s_T / \partial z_{k'}^{(q)}$ whenever
$T_{k,:} \neq T_{k',:}$. Equal rows do occur with positive probability (for hard Chamfer,
whenever two query slots share a nearest support slot), but even then the two gradients
differ through their anchors: $\Pi^{\perp}_{z_k^{(q)}}(t) = \Pi^{\perp}_{z_{k'}^{(q)}}(t)$
for a shared target $t$ requires $\langle t, z_k^{(q)}\rangle z_k^{(q)} = \langle t, z_{k'}^{(q)}\rangle z_{k'}^{(q)}$,
a measure-zero event. Proposition~\ref{prop:genericity} makes this precise.\footnote{
For differentiable $T$ (soft Chamfer, Sinkhorn, Hungarian relaxations), treating $T$ as 
locally constant corresponds to the ``detached'' gradient through matching weights; 
including $\partial T / \partial z_k^{(q)}$ introduces additional slot-specific terms 
(via explicit differentiation of softmax for soft Chamfer, or implicit differentiation 
for Sinkhorn) that do not alter the qualitative slot-specificity.}
\end{corollary}

\begin{proof}
By the chain rule on $z_k^{(q)} \in \mathcal{S}^{D-1}$, the Euclidean gradient of
$\cos(z_k^{(q)}, v)$ evaluated at the unit vector $z_k^{(q)}$ equals
$\Pi_{z_k^{(q)}}^{\perp}(v)$ (identity used in the proof of
Proposition~\ref{prop:gradient}(ii)). Differentiating $s_T$ term-by-term with $T$
fixed, only summands with query index $k$ contribute:
\begin{equation}
\frac{\partial s_T}{\partial z_k^{(q)}} = \frac{1}{K} \sum_{k'} T_{k,k'} \cdot \Pi_{z_k^{(q)}}^{\perp}\bigl(z_{k'}^{(c)}\bigr) = \frac{1}{K} \, \Pi_{z_k^{(q)}}^{\perp}\!\left( \sum_{k'} T_{k,k'} z_{k'}^{(c)} \right),
\end{equation}
by linearity of tangent projection. The $k$-dependence follows from the projection
anchor at $z_k^{(q)}$ and the generically distinct coupling rows $T_{k,:}$.
 
\noindent\emph{Equivalent form at the unnormalized projection.} Composing with the
L2-normalization Jacobian $\partial z_k^{(q)}/\partial \check y_k^{(q)} =
\|\check y_k^{(q)}\|^{-1}\Pi^{\perp}_{z_k^{(q)}}$ (as in the proof of
Proposition~\ref{prop:gradient}(ii)) yields
\begin{equation}
\frac{\partial s_T}{\partial \check y_k^{(q)}} = \frac{1}{K\,\|\check y_k^{(q)}\|}\, \Pi_{z_k^{(q)}}^{\perp}\!\left( \sum_{k'} T_{k,k'} z_{k'}^{(c)} \right),
\end{equation}
which adds a slot-specific scalar $1/\|\check y_k^{(q)}\|$ but preserves all structural
conclusions (rank, anchor, target).
\end{proof}

\begin{remark}[Structural distinction between holistic and matching gradients]
\label{rem:matcher_spectrum}
Within the assignment-based family (Eq.~\eqref{eq:assignment_gradient}), $T$ parameterizes
a continuum of gradient behaviors: permutation-like $T$ (hard Chamfer, Hungarian) yields 
maximally slot-specific gradients through both the projection anchor and a single 
selected target; at the opposite end, uniform $T = \tfrac{1}{K}\mathbf{1}\mathbf{1}^{\top}$ 
yields gradients that symmetrize over support targets---pointing toward the support 
centroid $\bar{z}^{(c)} = \tfrac{1}{K}\sum_{k'} z^{(c)}_{k'}$---but \emph{retain} the 
slot-specific projection anchor $\Pi^{\perp}_{z_k^{(q)}}$. Soft matchers (soft Chamfer 
at finite $\beta$, Sinkhorn at finite $\varepsilon$) interpolate between these extremes 
in the asymptotic limits ($\beta \to \infty$ or $\varepsilon \to 0$ recovers hard 
assignment; $\beta \to 0$ or $\varepsilon \to \infty$ recovers uniform coupling), though 
intermediate values need not yield monotonic behavior in empirical gradient-alignment 
metrics (Appendix~\ref{app:salign_def}).

Critically, uniform $T$ \emph{does not} recover the holistic gradient of 
Proposition~\ref{prop:gradient}(i). The two gradients differ structurally:
\begin{center}
\begin{tabular}{lll}
\toprule
 & \textbf{Holistic (Prop.~\ref{prop:gradient}(i))} & \textbf{Uniform-$T$ matching} \\
\midrule
Projection anchor & $\Pi^{\perp}_{e^{(q)}}$ (at aggregate) & $\Pi^{\perp}_{z_k^{(q)}}$ (at individual slot) \\
Target & $P_c$ (class prototype) & $\bar{z}^{(c)}$ (support centroid) \\
Slot index $k$ enters via & scalar $\omega_k^{(q)}$ only & anchor $\Pi^{\perp}_{z_k^{(q)}}$ \\
Gradient field rank & $\leq 1$ & $K$ \\
\bottomrule
\end{tabular}
\end{center}
These distinctions are not eliminated by any choice of $T$: the holistic objective 
operates on the \emph{aggregated embedding after L2-normalization}, while any 
assignment-based objective operates on \emph{individual slot embeddings}. The two 
induce fundamentally different gradient geometries; holistic training is not a limit 
point of the matching family but a distinct paradigm that engages the aggregate rather 
than per-slot embeddings as the locus of discriminative pressure.
\end{remark}

\subsection{Rigorous Extension of Corollary~\ref{cor:assignment_gradient} to Differentiable Matchers}
\label{app:sinkhorn_rigorous}

Corollary~\ref{cor:assignment_gradient} treats the coupling matrix $T$ as locally constant, which is exact for hard Chamfer (where $T$ is piecewise constant a.e.) but requires justification for differentiable matchers such as soft Chamfer, entropic optimal transport (Sinkhorn), and differentiable Hungarian relaxations, where $T$ depends on the query and support embeddings. We make this extension rigorous here, addressing the footnote in Corollary~\ref{cor:assignment_gradient} regarding the inclusion of $\partial T/\partial z_k^{(q)}$ terms.

\subsubsection{Full Gradient Decomposition}

For any differentiable coupling $T(z^{(q)}, z^{(c)}) \in \mathbb{R}^{K \times K}$ with $T_{k,k'} \geq 0$ and $\sum_{k'} T_{k,k'} = 1$, the assignment-based score is given by Eq.~\eqref{eq:assignment_score}. Applying the product rule and using $\partial \cos(z_j^{(q)}, z_{j'}^{(c)})/\partial z_k^{(q)} = \Pi^{\perp}_{z_k^{(q)}}(z_{j'}^{(c)}) \cdot \mathbb{1}[j = k]$, the full gradient with respect to query slot $z_k^{(q)} \in \mathcal{S}^{D-1}$ decomposes as
\begin{equation}
\frac{\partial s_T}{\partial z_k^{(q)}} = \underbrace{\frac{1}{K}\,\Pi^{\perp}_{z_k^{(q)}}\!\left(\sum_{k'} T_{k,k'}\, z_{k'}^{(c)}\right)}_{G^{\mathrm{dir}}_k \text{ (direct term)}} \;+\; \underbrace{\frac{1}{K} \sum_{j, j'} \frac{\partial T_{j,j'}}{\partial z_k^{(q)}} \cdot \cos\!\left(z_j^{(q)}, z_{j'}^{(c)}\right)}_{G^{\mathrm{ind}}_k \text{ (indirect term)}}.
\label{eq:grad_decomp}
\end{equation}
The direct term $G^{\mathrm{dir}}_k$ recovers the expression in Eq.~\eqref{eq:assignment_gradient} of Corollary~\ref{cor:assignment_gradient} and has rank $K$ generically. The indirect term $G^{\mathrm{ind}}_k$ captures gradient flow through the dependence of $T$ on $z^{(q)}$.

\begin{proposition}[Slot-specificity of the full gradient]
\label{prop:full_slot_specific}
For any differentiable coupling $T$ in the family defined by Corollary~\ref{cor:assignment_gradient}, the full gradient $\partial s_T/\partial z_k^{(q)}$ retains the slot-specific structure of $G^{\mathrm{dir}}_k$ --- i.e., distinct slots $k \neq k'$ receive gradients in distinct directions --- $\mu$-generically, provided the regularization parameter lies in the low-to-moderate regime specified case-by-case below.
\end{proposition}

\subsubsection{Case Analysis}

\paragraph{Case 1: Soft Chamfer with temperature $\beta$.}
The coupling is
\begin{equation}
T_{k,k'} = \frac{\exp\!\left(\beta\, \cos(z_k^{(q)}, z_{k'}^{(c)})\right)}{\sum_{j'} \exp\!\left(\beta\, \cos(z_k^{(q)}, z_{j'}^{(c)})\right)}.
\label{eq:soft_chamfer_T}
\end{equation}
Each row $T_{k,:}$ depends \emph{only} on $z_k^{(q)}$ (not on $z_j^{(q)}$ for $j \neq k$). Consequently,
\begin{equation}
\frac{\partial T_{j,j'}}{\partial z_k^{(q)}} = 0 \quad \text{whenever } j \neq k,
\label{eq:soft_sparse_jacobian}
\end{equation}
and the indirect term simplifies to
\begin{equation}
G^{\mathrm{ind}}_k = \frac{1}{K} \sum_{j'} \frac{\partial T_{k,j'}}{\partial z_k^{(q)}} \cdot \cos\!\left(z_k^{(q)}, z_{j'}^{(c)}\right),
\label{eq:soft_Gind}
\end{equation}
which depends \emph{entirely} on $z_k^{(q)}$ and the support set. The Jacobian $\partial T_{k,j'}/\partial z_k^{(q)}$ inherits slot-specificity from the softmax structure (anchored at $z_k^{(q)}$). Hence $G^{\mathrm{dir}}_k + G^{\mathrm{ind}}_k$ is slot-specific for all $\beta \in (0, \infty)$, and Proposition~\ref{prop:full_slot_specific} holds globally for soft Chamfer.

\paragraph{Case 2: Entropic optimal transport (Sinkhorn) with regularization $\varepsilon$.}
The coupling $T(\varepsilon)$ is the solution of
\begin{equation}
T(\varepsilon) = \arg\min_{T \in \Pi(\mathbf{1},\, \mathbf{1})} \; \langle T, -S \rangle - \varepsilon\, H(T),
\label{eq:sinkhorn_problem}
\end{equation}
where $S_{k,k'} = \cos(z_k^{(q)}, z_{k'}^{(c)})$, $\Pi(\mathbf{1}, \mathbf{1})$ denotes the set of doubly-stochastic couplings with row and column sums equal to $\mathbf{1}$ (matching the row-stochastic convention of Corollary~\ref{cor:assignment_gradient}; the standard $\mathbf{1}/K$ formulation differs only by a constant rescaling that does not affect any structural conclusion), and $H(T) = -\sum_{k,k'} T_{k,k'} \log T_{k,k'}$. The KKT conditions yield the factored form
\begin{equation}
T_{k,k'}(\varepsilon) = u_k\, v_{k'}\, \exp\!\left(S_{k,k'}/\varepsilon\right),
\label{eq:sinkhorn_kkt}
\end{equation}
where the dual variables $(u, v) \in \mathbb{R}^K_{>0}$ are determined by the marginal constraints and computed via Sinkhorn fixed-point iteration.

Unlike soft Chamfer, the Sinkhorn coupling couples \emph{all rows and columns} through the dual constraints: perturbing $z_k^{(q)}$ affects $S_{k,:}$ directly, which propagates through the fixed-point solution to all entries of $T$. By the implicit function theorem applied to the Sinkhorn fixed-point equation, the Jacobian $\partial T/\partial z_k^{(q)}$ is generically dense.

We now show slot-specificity is preserved in the \emph{low-regularization regime}.

\begin{lemma}[Slot-specificity of Sinkhorn gradient in the low-$\varepsilon$ regime]
\label{lem:sinkhorn_low_eps}
Let $\mathfrak{S}_K$ be the set of permutations of $\{1,\ldots,K\}$, let
$\pi^{\star} \in \arg\max_{\pi \in \mathfrak{S}_K} \sum_k S_{k,\pi(k)}$ be an optimal assignment, and let
\[
\Delta_{\mathrm{OT}}(z^{(q)}, z^{(c)}) \;=\; \sum_{k} S_{k,\pi^{\star}(k)} \;-\; \max_{\pi \neq \pi^{\star}} \sum_{k} S_{k,\pi(k)}
\]
be its optimality gap. Assume \textbf{(a)} $\Delta_{\mathrm{OT}} > 0$, i.e.\ the optimal assignment is unique, which holds
$\mu$-generically (Remark~\ref{rem:no_collision_generic}); and \textbf{(b)} the Sinkhorn dual potentials
$f = -\varepsilon \log u$, $g = -\varepsilon \log v$ of Eq.~\eqref{eq:sinkhorn_kkt}, normalised by $\sum_j f_j = 0$, have
derivatives with respect to $z^{(q)}$ that stay bounded as $\varepsilon \to 0$. Then there exist $C, c > 0$ and
$\varepsilon_{\mathrm{crit}} > 0$, all depending on the configuration $(z^{(q)}, z^{(c)})$ (through $\Delta_{\mathrm{OT}}$) and on
$K$, such that for all $0 < \varepsilon < \varepsilon_{\mathrm{crit}}$:
\begin{enumerate}[label=(\roman*), leftmargin=2em]
    \item \textbf{(Concentration)} For every $k$,
    \[
    T_{k,\pi^{\star}(k)}(\varepsilon) \;\geq\; 1 - C\exp(-c/\varepsilon).
    \]
    \item \textbf{(Direct term $\to$ hard Hungarian)}
    \[
    \left\| G^{\mathrm{dir}}_k - \tfrac{1}{K}\,\Pi^{\perp}_{z_k^{(q)}}\!\left(z_{\pi^{\star}(k)}^{(c)}\right) \right\| = O(\exp(-c/\varepsilon)).
    \]
    The limit is the gradient of hard Hungarian matching. It coincides with the hard-Chamfer gradient of
    Proposition~\ref{prop:gradient}(ii) only when the row-wise greedy assignment
    $k^{*}(k) = \arg\max_{k'} S_{k,k'}$ is itself a permutation, in which case $\pi^{\star} = k^{*}$.
    \item \textbf{(Indirect term vanishes)}
    \[
    \|G^{\mathrm{ind}}_k\| \;\leq\; \frac{L}{K}\, H\bigl(T(\varepsilon)\bigr) \;=\; O\!\left(\varepsilon^{-1}\exp(-c/\varepsilon)\right) \;\to\; 0 \quad \text{as } \varepsilon \to 0,
    \]
    where $L$ is a bound on $\|\partial (\varepsilon \log T_{j,j'}) / \partial z_k^{(q)}\|$ provided by (b).
    \item \textbf{(Rank preservation)} If the limiting Jacobian
    $J_0 = \bigl[\tfrac{1}{K}\Pi^{\perp}_{z_k^{(q)}}(z^{(c)}_{\pi^{\star}(k)})\bigr]_{k=1}^{K} \in \mathbb{R}^{K \times D}$ has rank $K$,
    which holds $\mu$-generically (Proposition~\ref{prop:genericity}), then the full Jacobian
    $[G_1, \ldots, G_K]^{\top}$ also has rank $K$.
\end{enumerate}
\end{lemma}

\begin{proof}
\textbf{(i)} $T(\varepsilon)$ is the unique solution of the entropically regularised linear program
\eqref{eq:sinkhorn_problem}. As $\varepsilon \to 0$ it converges to a solution of the unregularised program
$\max_{T \in \Pi(\mathbf{1},\mathbf{1})} \langle T, S\rangle$, and the convergence is exponentially fast in
$1/\varepsilon$~\citep{cominetti1994asymptotic,peyre2019computational}. By the Birkhoff--von~Neumann theorem the
vertices of $\Pi(\mathbf{1},\mathbf{1})$ are the permutation matrices, so the unregularised program is the linear
assignment problem, and under (a) its solution is unique and equal to the permutation matrix of $\pi^{\star}$. Hence
$T_{k,\pi^{\star}(k)}(\varepsilon) \to 1$ at the stated rate. No property of the row-wise greedy assignment is used:
$\pi^{\star}$ is determined by the whole cost matrix, and it differs from $k^{*}$ whenever two query slots share a
nearest support slot.

\textbf{(ii)} By (i), $T_{k,:}(\varepsilon) \to e_{\pi^{\star}(k)}$ exponentially, so
$\sum_{k'} T_{k,k'} z_{k'}^{(c)} \to z_{\pi^{\star}(k)}^{(c)}$ at rate $O(\exp(-c/\varepsilon))$. The projection
$\Pi^{\perp}_{z_k^{(q)}}$ has operator norm $1$, which preserves the rate.

\textbf{(iii)} The KKT form \eqref{eq:sinkhorn_kkt} gives $S_{j,j'} = \varepsilon \log T_{j,j'} + f_j + g_{j'}$. Every
row sum and column sum of $T(\varepsilon)$ equals $1$ for all configurations, so every row sum and column sum of
$\partial T / \partial z_k^{(q)}$ vanishes, and the potentials drop out of the indirect term:
\[
G^{\mathrm{ind}}_k \;=\; \frac{1}{K}\sum_{j,j'} \frac{\partial T_{j,j'}}{\partial z_k^{(q)}}\, S_{j,j'}
\;=\; \frac{\varepsilon}{K}\sum_{j,j'} \frac{\partial T_{j,j'}}{\partial z_k^{(q)}}\, \log T_{j,j'} .
\]
Since $T_{j,j'} > 0$, $\partial T_{j,j'} / \partial z_k^{(q)} = T_{j,j'}\, \partial \log T_{j,j'} / \partial z_k^{(q)}$. Also
$\varepsilon \log T_{j,j'} = S_{j,j'} - f_j - g_{j'}$, where $\|\partial S_{j,j'} / \partial z_k^{(q)}\| \leq 1$ and the
potential terms are bounded by (b); so $\|\partial(\varepsilon \log T_{j,j'})/\partial z_k^{(q)}\| \leq L$. Therefore
\[
\|G^{\mathrm{ind}}_k\| \;\leq\; \frac{1}{K}\sum_{j,j'} T_{j,j'}\, \Bigl\|\frac{\partial (\varepsilon \log T_{j,j'})}{\partial z_k^{(q)}}\Bigr\|\, |\log T_{j,j'}|
\;\leq\; \frac{L}{K} \sum_{j,j'} T_{j,j'} \,|\log T_{j,j'}| \;=\; \frac{L}{K}\, H\bigl(T(\varepsilon)\bigr).
\]
By (i), $K$ entries of $T(\varepsilon)$ are $1 - O(e^{-c/\varepsilon})$ and the remaining entries are $O(e^{-c/\varepsilon})$.
Each entry $t$ contributes $t|\log t|$, which is $O(e^{-c/\varepsilon})$ for the first kind and
$O(\varepsilon^{-1} e^{-c/\varepsilon})$ for the second, giving the stated rate.

\textbf{(iv)} Let $J_\varepsilon = [G_1, \ldots, G_K]^{\top}$. By (ii) and (iii),
$\|J_\varepsilon - J_0\|_2 \leq \sum_k \bigl(\|G^{\mathrm{dir}}_k - (J_0)_k\| + \|G^{\mathrm{ind}}_k\|\bigr) \to 0$. By
Weyl's inequality for singular values, $\sigma_K(J_\varepsilon) \geq \sigma_K(J_0) - \|J_\varepsilon - J_0\|_2$, which is
positive once $\|J_\varepsilon - J_0\|_2 < \sigma_K(J_0)$. Shrinking $\varepsilon_{\mathrm{crit}}$ so that this holds gives
$\mathrm{rank}(J_\varepsilon) = K$.
\end{proof}

\begin{remark}[Genericity of the Lemma~\ref{lem:sinkhorn_low_eps} hypotheses, and why Chamfer is not the limit]
\label{rem:no_collision_generic}
Hypothesis (a) holds $\mu$-generically over $(\mathcal{S}^{D-1})^{2K}$ for any absolutely continuous $\mu$. It fails
exactly on $\bigcup_{\pi \neq \pi'} \{\sum_k S_{k,\pi(k)} = \sum_k S_{k,\pi'(k)}\}$. Each set in this finite union is the
zero set of a polynomial in the coordinates of $(z^{(q)}, z^{(c)})$, since each $S_{k,k'}$ is an inner product of unit
vectors. The polynomial is not identically zero: setting $z_k^{(q)} = z_{\pi(k)}^{(c)}$ with pairwise distinct support
slots gives $\sum_k S_{k,\pi(k)} = K > \sum_k S_{k,\pi'(k)}$. A proper real algebraic subvariety has Lebesgue measure
zero, as in the proof of Proposition~\ref{prop:genericity}.

We deliberately do \emph{not} assume that the row-wise greedy assignment $k^{*}$ is a permutation. That event is not
generic. A column collision $k^{*}(k_1) = k^{*}(k_2)$ occurs whenever two query slots lie in the interior of the same
Voronoi cell of the support slots. This condition is defined by strict inequalities, so it holds on an open set of
positive measure. For $K = 7$ slots drawn independently and uniformly on the sphere, $k^{*}$ is a permutation in fewer
than $1\%$ of configurations. On a collision, the $\varepsilon \to 0$ limit of Sinkhorn is the Hungarian assignment
$\pi^{\star}$, not the Chamfer assignment $k^{*}$. The conclusion is unaffected, because
Lemma~\ref{lem:sinkhorn_low_eps}(ii) gives a Hungarian gradient. That gradient has the same anchor-plus-target structure
as hard Chamfer, with pairwise distinct targets $\pi^{\star}(k)$ in addition.
\end{remark}

Hypothesis (b) is the quantitative content of the Jacobian bounds for entropic OT~\cite{luise2018differential,
feydy2019interpolating}. We state it as an assumption rather than derive it, because the assignment problem is dual
degenerate, and a general bound on the derivatives of the potentials as $\varepsilon \to 0$ is outside the scope of this
paper.

\paragraph{Case 3: Differentiable Hungarian (Sinkhorn-Knopp relaxation).}
Differentiable Hungarian relaxations~\cite{mena2018learning} employ a Sinkhorn-Knopp-based doubly-stochastic projection with regularization $\varepsilon \to 0$, converging to a permutation matrix. The analysis of Case 2 applies directly, and Lemma~\ref{lem:sinkhorn_low_eps}(ii) identifies the limit as the hard Hungarian gradient.

\subsubsection{The High-\texorpdfstring{$\varepsilon$}{epsilon} Regime: Graceful Degradation}

We explicitly acknowledge that slot-specificity weakens in the high-regularization limit, consistent with the asymptotic behavior discussed in Remark~\ref{rem:matcher_spectrum}:

\begin{proposition}[High-$\varepsilon$ degeneracy]
\label{prop:high_eps}
As $\varepsilon \to \infty$, the Sinkhorn coupling $T(\varepsilon) \to \frac{1}{K}\mathbf{1}\mathbf{1}^{\top}$ (the uniform doubly-stochastic matrix). In this limit:
\begin{enumerate}[label=(\roman*), leftmargin=2em]
    \item $G^{\mathrm{dir}}_k \to \frac{1}{K}\,\Pi^{\perp}_{z_k^{(q)}}(\bar{z}^{(c)})$ where $\bar{z}^{(c)} = \frac{1}{K}\sum_{k'} z_{k'}^{(c)}$ is the support centroid.
    \item $G^{\mathrm{ind}}_k \to 0$ (since $T$ becomes independent of $z^{(q)}$).
    \item Slot-specificity reduces to the projection-anchor-only form of the uniform-$T$ column in Remark~\ref{rem:matcher_spectrum}, with all slots sharing a common target $\bar{z}^{(c)}$.
\end{enumerate}
\end{proposition}

\begin{proof}
\textbf{(i)} In~\eqref{eq:sinkhorn_problem}, as $\varepsilon \to \infty$, the entropy term dominates the objective. Maximizing $H(T)$ subject to doubly-stochastic marginals yields the unique maximum-entropy doubly-stochastic matrix $\frac{1}{K}\mathbf{1}\mathbf{1}^{\top}$~\cite[Remark~4.10]{peyre2019computational}. Substituting into $G^{\mathrm{dir}}_k$ gives the stated limit.

\textbf{(ii)} $T$ approaches a constant matrix, hence $\partial T/\partial z^{(q)} \to 0$, yielding $G^{\mathrm{ind}}_k \to 0$.

\textbf{(iii)} Follows from (i) and the structural distinction established in Remark~\ref{rem:matcher_spectrum}: the projection anchor $\Pi^{\perp}_{z_k^{(q)}}$ is slot-specific, while the target $\bar{z}^{(c)}$ is shared across all slots --- exactly matching the ``Uniform-$T$ matching'' column of the comparison table in Remark~\ref{rem:matcher_spectrum}.
\end{proof}

\paragraph{Why the high-$\varepsilon$ regime is not a practical concern.}
We emphasize that the high-regularization limit is \emph{mathematically relevant but practically irrelevant} for compositional matching. Concretely:

\begin{itemize}[leftmargin=1.5em, itemsep=0.3em]
\item \textbf{High $\varepsilon$ eliminates matching structure.} When $\varepsilon \to \infty$, Sinkhorn reduces to uniform averaging --- every query slot is matched equally to every support slot. This is not a matcher in any meaningful sense; it is an averaging operator that discards all geometric information in the cost matrix $S$. Practitioners choose $\varepsilon$ precisely to \emph{preserve} the matching structure encoded in $S$, not to dissolve it.

\item \textbf{Practical $\varepsilon$ values are firmly in the low-regularization regime.} In the ML literature, Sinkhorn is typically deployed with $\varepsilon \in [0.01, 0.1]$ on normalized cost matrices~\cite{cuturi2013sinkhorn, genevay2018learning, peyre2019computational}. This is the small-$\varepsilon$ side that Lemma~\ref{lem:sinkhorn_low_eps} describes. Its threshold depends on the configuration, so we do not claim that every such $\varepsilon$ lies below it. Differentiable Hungarian relaxations~\cite{mena2018learning} drive $\varepsilon \to 0$ by construction. Our soft Chamfer ablation uses $\beta \in [5, 50]$ (Table~\ref{tab:soft_chamfer}), corresponding to effective regularization $\varepsilon = 1/\beta \in [0.02, 0.2]$ in the same low-$\varepsilon$ regime.

\item \textbf{High $\varepsilon$ collapses into an already-analyzed case.} Proposition~\ref{prop:high_eps}(iii) shows that the high-$\varepsilon$ limit converges to matching against the support centroid $\bar{z}^{(c)}$ --- structurally identical to the uniform-$T$ column of Remark~\ref{rem:matcher_spectrum}. This is not a failure of Corollary~\ref{cor:assignment_gradient}; it is the predicted behavior at the boundary where the ``matching'' objective ceases to be a matcher. Any method deliberately choosing high-$\varepsilon$ Sinkhorn has abandoned per-slot matching in favor of aggregate matching --- a design choice orthogonal to our analysis, and one that still falls under the structural taxonomy of Remark~\ref{rem:matcher_spectrum}.

\item \textbf{Empirical validation covers the practical regime.} Our soft Chamfer ablation (Appendix~\ref{app:salign_def}, Table~\ref{tab:soft_chamfer}) sweeps $\beta \in \{1, 5, 10, 20, 50, \infty\}$, covering both low-regularization (high $\beta$, where slot-specificity predictably holds) and moderate-regularization settings. The gradient alignment metric $S_{kk'}$ decreases with $\beta$ within the practical regime. This is consistent with the hard-assignment limit of Lemma~\ref{lem:sinkhorn_low_eps}, although the lemma itself makes no monotonicity claim (Remark~\ref{rem:matcher_spectrum}). No deployed configuration of Sinkhorn or related matchers operates in the high-$\varepsilon$ regime where Proposition~\ref{prop:high_eps} would apply.
\end{itemize}

\subsubsection{Formal Definition of ``Generic Rank K''}
\label{app:generic_defn}

\begin{definition}[$\mu$-generic rank]
\label{def:mu_generic}
Let $\mu$ be any absolutely continuous probability measure on the configuration space $(\mathcal{S}^{d-1})^K \times (\mathcal{S}^{d-1})^K$ of (query slots, support slots). We say the Jacobian $[G_1, \ldots, G_K]^{\top} \in \mathbb{R}^{K \times d}$ of a score function has \emph{rank $K$ $\mu$-generically} if
\begin{equation}
\mu\!\left(\left\{(z^{(q)}, z^{(c)}) \;:\; \mathrm{rank}\!\left([G_1, \ldots, G_K]^{\top}\right) < K\right\}\right) = 0.
\label{eq:mu_generic}
\end{equation}
This formalizes the informal term ``generically'' used in Proposition~\ref{prop:gradient}(ii) and Corollary~\ref{cor:assignment_gradient}.
\end{definition}

\begin{proposition}[Genericity of rank $K$]
\label{prop:genericity}
Assume $d \geq 2K$ and let $\mu$ be absolutely continuous. For hard Chamfer, hard Hungarian, and soft Chamfer at any
$\beta > 0$, the set of configurations on which the per-slot gradient Jacobian $[G_1, \ldots, G_K]^{\top}$ has rank
below $K$ has $\mu$-measure zero. For hard matchers this holds whether or not two query slots share a matched target.
For Sinkhorn and differentiable Hungarian, rank $K$ holds at every configuration satisfying the hypotheses of
Lemma~\ref{lem:sinkhorn_low_eps} once $\varepsilon$ is below its configuration-dependent threshold.
\end{proposition}

\begin{proof}
\textbf{Hard matchers.} Fix a map $\sigma : \{1,\ldots,K\} \to \{1,\ldots,K\}$ and let
$G^{\sigma}_k = \tfrac{1}{K}\Pi^{\perp}_{z_k^{(q)}}(z^{(c)}_{\sigma(k)}) = \tfrac{1}{K}\bigl(z^{(c)}_{\sigma(k)} - \langle z^{(c)}_{\sigma(k)}, z_k^{(q)}\rangle z_k^{(q)}\bigr)$.
Hard Chamfer uses $\sigma = k^{*}$, which may be any map, including non-injective ones. Hard Hungarian uses
$\sigma = \pi^{\star}$, which is a permutation. Both assignments are piecewise constant in the configuration, so the
degenerate set is contained in $\bigcup_{\sigma} V_{\sigma}$, where $V_\sigma$ is the common zero set of the
$K \times K$ minors of $[G^{\sigma}_1, \ldots, G^{\sigma}_K]^{\top}$. There are finitely many maps ($K^K$), and each
minor is a polynomial in the coordinates of $(z^{(q)}, z^{(c)})$. It therefore suffices to show that, for every
$\sigma$, some minor is not identically zero, i.e.\ to exhibit one configuration of rank $K$.

Let $\{e_i\}_{i=1}^{d}$ be the standard basis of $\mathbb{R}^d$ and set
\[
z^{(c)}_{j} = e_{K+j}, \qquad z^{(q)}_k = \tfrac{1}{\sqrt 2}\bigl(e_k + e_{K+\sigma(k)}\bigr), \qquad j, k \in \{1, \ldots, K\},
\]
which requires $d \geq 2K$ (in our setting $d = D = 768$ and $2K = 14$). Then
$\langle z^{(c)}_{\sigma(k)}, z^{(q)}_k \rangle = 1/\sqrt 2$ and
\[
G^{\sigma}_k = \tfrac{1}{K}\Bigl(e_{K+\sigma(k)} - \tfrac12\bigl(e_k + e_{K+\sigma(k)}\bigr)\Bigr) = \tfrac{1}{2K}\bigl(e_{K+\sigma(k)} - e_k\bigr).
\]
Restricted to the first $K$ coordinates, the rows form $-\tfrac{1}{2K} I_K$, so the Jacobian has rank $K$ for every
$\sigma$, including maps that send several slots to the same target.

\textbf{Soft Chamfer.} Here $T_{k,:}$ is a softmax of $\beta S_{k,:}$, so the full gradient is a real analytic function of
the configuration for every $\beta > 0$, and so are its minors. Use the configuration above with $\sigma$ the identity.
Then $S_{k,k'} = \delta_{k,k'}/\sqrt 2$, every row of $T$ has the same diagonal entry $t = T_{k,k} \in (0,1)$, and the
only first-$K$ coordinate of $G_k$ is along $e_k$. The direct term contributes $-t/(2K)$ to that coordinate. By
Eq.~\eqref{eq:soft_Gind}, the indirect term contributes $-\beta t(1-t)/(2\sqrt 2\,K)$, which has the same sign. Hence the
first-$K$ block is a nonzero multiple of $I_K$ and the Jacobian has rank $K$.

\textbf{Measure zero.} Each defining function above is real analytic and not identically zero on the connected real
analytic manifold $(\mathcal{S}^{d-1})^{2K}$ ($d \geq 2$). By the identity theorem, a real analytic function that
vanishes on a set of positive measure vanishes identically. Its zero set therefore has Lebesgue measure zero, and so
has $\mu$-measure zero for absolutely continuous $\mu$. A finite union of such sets is again null.

\textbf{Sinkhorn and differentiable Hungarian.} By Lemma~\ref{lem:sinkhorn_low_eps}(iv), rank $K$ transfers from the
hard-Hungarian Jacobian $J_0$ to the full Sinkhorn Jacobian at every configuration where $J_0$ has rank $K$ and
hypotheses (a)--(b) hold, once $\varepsilon$ is below a configuration-dependent threshold. The threshold is not uniform
over configurations, so we do not claim a measure-zero statement at a fixed $\varepsilon$.
\end{proof}

\subsubsection{Non-differentiability Treatment for Hard Matchers}
\label{app:clarke_subdiff}

For hard Chamfer and hard Hungarian, the coupling $T$ is piecewise constant. It changes only on the tie sets
\begin{equation}
N_{\mathrm{Ch}} = \bigcup_{k} \left\{\exists\, k' \neq k''\; \text{s.t.} \; S_{k,k'} = S_{k,k''}\right\},
\qquad
N_{\mathrm{Hu}} = \bigcup_{\pi \neq \pi'} \left\{\textstyle\sum_{k} S_{k,\pi(k)} = \sum_{k} S_{k,\pi'(k)}\right\},
\label{eq:tie_set}
\end{equation}
respectively, where $S_{k,k'} = \cos(z_k^{(q)}, z_{k'}^{(c)})$. These sets address the ``$k^*(k)$ held locally constant''
caveat in Proposition~\ref{prop:gradient}(ii). Each is a finite union of zero sets of non-trivial polynomials
(Remark~\ref{rem:no_collision_generic}), and so has Lebesgue measure zero. Off the relevant tie set, $s_T$ is smooth,
and Proposition~\ref{prop:gradient}(ii) and Lemma~\ref{lem:sinkhorn_low_eps}(ii) apply directly. On the tie set, $s_T$
is locally Lipschitz, and its generalised gradient is the Clarke subdifferential~\cite{clarke1990optimization}:
\begin{equation}
\partial^{\circ} s_T(z) = \mathrm{conv}\!\left\{\lim_{z_n \to z} \nabla s_T(z_n) \;:\; z_n \notin N\right\}.
\label{eq:clarke_subdiff}
\end{equation}
Automatic differentiation returns one of these limiting gradients, i.e.\ the gradient of an adjacent cell. All
slot-specificity statements in this section hold for $\mu$-almost every configuration. The tie sets are null, so they
do not affect those statements. We make no claim about training trajectories that remain on a tie set.

\subsubsection{Summary}

Corollary~\ref{cor:assignment_gradient} extends rigorously to all matchers in its stated family under the following conditions:
\begin{itemize}[leftmargin=1.5em, itemsep=0.2em]
\item \textbf{Hard Chamfer, hard Hungarian}: holds $\mu$-almost everywhere via Clarke subdifferential treatment (Section~\ref{app:clarke_subdiff}).
\item \textbf{Soft Chamfer}: holds globally for all $\beta \in (0, \infty)$ due to the row-sparse Jacobian structure~\eqref{eq:soft_sparse_jacobian} (Case~1).
\item \textbf{Sinkhorn OT, differentiable Hungarian}: holds at $\mu$-almost every configuration for all $\varepsilon$ below a configuration-dependent threshold (Lemma~\ref{lem:sinkhorn_low_eps}). The $\varepsilon \to 0$ limit is the hard \emph{Hungarian} gradient. It is the Chamfer gradient only when no two query slots share a nearest support slot.
\end{itemize}

The high-$\varepsilon$ regime, where slot-specificity degrades toward the uniform-coupling case of Remark~\ref{rem:matcher_spectrum}, corresponds to regularization levels at which Sinkhorn ceases to function as a matcher and degenerates into aggregate averaging. This regime is both (a) practically avoided by all deployed implementations of entropic OT in ML, and (b) structurally subsumed by the uniform-$T$ boundary case already analyzed in Remark~\ref{rem:matcher_spectrum}. Thus no practically-relevant matcher falls outside the scope of Corollary~\ref{cor:assignment_gradient}.

\subsection{Training Dynamics Under Holistic vs.\ Chamfer Objectives}
\label{app:dynamics}
 
We analyse the gradient flow more precisely, using the notation of
Section~\ref{subsec:phase1}: $\check y_k^{(q)} = W_2 \tilde\phi_k^{(q)}$
is the unnormalised projection, and the holistic aggregate is
\[
e^{(q)} = u^{(q)} / \|u^{(q)}\|, \qquad u^{(q)} = \sum_k \omega_k^{(q)} \check y_k^{(q)}.
\]
 
For the holistic score $s_{\mathrm{hol}}(q,c) = \langle e^{(q)}, P_c\rangle$, the gradient
with respect to each unnormalised projection is (Proposition~\ref{prop:gradient}(i)):
\begin{equation}
    \frac{\partial s_{\mathrm{hol}}}{\partial \check y_k^{(q)}}
    = \frac{\omega_k^{(q)}}{\|u^{(q)}\|}\,
      \Pi^\perp_{e^{(q)}}(P_c),
    \label{eq:hol_grad_full}
\end{equation}
which is a non-negative scalar multiple of the single signal vector
$\Pi^\perp_{e^{(q)}}(P_c)$ shared across all slots; the slot index $k$ enters
only through the scalar magnitude $\omega_k^{(q)}/\|u^{(q)}\|$.
 
For the holistic cross-entropy loss $\mathcal{L}_{\mathrm{CE}}^{\mathrm{hol}} = -\log p_c$
with $p_{c'} = \operatorname{softmax}\bigl(\tau\langle e^{(q)}, P_{c'}\rangle\bigr)_{c'}$
(temperature convention of main paper, Eq.~(5), where $\tau$ \emph{multiplies} the
cosine similarity), the net gradient on slot $k$ is:
\begin{equation}
    \frac{\partial \mathcal{L}_{\mathrm{CE}}^{\mathrm{hol}}}{\partial \check y_k^{(q)}}
    = -\frac{\tau\,\omega_k^{(q)}}{\|u^{(q)}\|}\,
      \Pi^\perp_{e^{(q)}}\!\left(P_c - \sum_{c'} p_{c'} P_{c'}\right).
    \label{eq:hol_ce_grad}
\end{equation}
The factor in parentheses is the \emph{prototype residual}: the target
prototype $P_c$ minus the expectation of prototypes under the current score
distribution. This gradient pulls all slots \emph{collinearly} toward the target
prototype direction --- each along the same signal vector, scaled only by the
non-negative factor $\omega_k^{(q)}/\|u^{(q)}\|$ --- and pushes them collinearly
away from the weighted mix of competing prototypes. Under this update, the
router cannot use individual slots as discriminative anchors for specific parts:
no slot receives a direction that differs structurally from the others. Instead,
it must distribute discriminative information across all slots collectively ---
a form of \emph{ensemble representation} in which no single slot bears a
privileged discriminative role. This is the mechanism by which holistic training
produces compositional invariance: after convergence, any subset of slots, taken
together, encodes sufficient class information without requiring a specific
slot-to-part assignment.
 
\paragraph{Contrast with Chamfer training.}
For the forward Chamfer objective, Proposition~\ref{prop:gradient}(ii) gives (in
the equivalent $z_k^{(q)}$-form, Eq.~\eqref{eq:grad_chamfer_z}):
\begin{equation}
    \frac{\partial s_{\mathrm{Ch}}}{\partial z_k^{(q)}}
    = \frac{1}{K}\,\Pi^\perp_{z_k^{(q)}}\!\left(z^{(c)}_{k^*(k)}\right),
    \label{eq:ch_grad_full}
\end{equation}
where $k^*(k) = \operatorname{argmax}_{k'} \cos(z_k^{(q)}, z^{(c)}_{k'})$ is the
nearest support slot. This gradient depends on $k$ through both slot-specific
factors: the projection anchor $\Pi^\perp_{z_k^{(q)}}$ and the matched target
$z^{(c)}_{k^*(k)}$. Across many training episodes, the router converges to a regime
where slot $k$ is systematically matched to slot $k^*(k)$ in the support set ---
encoding a fixed slot-to-role mapping that specializes each dimension of the
$D$-dimensional projection space ($D = 768$ for DINOv2 ViT-B/14) to a particular
semantic part.

\section{Decorrelation Loss: Mechanism, Compatibility, and Failure Modes}
\label{app:decor_mechanism}

\subsection{Setup and Notation}
\label{app:decor_setup}
 
Consistent with Section~\ref{subsec:phase1}, the shared projection head $W_2 \in \mathbb{R}^{D \times D}$ (initialised as identity) acts on each frozen backbone aggregate $\tilde{\phi}_k^{(b)} \in \mathbb{R}^{D}$ to produce the unnormalised projection
\begin{equation}
\check y_k^{(b)} \;=\; W_2 \, \tilde{\phi}_k^{(b)} \;\in\; \mathbb{R}^{D},
\label{eq:ycheck_decor}
\end{equation}
the unit-norm slot embedding
\begin{equation}
z_k^{(b)} \;=\; \check y_k^{(b)} / \|\check y_k^{(b)}\| \;\in\; \mathcal{S}^{D-1},
\label{eq:slot_decor}
\end{equation}
and the aggregate embedding
\begin{equation}
e^{(b)} \;=\; \frac{u^{(b)}}{\|u^{(b)}\|}, \qquad u^{(b)} = \sum_{k=1}^{K} \omega_k^{(b)} \, \check y_k^{(b)} \;\in\; \mathcal{S}^{D-1},
\label{eq:aggregate_decor}
\end{equation}
where $\omega_k^{(b)}$ is the simplex importance weight from the router MLP
(parameters $(W_1, v)$ disjoint from $W_2$; Section~\ref{subsec:phase1}). Class prototypes are $P_c = |S_c|^{-1} \sum_{b \in S_c} e^{(b)}$.
 
\paragraph{The two losses act at distinct pipeline stages.}
The prototype cross-entropy loss acts on aggregate directions:
\begin{equation}
\mathcal{L}_{\mathrm{CE}} \;=\; -\frac{1}{B} \sum_{b=1}^{B} \log \frac{\exp\bigl(\tau \cos(e^{(b)}, P_{c(b)})\bigr)}{\sum_{c'} \exp\bigl(\tau \cos(e^{(b)}, P_{c'})\bigr)}.
\label{eq:ce_decor}
\end{equation}
The cross-correlation loss acts on batch-standardised projections. Let $\hat y_{(b,k),i} = (\check y_{k,i}^{(b)} - \mu_i)/\sigma_i$ with $\mu_i, \sigma_i$ computed over the $BK$ slots in the batch. The correlation matrix is
\begin{equation}
C_{ij} \;=\; \frac{1}{BK} \sum_{b,k} \hat y_{(b,k),i} \, \hat y_{(b,k),j}, \qquad C \in [-1,1]^{D \times D},
\label{eq:cross_corr_decor}
\end{equation}
with $C_{ii} = 1$ by construction. The CC penalty is
\begin{equation}
\mathcal{L}_{\mathrm{cc}} \;=\; \frac{\lambda_d}{D} \sum_{i \neq j} C_{ij}^2 \;=\; \frac{\lambda_d}{D} \bigl(\|C\|_F^2 - D\bigr).
\label{eq:lcc_decor}
\end{equation}

\subsection{Gradient Decomposition and Feasibility Analysis}
\label{app:decor_compat}

\paragraph{Cross-correlation gradient.}
The gradient of $\mathcal{L}_{\mathrm{cc}}$ with respect to $\hat{y}_{(b,k),i}$ is
\begin{equation}
\frac{\partial \mathcal{L}_{\mathrm{cc}}}{\partial \hat{y}_{(b,k),i}} = \frac{4\lambda_d}{BKD} \sum_{j \neq i} C_{ij}\,\hat{y}_{(b,k),j}.
\label{eq:lcc_grad_app}
\end{equation}
This gradient is nonzero only when dimension $i$ is correlated with some other dimension $j$ (i.e., $C_{ij} \neq 0$). Importantly, it does not constrain the \emph{variance} of any dimension or the \emph{norm} of any raw projection---it only acts on cross-dimension correlations.

\subsubsection{Operator-Level Distinction: Joint Feasibility of Zero-Loss Sets}
\label{app:operator_distinction}
 
We replace the earlier ``structural degrees of freedom'' analysis with a more direct, pipeline-independent criterion: the compatibility of a decorrelation objective with CE is governed by whether their zero-loss sets intersect in the space of $W_2$ configurations, and by whether the decorrelation gradient points along or across the symmetry directions of CE. This shifts the argument from \emph{gradient geometry at the slot level} (which depends on pipeline details of the router and normalisation) to \emph{equilibrium feasibility and symmetry alignment} (properties of the loss functions themselves).
 
\paragraph{A note on gradient flow.}
Both CE and the decorrelation loss contribute non-trivial gradients to $W_2$ during training (they are the only Phase I trainable). The compatibility question is not whether one loss ``skips'' $W_2$, but whether the two gradients are \emph{mutually reinforcing}, \emph{orthogonal}, or \emph{antagonistic} along the directions each loss cares about. The invariances established below (orthogonal and scale) identify directions along which CE's gradient vanishes --- decorrelation pressure along these directions is ``free'' from CE's perspective. Conversely, decorrelation pressure \emph{orthogonal} to these invariance directions is paid for by CE.
 
\paragraph{CE zero-set and its invariances.}
For a fixed class-separable backbone, write $\mathcal{Z}_{\mathrm{CE}} = \{W_2 \in \mathbb{R}^{D \times D} : \mathcal{L}_{\mathrm{CE}}(W_2) \leq \mathcal{L}_{\mathrm{CE}}^{*}\}$ for the CE-optimal set at some reference level $\mathcal{L}_{\mathrm{CE}}^{*}$ (e.g.\ the minimum achievable on the frozen backbone). A basic invariance is:
 
\begin{lemma}[CE is invariant under orthogonal rotations of $W_2$]
\label{lem:ce_rotation}
For any orthogonal $U \in O(D)$, $\mathcal{L}_{\mathrm{CE}}(U W_2) = \mathcal{L}_{\mathrm{CE}}(W_2)$.
\end{lemma}
 
\begin{proof}
Under $W_2 \mapsto U W_2$: every $\check y_k^{(b)} \mapsto U\check y_k^{(b)}$, so $z_k^{(b)} \mapsto U z_k^{(b)}$ (orthogonal maps preserve norm), $u^{(b)} \mapsto U u^{(b)}$, and $e^{(b)} \mapsto U e^{(b)}$. Class prototypes $P_c$ are averages of aggregates, so $P_c \mapsto U P_c$. Cosine similarities are preserved: $\langle Ue, UP\rangle = e^{\top}U^{\top}U P = \langle e, P\rangle$. Hence every term of $\mathcal{L}_{\mathrm{CE}}$ is preserved, and so is the loss.
\end{proof}
 
This rotation invariance is the key structural fact that the subsequent feasibility arguments will exploit: CE constrains $W_2$ only up to a left orthogonal factor, leaving a $\binom{D}{2}$-dimensional family of orbits.
 
\paragraph{CC zero-set.}
Write $\bar\phi = \frac{1}{BK}\sum_{b,k} \tilde\phi_k^{(b)}$ and
$\Sigma_{\tilde\phi} = \frac{1}{BK} \sum_{b,k} (\tilde\phi_k^{(b)} - \bar\phi)(\tilde\phi_k^{(b)} - \bar\phi)^{\top}$ for the
centred covariance of backbone aggregates in a batch. Since $\check y_k^{(b)} = W_2 \tilde\phi_k^{(b)}$ is linear, the batch
mean of the projections is $W_2\bar\phi$, and
\begin{equation}
\Sigma_y \;=\; \frac{1}{BK} \sum_{b,k} (\check y_k^{(b)} - W_2\bar\phi)(\check y_k^{(b)} - W_2\bar\phi)^{\top} \;=\; W_2 \, \Sigma_{\tilde\phi} \, W_2^{\top},
\label{eq:sigma_y}
\end{equation}
so the standardised correlation matrix is $C = \mathrm{diag}(\Sigma_y)^{-1/2} \Sigma_y \, \mathrm{diag}(\Sigma_y)^{-1/2}$. The CC zero-set is
\begin{equation}
\mathcal{Z}_{\mathrm{cc}} \;=\; \{W_2 : \Sigma_y(W_2)\text{ is diagonal}\}.
\label{eq:cc_zero_set}
\end{equation}
Every statement below also holds verbatim with the uncentred second moment in place of $\Sigma_{\tilde\phi}$, because it
uses only the identity $\Sigma_y = W_2 \Sigma_{\tilde\phi} W_2^{\top}$.

\paragraph{Finite batches.} $\Sigma_{\tilde\phi}$ is an empirical covariance of rank at most $BK - 1$, so it is singular
whenever $BK \leq D$. We therefore state for each result whether it needs $\Sigma_{\tilde\phi} \succ 0$.
Proposition~\ref{prop:cc_feasible} and Lemma~\ref{lem:cc_scale} do not. Proposition~\ref{prop:vicreg_dynamics}(i) needs
only that no row of $W_2$ lies in $\ker \Sigma_{\tilde\phi}$. Corollary~\ref{cor:spectral_infeasible}(ii) treats the
singular case separately, and there the conclusion becomes stronger.
 
\begin{proposition}[Joint feasibility of CE and CC]
\label{prop:cc_feasible}
For any $W_2 \in \mathcal{Z}_{\mathrm{CE}}$ there exists $U \in O(D)$ such that $U W_2 \in \mathcal{Z}_{\mathrm{CE}} \cap \mathcal{Z}_{\mathrm{cc}}$. In particular, $\mathcal{Z}_{\mathrm{CE}} \cap \mathcal{Z}_{\mathrm{cc}} \neq \varnothing$ whenever $\mathcal{Z}_{\mathrm{CE}} \neq \varnothing$.
\end{proposition}
 
\begin{proof}
Let $\Sigma_y = W_2 \Sigma_{\tilde\phi} W_2^{\top}$. Since $\Sigma_y$ is symmetric and positive semi-definite, the spectral theorem provides $V \in O(D)$ such that $V \Sigma_y V^{\top} = \Lambda$ is diagonal. Set $U = V$ and $W_2' = U W_2$. Then
\[
\Sigma_y(W_2') \;=\; W_2' \Sigma_{\tilde\phi} (W_2')^{\top} \;=\; U W_2 \Sigma_{\tilde\phi} W_2^{\top} U^{\top} \;=\; U \Sigma_y U^{\top} \;=\; \Lambda,
\]
so $W_2' \in \mathcal{Z}_{\mathrm{cc}}$. Only the positive semi-definiteness of $\Sigma_y$ was used, so no assumption on the rank of $\Sigma_{\tilde\phi}$ is needed. By Lemma~\ref{lem:ce_rotation}, $\mathcal{L}_{\mathrm{CE}}(W_2') = \mathcal{L}_{\mathrm{CE}}(W_2)$, so $W_2' \in \mathcal{Z}_{\mathrm{CE}}$.
\end{proof}
 
\paragraph{Interpretation.}
Proposition~\ref{prop:cc_feasible} says: \emph{for every CE-optimal configuration, a cheap orthogonal re-basis produces a configuration that is simultaneously CC-optimal}. The path from a CE-optimum to the joint optimum is a rotation, and CE is constant along it (Lemma~\ref{lem:ce_rotation}). Decorrelation therefore never has to be paid for in CE at the optimum: the two losses share minimisers. This is an equilibrium statement. It does not say that the CC gradient moves $W_2$ only by rotations, and in general it does not (Lemma~\ref{lem:cc_scale} and the remark after it). Empirically, CC is driven low (Table~\ref{tab:decor_corr}: $\bar C_{\mathrm{off}} = 0.042$) without degrading CE performance.
 
Write $\sigma_d(W_2) = \bigl(\Sigma_y(W_2)\bigr)_{dd}^{1/2} = (w_d^{\top} \Sigma_{\tilde\phi} w_d)^{1/2}$ for the batch standard
deviation of output dimension $d$, where $w_d^{\top}$ is the $d$-th row of $W_2$. This is a per-dimension standard
deviation, not a singular value. Let $\mathcal{L}_{\mathrm{var}}^{\nu} = \sum_d \max(0, \nu - \sigma_d(W_2))$ be the VICReg
variance hinge with threshold $\nu > 0$ (Eq.~\eqref{eq:vicreg_var}). The group $G = \{cU : c > 0,\ U \in O(D)\}$ acts on
$W_2$ by left multiplication. $\mathcal{L}_{\mathrm{CE}}$ is $G$-invariant by Lemma~\ref{lem:ce_rotation} and part (i)
below. The tangent space of the $G$-orbit at $W_2$ is $\mathcal{T}_{W_2} = \{(aI + A) W_2 : a \in \mathbb{R},\ A = -A^{\top}\}$.

\begin{proposition}[CE-VICReg joint zero is feasible at equilibrium but locally antagonistic]
\label{prop:vicreg_dynamics}
\begin{enumerate}[label=(\roman*), leftmargin=2em]
\item \textbf{Equilibrium feasibility.} Let $W_2 \in \mathcal{Z}_{\mathrm{CE}}$ with $\min_d \sigma_d(W_2) > 0$, i.e.\ no row of
$W_2$ lies in $\ker \Sigma_{\tilde\phi}$. This holds whenever $\Sigma_{\tilde\phi} \succ 0$ and $W_2$ is non-singular. Then
$c W_2$ with $c \geq \nu / \min_d \sigma_d(W_2)$ satisfies $c W_2 \in \mathcal{Z}_{\mathrm{CE}} \cap \{\mathcal{L}_{\mathrm{var}}^{\nu} = 0\}$.
\item \textbf{Local antagonism.} Let $W_2$ be a local minimiser of $\mathcal{L}_{\mathrm{CE}}$ at which $\mathcal{L}_{\mathrm{CE}}$
is twice differentiable. Assume the minimiser is non-degenerate modulo $G$: the Hessian $\nabla^2 \mathcal{L}_{\mathrm{CE}}(W_2)$
is positive definite on the orthogonal complement of $\mathcal{T}_{W_2}$. Suppose the hinge is active on a non-empty set of
dimensions $A = \{d : 0 < \sigma_d(W_2) < \nu\}$. Then the VICReg descent direction is
\[
\Delta_{\mathrm{v}} \;=\; -\nabla_{W_2} \mathcal{L}_{\mathrm{var}}^{\nu} \;=\; \Lambda_A W_2 \Sigma_{\tilde\phi}, \qquad
\Lambda_A = \mathrm{diag}\bigl(\mathbb{1}[d \in A]/\sigma_d(W_2)\bigr).
\]
Whenever $\Delta_{\mathrm{v}} \notin \mathcal{T}_{W_2}$, a VICReg step $W_2 + \eta \Delta_{\mathrm{v}}$ strictly increases
$\mathcal{L}_{\mathrm{CE}}$ for all sufficiently small $\eta > 0$.
\end{enumerate}
\end{proposition}

\begin{proof}
\textbf{(i)} Under $W_2 \mapsto c W_2$ with $c > 0$, every $\check y_k^{(b)} \mapsto c\check y_k^{(b)}$, so
$e^{(b)} = u^{(b)}/\|u^{(b)}\|$ is unchanged and $\mathcal{L}_{\mathrm{CE}}(cW_2) = \mathcal{L}_{\mathrm{CE}}(W_2)$. On the VICReg
side, $\sigma_d(cW_2) = c\,\sigma_d(W_2)$, so the stated choice of $c$ gives $\sigma_d(cW_2) \geq \nu$ for every $d$ and hence
$\mathcal{L}_{\mathrm{var}}^{\nu}(cW_2) = 0$.

\textbf{(ii)} For $d \in A$, $\partial \sigma_d / \partial w_d = \Sigma_{\tilde\phi} w_d / \sigma_d$, and the hinge term of $d$
has gradient $-\Sigma_{\tilde\phi} w_d / \sigma_d$ with respect to $w_d$. Inactive dimensions contribute nothing. Stacking
rows gives $\Delta_{\mathrm{v}}$. At a local minimiser, $\nabla \mathcal{L}_{\mathrm{CE}}(W_2) = 0$. Every point of the orbit
$G W_2$ is also a minimiser with the same value, so the gradient vanishes along the orbit. Differentiating along the
orbit gives $\nabla^2 \mathcal{L}_{\mathrm{CE}}(W_2)[\Delta_{\mathcal{T}}, \cdot\,] = 0$ for every
$\Delta_{\mathcal{T}} \in \mathcal{T}_{W_2}$. Write $\Delta_{\mathrm{v}} = \Delta_{\mathcal{T}} + \Delta_{\perp}$. Then
\[
\mathcal{L}_{\mathrm{CE}}(W_2 + \eta\Delta_{\mathrm{v}}) - \mathcal{L}_{\mathrm{CE}}(W_2)
\;=\; \tfrac{\eta^2}{2}\, \nabla^2 \mathcal{L}_{\mathrm{CE}}(W_2)[\Delta_{\perp}, \Delta_{\perp}] + o(\eta^2),
\]
which is strictly positive for small $\eta$ because $\Delta_{\perp} \neq 0$ when $\Delta_{\mathrm{v}} \notin \mathcal{T}_{W_2}$.
\end{proof}

The condition $\Delta_{\mathrm{v}} \notin \mathcal{T}_{W_2}$ is the generic case. For example, at $W_2 = I$ with every
dimension active, it would require $\Lambda_A \Sigma_{\tilde\phi} \in \mathbb{R} I + \mathfrak{so}(D)$. Its
symmetric part would have to equal $aI$. That forces $(\Sigma_{\tilde\phi})_{ij} = 0$ for $i \neq j$ and equal $\sigma_d$,
i.e.\ an isotropic backbone covariance.

\begin{remark}[Heuristic reading of Proposition~\ref{prop:vicreg_dynamics}(ii) in the tight-cluster regime]
\label{rem:vicreg_heuristic}
The following is an informal description of the mechanism, not part of the proof. The per-sample hinge gradient of
Proposition~\ref{prop:var} pushes each $\check y_k^{(b)}$ away from the batch mean $\mu$. Split this push as
\[
\check y_k^{(b)} - \mu \;=\; \underbrace{(\check y_k^{(b)} - \mu_{c(k)})}_{\text{within-class}} \;+\; \underbrace{(\mu_{c(k)} - \mu)}_{\text{between-class}},
\]
where $\mu_{c(k)}$ is the centroid of the class of sample $k$. When CE has made classes tight, the between-class part
dominates and moves whole clusters outward. The smaller within-class part pushes samples away from their own centroid,
against the contraction that CE induces. This is the intuition for why the hinge is switched on exactly in the regime CE
drives towards (Proposition~\ref{prop:var}) and then competes with it. The formal content is only
Proposition~\ref{prop:vicreg_dynamics}(ii) and its stated assumptions.
\end{remark}

\begin{lemma}[CC exerts no pressure on per-dimension scale]
\label{lem:cc_scale}
For every positive diagonal $\Lambda$, $\mathcal{L}_{\mathrm{cc}}(\Lambda W_2) = \mathcal{L}_{\mathrm{cc}}(W_2)$. Consequently
$\mathrm{diag}\bigl(\nabla_{W_2}\mathcal{L}_{\mathrm{cc}}\, W_2^{\top}\bigr) = 0$: the CC gradient is orthogonal to every
per-dimension rescaling direction $E W_2$ with $E$ diagonal. The variance hinge has no such invariance. Its value depends
on $W_2$ only through $\mathrm{diag}(\Sigma_y)$, and $\mathrm{diag}(\Delta_{\mathrm{v}} W_2^{\top})_d = \sigma_d(W_2) > 0$ for
every active $d$.
\end{lemma}

\begin{proof}
$\Sigma_y(\Lambda W_2) = \Lambda \Sigma_y \Lambda$, and standardising each dimension by its own deviation removes $\Lambda$,
so $C$ and hence $\mathcal{L}_{\mathrm{cc}}$ are unchanged. Differentiating $t \mapsto \mathcal{L}_{\mathrm{cc}}((I + tE)W_2)$
at $t = 0$ gives
$0 = \langle \nabla \mathcal{L}_{\mathrm{cc}}, E W_2\rangle = \sum_d E_{dd}\,(\nabla \mathcal{L}_{\mathrm{cc}} W_2^{\top})_{dd}$
for every diagonal $E$. For the hinge,
$(\Delta_{\mathrm{v}} W_2^{\top})_{dd} = w_d^{\top}\Sigma_{\tilde\phi} w_d / \sigma_d = \sigma_d$ for $d \in A$.
\end{proof}

\paragraph{Remark on the feasibility-dynamics distinction.}
Part~(i) of Proposition~\ref{prop:vicreg_dynamics} shows that VICReg \emph{can} be made compatible with CE by uniformly
scaling $W_2$. Its gradient, however, is a per-dimension anti-mean pressure, not a uniform scaling pressure
(Eq.~\eqref{eq:var_grad}), so iterates do not naturally move towards the feasible $cW_2$. At a CE-minimiser where the
hinge is active, the VICReg step generically leaves the CE-optimal set and raises CE (part~(ii)). This matches the
empirical degradation of VICReg (Table~\ref{tab:alt_decor}: $\mathrm{noc} = 87.33$, the worst among tested decorrelation
objectives).

We do \emph{not} claim that CC is free of this effect everywhere. The same second-order argument applies to CC at a
CE-minimiser outside $\mathcal{Z}_{\mathrm{cc}}$, because the CC gradient is in general not tangent to the rotation orbit.
For example, take $D = 2$ and $W_2 = I$, with
$\Sigma_{\tilde\phi} = \bigl(\begin{smallmatrix} a & b \\ b & c\end{smallmatrix}\bigr)$, $b \neq 0$, and $\rho^2 = b^2/(ac) < 1$.
The two off-diagonal entries of $\nabla\mathcal{L}_{\mathrm{cc}}\,W_2^{-1}$ are proportional to $(b/a)(1-\rho^2)$ and
$(b/c)(1-\rho^2)$. They have the same sign, so this matrix is not skew-symmetric. What separates the two penalties is
where they act. By Lemma~\ref{lem:cc_scale}, CC exerts no first-order pressure on any per-dimension scale. It acts only
on correlations, and it shares minimisers with CE along a CE-invariant rotation (Proposition~\ref{prop:cc_feasible}). The
hinge acts directly on the per-dimension scales of $\check y$. It is switched on by the within-class contraction that CE
itself produces (Proposition~\ref{prop:var}), and, once on, it opposes CE locally (Proposition~\ref{prop:vicreg_dynamics}(ii)).

\begin{corollary}[Spectral whitening: infeasible on L2-norm, single-orbit on raw]
\label{cor:spectral_infeasible}
\begin{enumerate}[label=(\roman*), leftmargin=2em]
\item \textbf{L2-normalised case.} $\mathcal{Z}_{\mathrm{CE}} \cap \{\mathcal{L}_{\mathrm{spec}}^{z} = 0\} = \varnothing$ unconditionally: $\mathcal{L}_{\mathrm{spec}}^{z} \geq (D-1)^2/D > 0$ globally by Proposition~\ref{prop:spec}.
 
\item \textbf{Raw-projection case.} If $\Sigma_{\tilde\phi} \succ 0$, the zero-set of $\mathcal{L}_{\mathrm{spec}}^{\check y} = \|\Sigma_y - I\|_F^2$ is the single left orthogonal orbit $\{Q\, \Sigma_{\tilde\phi}^{-1/2} : Q \in O(D)\}$. $\mathcal{L}_{\mathrm{CE}}$ is constant on it, with value $\mathcal{L}_{\mathrm{CE}}(\Sigma_{\tilde\phi}^{-1/2})$. Hence spectral whitening and CE have a common minimiser if and only if $\Sigma_{\tilde\phi}^{-1/2} \in \mathcal{Z}_{\mathrm{CE}}$, i.e.\ only if this one orbit happens to be CE-optimal. If $\Sigma_{\tilde\phi}$ is singular (for instance $BK \leq D$), the zero-set is empty.
\end{enumerate}
\end{corollary}
 
\begin{proof}
\textbf{(i)} Immediate from Proposition~\ref{prop:spec}.
 
\textbf{(ii)} If $\Sigma_{\tilde\phi}$ is singular, $\mathrm{rank}\,\Sigma_y \leq \mathrm{rank}\,\Sigma_{\tilde\phi} < D$, so $\Sigma_y = I$ is impossible. Otherwise, $\Sigma_y(W_2) = W_2 \Sigma_{\tilde\phi} W_2^{\top} = I$ is equivalent to $(W_2 \Sigma_{\tilde\phi}^{1/2})(W_2 \Sigma_{\tilde\phi}^{1/2})^{\top} = I$, i.e., $W_2 \Sigma_{\tilde\phi}^{1/2} = Q$ for some $Q \in O(D)$, giving $W_2 = Q \Sigma_{\tilde\phi}^{-1/2}$. By Lemma~\ref{lem:ce_rotation}, $\mathcal{L}_{\mathrm{CE}}$ is constant on this orbit at value $\mathcal{L}_{\mathrm{CE}}(\Sigma_{\tilde\phi}^{-1/2})$, and the equivalence follows.
\end{proof}

\paragraph{Why that orbit is not expected to be CE-optimal (heuristic).}
Corollary~\ref{cor:spectral_infeasible}(ii) is exact. Whether its condition fails is an empirical matter, which the
following argument makes plausible but does not prove. The representative $\Sigma_{\tilde\phi}^{-1/2}$ rescales each
eigendirection of the backbone covariance by $\lambda_i^{-1/2}$. It down-weights high-variance directions and up-weights
low-variance ones. For self-supervised features, class-discriminative information tends to sit in high-variance
directions. When that holds, inverting the spectrum works against CE, and the CE-optimal orbit is a different one. The
argument would fail if class separability lay mainly in low-variance directions.
 
\paragraph{Remark on empirical signatures.}
Corollary~\ref{cor:spectral_infeasible} predicts that spectral whitening produces a CE-Spectral tension stronger than variance-hinge's: whereas VICReg allows a scaling workaround (Proposition~\ref{prop:vicreg_dynamics}(i)), Spectral forces a specific orbit that inverts the backbone spectrum. Empirically (Table~\ref{tab:decor_corr}), Spectral yields $\mathrm{noc} = 88.68$, within $0.53$\,pp of the no-decorrelation baseline ($88.15$) and $4.78$\,pp below cross-correlation --- consistent with the interpretation that Spectral's CE-incompatible orbit is neither reached nor approached, so the loss contributes no useful decorrelation signal.
 
\paragraph{Summary of the operator-level distinction.}
\begin{itemize}[leftmargin=1.5em, itemsep=0.2em]
\item \textbf{CC}: zero-set is the full set of $W_2$ producing diagonal $\Sigma_y$. It is reachable from any $W_2 \in \mathcal{Z}_{\mathrm{CE}}$ by an orthogonal rebasing along which CE is constant (Proposition~\ref{prop:cc_feasible}), with no rank assumption on $\Sigma_{\tilde\phi}$. CC exerts no pressure on per-dimension scale (Lemma~\ref{lem:cc_scale}).
\item \textbf{VICReg}: zero-set reachable from any $W_2 \in \mathcal{Z}_{\mathrm{CE}}$ by uniform scaling (Proposition~\ref{prop:vicreg_dynamics}(i)). Its gradient acts on per-dimension scale rather than along the scaling direction. The hinge is switched on by CE's own within-class contraction (Proposition~\ref{prop:var}), and at a non-degenerate CE-minimiser the VICReg step raises CE (Proposition~\ref{prop:vicreg_dynamics}(ii)).
\item \textbf{Spectral}: zero-set on L2-normalised embeddings is empty (Proposition~\ref{prop:spec}); on raw projections it is a single orthogonal orbit $O(D)\,\Sigma_{\tilde\phi}^{-1/2}$, which meets the CE-optimum set only if that one orbit is CE-optimal, and it is empty when $\Sigma_{\tilde\phi}$ is singular (Corollary~\ref{cor:spectral_infeasible}).
\end{itemize}
 
This taxonomy depends only on the structure of each loss as a function of $\Sigma_y$ and, for CE, on its invariance group. It does not depend on per-slot gradient geometry.

\subsubsection{Spectral Whitening Is Structurally Infeasible on L2-Normalized Embeddings}

Unlike CC, spectral whitening applied to L2-normalized embeddings cannot reach its zero-loss configuration on any pipeline state.

\begin{proposition}[Spectral whitening cannot vanish on L2-normalized embeddings]
\label{prop:spec}
Let $\Sigma_z = \frac{1}{BK} \sum_{b,k} z_k^{(b)} (z_k^{(b)})^{\top}$ denote the second-moment matrix of L2-normalized slot embeddings, and let $\mathcal{L}_{\mathrm{spec}}^{z} = \|\Sigma_z - I\|_F^2$ denote the spectral whitening penalty. Then for all $d > 1$,
\begin{equation}
\mathcal{L}_{\mathrm{spec}}^{z} \;\geq\; \frac{(d-1)^2}{d} \;>\; 0,
\end{equation}
with equality if and only if $\Sigma_z = \tfrac{1}{d} I$. The bound also holds for the centred covariance of the $z_k^{(b)}$, whose trace is $1 - \|\bar z\|^2 \leq 1$.
\end{proposition}

\begin{proof}
For any unit vector $z$, $\mathrm{tr}(zz^{\top}) = \|z\|^2 = 1$. Averaging over all $BK$ slots, $\mathrm{tr}(\Sigma_z) = 1$, hence $\mathrm{tr}(\Sigma_z - I) = 1 - d$. Since $\Sigma_z - I$ is symmetric, its squared Frobenius norm equals the sum of squared eigenvalues. Let $\{\lambda_i\}_{i=1}^{d}$ denote these eigenvalues. By Cauchy--Schwarz applied to $(\lambda_1, \ldots, \lambda_d)$ and $(1, \ldots, 1)$,
\begin{equation}
\|\Sigma_z - I\|_F^2 = \sum_{i=1}^{d} \lambda_i^2 \;\geq\; \frac{1}{d}\!\left(\sum_{i=1}^{d} \lambda_i\right)^{\!2} = \frac{(\mathrm{tr}(\Sigma_z - I))^2}{d} = \frac{(d-1)^2}{d}.
\end{equation}
Equality holds iff all $\lambda_i$ are equal, i.e., $\Sigma_z - I = -\tfrac{d-1}{d} I$, which gives $\Sigma_z = \tfrac{1}{d} I$. For the centred covariance, $\mathrm{tr}(\Sigma_z - I) \leq 1 - d$, so $|\mathrm{tr}(\Sigma_z - I)| \geq d - 1$ and the same bound follows. In our setting $d = D = 768$ (the projection space dimension on a DINOv2 ViT-B/14 backbone), the lower bound is $767^2/768 \approx 766.0$, far from zero.
\end{proof}

The spectral objective thus maintains a positive floor regardless of configuration, producing gradient pressure toward a region that is geometrically unreachable on the sphere. When spectral whitening is instead applied to raw projections $\check y_k^{(b)}$ (avoiding the trace constraint), the objective can in principle reach zero, but only at configurations where $\{\check y_k^{(b)}\}$ are isotropically distributed in $\mathbb{R}^D$---a configuration that destroys the cluster structure induced by CE along prototype-aligned directions. Empirically (Table~\ref{tab:alt_decor}), spectral whitening yields $\mathrm{noc} = 88.68$, within $0.53$\,pp of the no-decorrelation baseline ($88.15$), consistent with the spectral objective suppressing CE-induced clustering regardless of whether it is applied pre- or post-normalization.

\subsubsection{Variance-Hinge Conflicts With CE in the Tight-Cluster Regime}

\begin{proposition}[Variance-hinge activation regime]
\label{prop:var}
Let $\mathcal{L}_{\mathrm{var}} = \sum_d \max(0, \nu - \mathrm{std}(\check y_{\cdot d}))$ denote the variance-hinge penalty on raw projection dimensions. Consider the \emph{tight-cluster regime} in which CE has driven within-class variance of each dimension to zero, so that $\mathrm{std}(\check y_{\cdot d}) = \sigma_d^{\mathrm{bc}}$ equals the between-class standard deviation of the per-class means (assuming approximately balanced batch composition). If $\nu > \max_d \sigma_d^{\mathrm{bc}}$, then the hinge is active on every dimension at this configuration, and its gradient is
\begin{equation}
\frac{\partial \mathcal{L}_{\mathrm{var}}}{\partial \check y_{k,d}^{(b)}} = -\frac{1}{BK \cdot \mathrm{std}(\check y_{\cdot d})}\bigl(\check y_{k,d}^{(b)} - \mu_d\bigr),
\label{eq:var_grad}
\end{equation}
so that the gradient-descent update $\check y_{k,d}^{(b)} \leftarrow \check y_{k,d}^{(b)} - \eta\, \partial \mathcal{L}_{\mathrm{var}}/\partial \check y_{k,d}^{(b)}$ moves $\check y_{k,d}^{(b)}$ away from the dimension mean $\mu_d$.
\end{proposition}

\begin{proof}
When the hinge is active on dimension $d$, differentiating the hinge term directly yields the stated gradient. The condition $\nu > \max_d \sigma_d^{\mathrm{bc}}$ ensures activity on every dimension in the tight-cluster regime, where $\mathrm{std}(\check y_{\cdot d}) = \sigma_d^{\mathrm{bc}}$ for each $d$.
\end{proof}

The variance-hinge gradient has two structural issues for joint optimization with CE. First, unlike CC, it has a non-zero projection onto the tangent space of $z_k^{(b)}$ generically, because the vectorised gradient $(\check y_k^{(b)} - \mu) \in \mathbb{R}^D$ is a generic direction with both tangent and radial components. Second, the gradient direction pushes samples \emph{away} from dimension means $\mu_d$, which under approximately balanced batch composition coincide with the global batch mean $\mu$. As described in Remark~\ref{rem:vicreg_heuristic}, the resulting anti-mean push has a dominant between-class component that translates whole clusters radially away from $\mu$ and a smaller within-class component that pushes samples \emph{away} from their own class centroid $\mu_{c(k)}$. The latter directly opposes the CE gradient, which contracts samples toward $\mu_{c(k)}$, producing a saddle-point competition on the within-class subspace with equilibrium reached at a suboptimal CE configuration. This matches the empirical observation (Table~\ref{tab:alt_decor}) that VICReg yields the largest $\mathrm{noc}$ degradation among tested decorrelation objectives ($-6.13$ pp versus CC).

\subsection{Failure Mode I: VICReg Variance Constraint}
\label{app:vicreg_failure}

Proposition~\ref{prop:var} formalizes the activation regime of the variance-hinge penalty. Here we develop the failure dynamics in detail and connect them to the empirical $\mathrm{noc}$ degradation.

VICReg~\citep{vicreg} augments decorrelation with a variance-hinge term:
\begin{equation}
    \mathcal{L}_\text{var} = \sum_{d=1}^{D} \max(0,\;\nu - \text{std}(\check y_{\cdot d})),
    \label{eq:vicreg_var}
\end{equation}
where $\text{std}(\check y_{\cdot d})$ is the standard deviation of dimension $d$ across all $BK$ raw projections in the batch. When $\text{std}(\check y_{\cdot d}) < \nu$, the gradient is given by Eq.~\eqref{eq:var_grad}, and the gradient-descent update pushes $\check y_{k,d}^{(b)}$ \emph{away} from the dimension mean $\mu_d$.

Consider a well-trained CE model that produces tight class clusters in the aggregate embedding space. Let class $c$ occupy a small region around prototype $P_c$. For a batch containing multiple classes sampled approximately uniformly (as in the N-way K-shot episodes we use), $\text{std}(\check y_{\cdot d}) \approx \text{std}(\mu_c^{(d)})$ (between-class variance of the per-class means) plus a small within-class term. If the between-class separation is large, the hinge may be inactive. However, as training progresses and within-class variance shrinks (CE drives it toward zero), the total standard deviation can fall below $\nu$, activating the variance-hinge gradient.

Once active, this gradient pushes samples \emph{away from the global batch mean} $\mu$ through shared parameter $W_2$. By the decomposition in Remark~\ref{rem:vicreg_heuristic}, the dominant component translates whole clusters radially outward from $\mu$ (along the between-class direction $\mu_{c(k)} - \mu$), while a smaller within-class component pushes samples away from their own class centroid $\mu_{c(k)}$ --- directly opposing the CE gradient on aggregates. The resulting training dynamics exhibit a saddle-point competition: CE contracts within-class clusters, variance hinge expands them along the within-class subspace (and translates the cluster cloud outward along the between-class subspace), and convergence is reached at a suboptimal equilibrium where within-class clusters are larger than CE alone would produce. In the slot embedding space, this means prototypes are less discriminative, explaining the $-6.13$\,pp $\mathrm{noc}$ degradation vs.\ cross-correlation (Table~\ref{tab:alt_decor}).

The key asymmetry between VICReg and cross-correlation is the \emph{level of constraint}: cross-correlation is a second-order statistic between \emph{different dimensions} and constrains only cross-dimension correlations (Appendix~\ref{app:operator_distinction}), leaving per-dimension variances free; VICReg additionally constrains first- and second-order statistics \emph{within each dimension}, and its gradient has non-zero projection onto both tangent and radial components of each slot \emph{in a way that directly opposes CE's cluster-tightening}. This additional constraint eliminates the structural room for cooperation established in Appendix~\ref{app:operator_distinction}.

\subsection{Failure Mode II: Spectral Whitening Overconstraint}
\label{app:spectral_failure}

Proposition~\ref{prop:spec} establishes the structural infeasibility of spectral whitening on L2-normalized embeddings. Here we analyze the dynamics when spectral whitening is applied to raw projections $\check y_k^{(b)}$ (where the trace constraint is avoided) and connect them to the empirical $\mathrm{noc}$ result.

Spectral decorrelation minimises
\begin{equation}
    \mathcal{L}_\text{spec} = \|\Sigma_y - I\|_F^2
    = \sum_{i} (\Sigma_{y,ii} - 1)^2 + \sum_{i \neq j} \Sigma_{y,ij}^2,
    \label{eq:spectral_app}
\end{equation}
where $\Sigma_y$ is the batch covariance of raw projections (Eq.~\eqref{eq:sigma_y}). This penalises \emph{both} the diagonal ($\Sigma_{y,ii} = 1$, forcing unit per-dimension variance) and the off-diagonal ($\Sigma_{y,ij} = 0$, decorrelation).

The diagonal penalty is strictly more restrictive than VICReg's hinge: it applies a \emph{squared} penalty that is always active, not a one-sided hinge. The combined effect enforces full whitening of the raw projection space, mapping the distribution to an isotropic Gaussian-like shape in $\mathbb{R}^D$.

Whitening destroys the prototype cluster structure in two ways:
\begin{enumerate}[label=(\arabic*)]
    \item \textbf{Scale erasure.} Forcing $\Sigma_{y,ii} = 1$ removes the scale differences between dimensions. When high-variance dimensions carry more discriminative information, as is typical for self-supervised features (see the heuristic after Corollary~\ref{cor:spectral_infeasible}), CE relies on those differences. After whitening, all dimensions contribute equally, degrading the alignment between prototype geometry and discriminative geometry.
    \item \textbf{Cluster expansion/contraction along each dimension.} Let $\mathcal{L}_\text{spec}^\text{diag} = \sum_i (\Sigma_{y,ii} - 1)^2$ denote the diagonal-only component of $\mathcal{L}_\text{spec}$ (cf.\ Eq.~\eqref{eq:spectral_app}). Differentiating just this diagonal penalty,
    \begin{equation}
        \frac{\partial \mathcal{L}_\text{spec}^\text{diag}}{\partial \check y_{k,d}^{(b)}} = \frac{4(\Sigma_{y,dd}-1)\, \check y_{k,d}^{(b)}}{BK}.
    \end{equation}
    Under gradient descent, when $\Sigma_{y,dd} < 1$ (variance too small), the update increases $|\check y_{k,d}^{(b)}|$, expanding the cluster along dimension $d$; when $\Sigma_{y,dd} > 1$, the update decreases $|\check y_{k,d}^{(b)}|$, contracting it. This directly opposes CE's cluster-tightening effect throughout training, not only at a saddle point.
\end{enumerate}

The empirical result---Spectral gives sys $= 90.44$ (better than cross-corr's $90.17$) but $\mathrm{noc} = 88.68$ (worse than cross-corr's $93.46$)---is consistent with this analysis. Whitening increases inter-prototype separation along seen-concept directions (improving sys/pro) while disrupting the slot representation geometry needed for compositional transfer (reducing $\mathrm{noc}$). Notably, Spectral's $\mathrm{noc}$ ($88.68$, single seed) is only $0.53$\,pp above removing decorrelation entirely ($\lambda_d=0$: $\mathrm{noc} = 88.15$, 3 seeds), so Spectral's whitening recovers almost none of the novel-concept benefit that cross-correlation provides.

\subsection{Why Cross-Correlation Is Minimal Among Tested Objectives}
\label{app:cc_minimal}
 
Among the three decorrelation objectives compared (cross-correlation, variance-hinge, spectral whitening), cross-correlation is structurally minimal in the sense that its zero-loss set is the \emph{weakest} constraint among those that are CE-compatible. Concretely, $\mathcal{Z}_{\mathrm{cc}}$ is carved out by the $\binom{D}{2}$ off-diagonal constraints on $\Sigma_y$, and leaves three properties of the distribution of $\{\check y_k^{(b)}\}$ entirely free:
\begin{enumerate}[label=(\roman*), leftmargin=2em]
    \item \textbf{Per-dimension marginal variances $\{\Sigma_y{}_{dd}\}$}: unconstrained (any positive values permitted).
    \item \textbf{Per-slot magnitudes $\{\|\check y_k^{(b)}\|\}$}: unconstrained at the level of the batch covariance.
    \item \textbf{Ambient cluster geometry}: unconstrained beyond the off-diagonal decorrelation.
\end{enumerate}
 
Spectral whitening removes freedom (i) by forcing $\Sigma_y{}_{dd} = 1$, which \emph{cannot} be satisfied at a CE-optimum on L2-normalised embeddings (Proposition~\ref{prop:spec}), and on raw projections pins $W_2$ to a single orthogonal orbit that generically does not intersect $\mathcal{Z}_{\mathrm{CE}}$ (Corollary~\ref{cor:spectral_infeasible}). Variance-hinge removes (i) from below by imposing $\sigma_d \geq \nu$. A uniform scaling of $W_2$ could satisfy CE and VICReg simultaneously (Proposition~\ref{prop:vicreg_dynamics}(i)), but the VICReg gradient acts on per-dimension scale, not along uniform scaling, and at a non-degenerate CE-minimiser it raises CE (Proposition~\ref{prop:vicreg_dynamics}(ii)).
 
Cross-correlation is therefore the only tested objective that both shares minimisers with CE along a CE-invariant path, whatever the rank of $\Sigma_{\tilde\phi}$, and leaves per-dimension scale untouched (Lemma~\ref{lem:cc_scale}). The empirical ordering in Table~\ref{tab:decor_corr}, CC $>$ Spectral $\approx$ None $>$ VICReg on $\mathrm{noc}$, is consistent with this taxonomy.

\subsection{Visualising the Correlation Structure}
\label{app:decor_viz}

\begin{table}[ht]
\caption{Mean off-diagonal correlation magnitude $\bar{C}_\text{off} =
\frac{1}{D(D-1)}\sum_{i\neq j}|C_{ij}|$ (lower = better decorrelation)
and noc accuracy for each redundancy objective. Measured on CGQA after
convergence, single seed; the noc of the cross-correlation and None rows are the 3-seed means of Tables~\ref{tab:main_reduce}
and~\ref{tab:lambda_sensitivity}.}
\label{tab:decor_corr}
\centering
\begin{tabular}{lcc}
\toprule
\textbf{Objective} & $\bar{C}_\text{off}\downarrow$ & \textbf{noc (\%)} \\
\midrule
Cross-corr.\ (ours) & \textbf{0.042} & \textbf{93.46} \\
Spectral            & 0.044 & 88.68 \\
VICReg              & 0.049 & 87.33 \\
None ($\lambda_d=0$) & 0.197 & 88.15 \\
\bottomrule
\end{tabular}
\end{table}

Table~\ref{tab:decor_corr} reveals that while Spectral achieves near-identical
$\bar{C}_\text{off}$ to cross-correlation (0.044 vs.\ 0.042), its noc is
4.78\,pp lower.
This confirms that the failure of Spectral is not due to \emph{insufficient}
decorrelation but to \emph{excessive} constraint (whitening), consistent with
the analysis in Section~\ref{app:spectral_failure}.
VICReg achieves moderate decorrelation (0.049) but the worst noc among active
regularizers, attributable to the variance hinge conflict.
The baseline (no regularization) has the worst decorrelation (0.197) and
intermediate noc---demonstrating that some decorrelation is necessary, but
the form matters critically.

\subsection{Empirical Comparison of Redundancy Objectives}
\begin{table}[ht]
\caption{Alternative redundancy objectives on CGQA and COBJ. The cross-correlation row is COMPOSE
(mean over 3 seeds, Table~\ref{tab:main}) and the None row is $\lambda_d{=}0$ in Table~\ref{tab:lambda_sensitivity}
(mean over 3 seeds); Spectral and VICReg are single-seed. \textbf{Bold} = best per column.}
\label{tab:alt_decor}
\centering
\resizebox{\textwidth}{!}{%
\begin{tabular}{lccccccccccc}
\toprule
& \multicolumn{6}{c|}{\textbf{CGQA}} & \multicolumn{5}{c}{\textbf{COBJ}} \\
\textbf{Decorrelation} & sys & pro & sub & non & noc & $H_a\uparrow$
    & sys & pro & non & noc & $H_a\uparrow$ \\
\midrule
\textbf{Cross-corr.\ redundancy (ours)}
    & 90.17 & 94.17 & 93.23 & 91.20 & \textbf{93.46} & \textbf{92.42}
    & 81.38 & 66.73 & 76.07 & \textbf{88.32} & 77.31 \\
Spectral
    & 90.44 & \textbf{94.87} & \textbf{93.42} & \textbf{92.29} & 88.68 & 91.89
    & 81.40 & 65.87 & 75.03 & 87.94 & 76.68 \\
VICReg
    & 90.04 & 94.47 & 93.35 & 92.21 & 87.33 & 91.41
    & 81.17 & 65.63 & 74.83 & 87.80 & 76.47 \\
\midrule
None ($\lambda_d{=}0$)
    & \textbf{90.51} & 94.23 & 93.18 & 91.70 & 88.15 & 91.50
    & \textbf{81.48} & \textbf{67.48} & \textbf{76.55} & 87.32 & \textbf{77.51} \\
\bottomrule
\end{tabular}%
}
\end{table}
The theoretical failure modes analysed in Sections~C.3--C.4 predict a clear
performance ordering: cross-correlation redundancy reduction should achieve
the highest noc, followed by Spectral (overconstraint) and VICReg (direct conflict).
Table~\ref{tab:alt_decor} confirms this prediction empirically on both CFST datasets.

Cross-correlation redundancy reduction achieves the largest noc margin over
alternatives ($+4.78$\,pp over Spectral,
$+6.13$\,pp over VICReg on CGQA).
Crucially, these gaps are not attributable to differences in decorrelation
strength: as reported in Table~\ref{tab:decor_corr}, Spectral achieves
near-identical $\bar{C}_{\mathrm{off}}$ to cross-correlation ($0.044$ vs.\
$0.042$) yet loses $4.78$\,pp noc, confirming that the failure of Spectral
is not insufficient decorrelation but excessive constraint (whitening),
consistent with Section~C.4.
Similarly, VICReg achieves moderate decorrelation ($0.049$) but the worst
noc among active regularisers ($-6.13$\,pp vs.\ cross-correlation),
attributable to the variance-hinge conflict identified in Section~C.3.
A secondary pattern confirms the noc--sys trade-off identified in
Section~6.2: Spectral achieves the highest sys on CGQA among the regularisers ($90.44$ vs.\
$90.17$ for cross-correlation), consistent with whitening increasing
inter-prototype separation along seen-concept directions at the cost of
novel-concept geometry.
Removing decorrelation entirely ($\lambda_d{=}0$) yields noc $= 88.15$
on CGQA, within $0.53$\,pp of single-seed Spectral ($88.68$): Spectral's
whitening recovers almost none of the $5.31$\,pp noc that cross-correlation adds
over no regularisation. On COBJ, $\lambda_d{=}0$ gives slightly higher sys, pro, non
and $H_a$ than cross-correlation but $1.00$\,pp lower noc, so decorrelation pays off
on the novel-concept axis.
This confirms the conclusion of Section~C.6: cross-correlation is the
\emph{minimal sufficient} decorrelation mechanism for prototype-based
compositional learning.

\section{Additional Analyses and Sensitivity Studies}

\subsection{Redundancy Reduction Threshold (\texorpdfstring{$\lambda_d$}{lambda\_d})}
\label{app:redundancy_lamda}

Performance jumps sharply from $\lambda_d{=}0$ to $\lambda_d{=}0.005$
($+4.84$\,pp noc on CGQA) then plateaus: from $\lambda_d{=}0.005$ to $0.1$, CGQA noc stays within $0.54$\,pp.
This threshold effect is consistent with the theoretical analysis in
Section~C.2: a minimal decorrelation pressure is necessary to prevent slot
redundancy, but once off-diagonal correlations are suppressed below a critical level, further tightening yields diminishing returns. We use $\lambda_d{=}0.02$ across all experiments. Full sensitivity results are reported in Table~\ref{tab:lambda_sensitivity}.

\begin{table}[ht]
\centering
\caption{
  \textbf{Effect of Barlow Twins weight $\lambda_d$ on CGQA and COBJ.}
  $\lambda_d{=}0$ = no decorrelation; default $\lambda_d{=}0.02$ in \textbf{bold}.
  10-way 10-shot, mean over 3 seeds; the $0.02$ row is Table~\ref{tab:main_reduce}.
}
\label{tab:lambda_sensitivity}
\small
\begin{tabular}{lcccc}
\toprule
$\lambda_d$ & $H_a\uparrow$ CGQA & noc CGQA & $H_a\uparrow$ COBJ & noc COBJ \\
\midrule
$0.0$          & 91.50          & 88.15          & 77.51          & 87.32 \\
$0.005$        & 92.42          & 92.99          & 77.54          & 88.07 \\
$0.01$         & 92.44          & 93.33          & 77.45          & 88.24 \\
$\mathbf{0.02}$& \textbf{92.42} & \textbf{93.46} & \textbf{77.31} & \textbf{88.32} \\
$0.05$         & 92.46          & 93.53          & 77.22          & 88.48 \\
$0.1$          & 92.49          & 93.52          & 77.14          & 88.56 \\
\bottomrule
\end{tabular}
\end{table}

\subsection{Blend Coefficient \texorpdfstring{$\alpha_{\mathrm{blend}}$}{alpha\_blend}}
\label{app:sensitivity_alpha}

Throughout this appendix $\alpha_{\mathrm{blend}}$ denotes the same coefficient written $\gamma$ in Eq.~\eqref{eq:blend}: the weight on the holistic score, so Chamfer receives $1-\alpha_{\mathrm{blend}}$.

Table~\ref{tab:alpha_sensitivity} reveals an asymmetric trade-off around the default
$\alpha{=}0.3$.
Reducing $\alpha$ below 0.3 (increasing Chamfer weight at inference) degrades
\emph{both} sys and noc on CGQA: at $\alpha{=}0.0$, sys drops from 90.13 to
82.76 and noc from 93.43 to 90.68, indicating that pure Chamfer inference loses
the holistic class identity encoded by the router.
Increasing $\alpha$ beyond 0.3 improves sys marginally (peaking at 0.4)
but reduces noc, consistent with the noc--sys trade-off identified in
Section~6.2: the holistic branch provides the compositional invariance needed
for novel-concept generalisation, while the Chamfer branch contributes
seen-concept discriminability.
COBJ prefers a higher holistic weight ($\alpha_{\mathrm{blend}}{=}0.6$--$0.7$,
$\mathrm{noc}{=}88.86\%$).
Full sensitivity results appear in Table~\ref{tab:alpha_sensitivity}.

\begin{table}[t]
\centering
\caption{$\alpha_{\mathrm{blend}}$ sensitivity on CGQA and COBJ (10-way 10-shot, mean over 3 seeds; same runs as
Table~\ref{tab:auto_gamma}). Default $\alpha{=}0.3$ \textbf{bold}.}
\label{tab:alpha_sensitivity}
\begin{tabular}{l cccc c cccc}
\toprule
& \multicolumn{4}{c}{\textbf{CGQA}} & & \multicolumn{4}{c}{\textbf{COBJ}} \\
\cmidrule(lr){2-5} \cmidrule(lr){7-10}
$\alpha$ & sys & pro & sub & non & noc & sys & pro & non & noc \\
\midrule
0.0 & 82.76 & 88.62 & 87.08 & 84.38 & 90.68 & 76.00 & 57.33 & 64.70 & 84.57 \\
0.1 & 87.32 & 92.24 & 90.75 & 88.87 & 92.64 & 79.38 & 62.62 & 71.68 & 86.09 \\
0.2 & 89.37 & 93.92 & 92.25 & 90.78 & 93.36 & 80.85 & 65.29 & 74.82 & 87.24 \\
\textbf{0.3} & \textbf{90.13} & \textbf{94.50} & \textbf{92.90} & \textbf{91.52} & \textbf{93.43} & \textbf{81.40} & \textbf{66.65} & \textbf{76.23} & \textbf{87.93} \\
0.4 & 90.31 & 94.59 & 93.13 & 91.68 & 93.17 & 81.57 & 67.49 & 76.97 & 88.38 \\
0.5 & 90.21 & 94.44 & 93.13 & 91.68 & 92.83 & 81.48 & 67.77 & 77.34 & 88.64 \\
0.6 & 89.94 & 94.08 & 93.00 & 91.46 & 92.38 & 81.33 & 67.95 & 77.53 & 88.84 \\
0.7 & 89.60 & 93.67 & 92.80 & 91.24 & 91.95 & 81.07 & 67.99 & 77.55 & 88.86 \\
0.8 & 89.26 & 93.20 & 92.59 & 90.97 & 91.46 & 80.73 & 67.91 & 77.47 & 88.78 \\
0.9 & 88.93 & 92.68 & 92.37 & 90.73 & 90.98 & 80.44 & 67.79 & 77.35 & 88.67 \\
1.0 & 88.62 & 92.14 & 92.09 & 90.48 & 90.55 & 80.15 & 67.58 & 77.25 & 88.53 \\
\bottomrule
\end{tabular}
\end{table}

\subsection{Shot Sensitivity}
\label{app:shot_sensitivity}

The noc gap between COMPOSE and COMPOSE-CT is largest at 1-shot and narrows
monotonically as the number of support shots increases, as shown in
Table~\ref{tab:shot_sensitivity}.
This confirms that holistic training is most critical in the low-data regime:
with only one support example per class, the holistic prototype must capture
sufficient class identity from a single image, and overspecialised slot
representations (COMPOSE-CT) fail to produce a reliable prototype.
As shot count increases, prototype averaging smooths out slot-level noise,
partially compensating for overspecialisation---but the noc advantage of
holistic training persists at all shot levels evaluated
($+4.99$\,pp at 1-shot, $+2.69$\,pp at 10-shot, $+2.42$\,pp at 20-shot).

\begin{table}[h]
\caption{Shot sensitivity on CGQA (10-way): noc (\%) and $H_a$ (\%) for
COMPOSE and COMPOSE-CT across support shots (mean over 3 seeds).
$\Delta$ = COMPOSE $-$ COMPOSE-CT.}
\label{tab:shot_sensitivity}
\centering
\small
\begin{tabular}{l cccc}
\toprule
& \multicolumn{4}{c}{\textbf{noc (\%)}} \\
\cmidrule(lr){2-5}
\textbf{Method} & 1-shot & 5-shot & 10-shot & 20-shot \\
\midrule
COMPOSE  & 74.67 & 91.22 & 93.46 & 94.83 \\
COMPOSE-CT & 69.69 & 87.88 & 90.77 & 92.41 \\
\midrule
$\Delta$                  & $+4.99$ & $+3.34$ & $+2.69$ & $+2.42$ \\
\bottomrule
\end{tabular}
\end{table}

\subsection{Alternative Structured Matchers}
\label{app:alt_matchers}

To verify that the \emph{train holistically, infer compositionally} principle
generalises beyond Chamfer distance, we evaluate three configurations across
shot counts on CGQA: COMPOSE (holistic encoder + Chamfer inference),
COMPOSE-Hungarian (holistic encoder + Hungarian assignment inference), and COMPOSE-CT.
Table~\ref{tab:alt_matchers} reports noc and $H_a$ across 1- to 20-shot.

\begin{table}[H]
\caption{Matcher/encoder ablation on CGQA (10-way): noc (\%) and $H_a$ (\%)
across shot counts (mean over 3 seeds). COMPOSE and COMPOSE-Hungarian share the same trained encoder for
each seed and differ only in the test-time matcher; the COMPOSE and COMPOSE-CT rows are those of
Table~\ref{tab:shot_sensitivity}.}
\label{tab:alt_matchers}
\centering
\small
\begin{tabular}{l cccc cccc}
\toprule
& \multicolumn{4}{c}{\textbf{noc (\%)}}
& \multicolumn{4}{c}{$\boldsymbol{H_a}$ \textbf{(\%)}} \\
\cmidrule(lr){2-5} \cmidrule(lr){6-9}
\textbf{Variant}
    & 1-sh & 5-sh & 10-sh & 20-sh
    & 1-sh & 5-sh & 10-sh & 20-sh \\
\midrule
COMPOSE     & 74.67 & 91.22 & 93.46 & 94.83
                            & 73.71 & 90.21 & 92.42 & 93.58 \\
COMPOSE-Hungarian             & 73.11 & 90.66 & 92.97 & 94.44
                            & 72.19 & 89.39 & 91.56 & 92.72 \\
COMPOSE-CT  & 69.69 & 87.88 & 90.77 & 92.41
                            & 77.35 & 91.42 & 93.15 & 94.03 \\
\bottomrule
\end{tabular}
\end{table}

Three observations are consistent with the dual-phase claim.
First, COMPOSE-Hungarian gives stronger noc than COMPOSE-CT across all shot
counts ($+3.43$\,pp at 1-shot, $+2.19$\,pp at 10-shot), confirming that
the noc advantage comes from the holistically trained encoder rather than from
the Chamfer matcher: read with Hungarian matching, it still beats a Chamfer-trained
encoder read with the Chamfer matcher it was trained for.
Second, COMPOSE-CT achieves higher $H_a$ at every shot count because
Chamfer training improves discriminability on the other four axes (sys, pro, sub, non),
consistent with the noc--sys trade-off established in Section~6.2.
Third, COMPOSE is strongest in noc at every shot count, and its lead over
COMPOSE-Hungarian is largest at 1-shot ($+1.56$\,pp) and shrinks to $+0.39$\,pp at 20-shot,
suggesting that highly sparse support favours Chamfer's many-to-one matching over a one-to-one assignment.
Together, these results indicate that the noc gain does not depend on Chamfer
being the test-time matcher: swapping in Hungarian assignment changes noc by at most
$1.56$\,pp, far less than swapping the training objective ($-4.99$ to $-2.42$\,pp,
Table~\ref{tab:shot_sensitivity}). This supports the claim of Section~4.2 that
the compositional trap is a property of the \emph{training objective}. We did not
train an encoder with a Hungarian loss, so the table does not test whether every
structured training matcher behaves like Chamfer; Section~4.2 argues this analytically.

\subsection{Backbone Scale Plug-in}
\label{app:supervised_vit}
\begin{table}[t]
\caption{Backbone scale plug-in on CGQA. The DINOv2 row is the 3-seed mean of Table~\ref{tab:main};
the supervised ViT-H/14 row is a single seed.
All other COMPOSE components fixed; same DINOSAUR attention maps reused.
$\Delta$ = ViT-H/14 $-$ DINOv2 B/14.}
\label{tab:backbone_scale}
\centering
\small
\begin{tabular}{lccccc}
\toprule
\textbf{Patch features} & sys & pro & sub & non & noc \\
\midrule
DINOv2 ViT-B/14 (LVD-142M) & 90.17 & 94.17 & 93.23 & 91.20 & 93.46 \\
Supervised ViT-H/14 (IN-21K) & 76.07 & 80.87 & 83.30 & 79.34 & 82.75 \\
\midrule
$\Delta$
    & $-14.10$ & $-13.30$ & $-9.93$ & $-11.86$ & $-10.71$ \\
\bottomrule
\end{tabular}%
\end{table}
To isolate the contribution of backbone scale from pretraining objective,
we replace the DINOv2 ViT-B/14 patch features with a \emph{supervised} ViT-H/14 (IN-21K) backbone while reusing the same DINOSAUR attention maps.
Despite the larger model capacity, noc drops by $10.71$\,pp on CGQA
($93.46\% \to 82.75\%$), with consistent degradation across all splits ($\Delta$sys $= -14.10$\,pp, $\Delta$pro $= -13.30$\,pp). A larger supervised backbone cannot compensate for the absence of local self-distillation objectives, confirming that the bottleneck for compositional transfer lies in the \emph{quality} of patch geometry rather than model capacity. Full per-split results are reported in Table~\ref{tab:backbone_scale}.

\subsection{Computational Efficiency}
\label{app:computational_eff}
COMPOSE introduces minimal computational overhead over a frozen DINOv2 baseline. The trainable router and projection head together contain ${\approx}640$k parameters---less than $1\%$ of the DINOv2 ViT-B/14 backbone (${\approx}86$M parameters, $640\text{k}/86\text{M} \approx 0.74\%$). Inference runs at $66.7$\,ms per episode on a single H100 GPU, and full continual training across all sessions requires ${\approx}5$\,min per dataset
with no test-time adaptation. This efficiency profile compares favourably with backbone fine-tuning methods such as CoFiMA (${\approx}4$\,h per dataset) while delivering superior noc
performance, confirming that the gains of COMPOSE are architectural and
objective-level rather than a consequence of increased compute.

\subsection{Phase~II on a Real Query}
\label{app:worked_example}

Figure~\ref{fig:example} follows one query through Phase~II using the numbers of an actual run: COMPOSE
trained on COBJ with seed 42, one of the three seeds averaged in Table~\ref{tab:main}. Rerunning this seed
reproduces its accuracy on every split to within $0.03$\,pp. The query comes from the 10-way 10-shot
novel-concept episodes and was selected by criteria fixed before any image was inspected. It is classified
correctly by the blended score and by each branch alone, its margin over the runner-up class is above the
median of correctly classified queries, and its class is in the more accurate half of the 19 novel-concept
classes. It is therefore a clear case rather than a typical one. On this split the router weights are close to
uniform: over the 4{,}001 novel-concept images, the largest weight has a median of $0.29$, against $0.37$ to
$0.41$ on the seen-concept splits. The weights in the figure should therefore not be read as object importance.

\begin{figure}[t]%
    \centering
    \includegraphics[width=\textwidth]{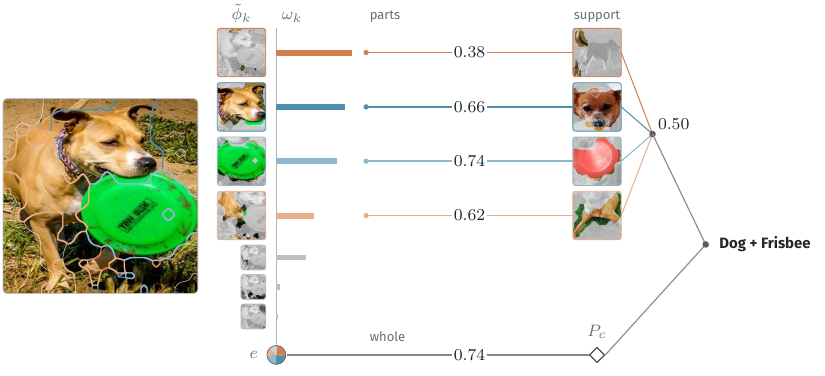}
    \caption{\textbf{Phase~II on one real query.} A COBJ novel-concept query (Dog + Frisbee; neither object appears in the training sessions; the query image of Figure~\ref{fig:pipeline}) in a 10-way 10-shot episode. Slot attention splits the image into $K = 7$ slots; each tile shows where a slot's attention is largest, while $\tilde{\phi}_k$ averages the patch tokens with the soft attention weights. Bars give the router weights $\omega_k$ ($0.27$, $0.25$, $0.22$ and $0.14$ for the four kept slots; $0.11$, $0.01$ and $0.01$ for the others), so the two background slots contribute almost nothing. \emph{Parts:} each of the top $\kappa = 4$ slots is linked to its nearest slot among the $40$ kept slots of the class's $10$ support images, with the cosine on the link; the frisbee slot finds a frisbee of another colour. Averaged over both matching directions, these give $s_{\mathrm{Ch}} = 0.50$. \emph{Whole:} the embedding $e$, drawn with sectors proportional to $\omega_k$, has cosine $0.74$ with the class prototype $P_c$. With $\tau = 12.88$, the blended score of Eq.~\eqref{eq:blend} is $0.3 \cdot 9.55 + 0.7 \cdot 6.39 = 7.34$ for Dog + Frisbee, against $1.45$ for the runner-up class.}
    \label{fig:example}
\end{figure}

\section{Disentangling Decomposition, Training Objective, and Inference Matching}
\label{sec:disentangle}

Sections~\ref{subsec:trap} and~\ref{sec:main_results} make two claims: part-level matching is harmful in the
\emph{training} loss but useful in the \emph{inference} score. A single ablation column cannot distinguish
these effects, so our cross-ablation (Table~\ref{tab:cross_ablation}) varies $\gamma$ only at test time,
without retraining any encoder.

\paragraph{The holistic and Chamfer branches are complementary at inference, and the best $\gamma$ lies
strictly between $0$ and $1$.}
$\gamma$ is the weight on the holistic score, so the Chamfer score receives the weight $1-\gamma$. Looking only
at the two endpoints, $\gamma{=}0$ and $\gamma{=}1.0$, would be misleading, because the best setting lies
between them. On CGQA, the noc of COMPOSE rises from $90.79$ at $\gamma{=}0$ to $93.48$ at $\gamma{=}0.3$, and
then falls again to $90.86$ at $\gamma{=}1.0$ (Table~\ref{tab:cross_ablation}). A finer 3-seed sweep in
Table~\ref{tab:alpha_sensitivity} confirms the pattern. \textbf{At the best setting, the Chamfer score
receives $70\%$ of the weight in the blended score.} Removing either branch changes noc by about the same
amount: the change is $-2.69$ when the holistic branch is removed and $-2.62$ when the Chamfer branch is
removed. On $H_a$, the two removals differ much more, with changes of $-5.87$ and $-1.70$ respectively. This
larger difference on $H_a$ comes from the \emph{seen}-concept splits. For example, sys falls from $90.18$ to
$82.80$ when the holistic branch is removed. The seen-concept splits are where the holistic branch carries
class identity. Since these splits drive the $H_a$ difference, $H_a$ cannot identify which branch supports
transfer to novel concepts.

\paragraph{The asymmetry lies in training: Chamfer in the training loss lowers noc.}
COMPOSE and COMPOSE-CT compute the same score at test time. The two models differ only in whether Chamfer
gradients reach the encoder during training. Table~\ref{tab:cross_ablation} evaluates both trained encoders at
four values of $\gamma$. \textbf{On CGQA, COMPOSE-CT has lower noc than COMPOSE at every $\gamma$ and at every
seed, and the gap is largest when the test-time score contains no Chamfer term at all.} At $\gamma{=}1.0$,
COMPOSE-CT is $6.31$\,pp behind COMPOSE, even though the test-time score contains no Chamfer term. The deficit
must have arisen during training. At $\gamma{=}0$, the holistically trained encoder is better even on
Chamfer's own score ($-2.41$\,pp), which separates the two effects. The dose curve on the right of
Table~\ref{tab:cross_ablation} extends the comparison to five Chamfer weights in the \emph{training} loss,
evaluating every encoder at $\gamma{=}1.0$. As the training weight increases, noc falls monotonically by
$10.7$\,pp, while $H_a$ changes by at most $1.3$\,pp whenever Chamfer is not the sole training signal. The
damage is thus concentrated on novel concepts. The five encoders keep
the same ordering by noc when all of them are evaluated at $\gamma{=}0$ instead ($90.79$, $90.44$, $89.79$,
$88.38$, $83.15$).

\begin{table}[t]
\caption{\textbf{Train $\times$ inference cross-ablation on CGQA} (noc \%, 10-way 10-shot, 3 seeds).
$\gamma$ is the holistic weight and affects test time only; no encoder is retrained across columns.
\emph{Right:} dose curve, all five encoders evaluated at $\gamma{=}1.0$.
The cells come from a separate set of runs of the Table~\ref{tab:main_reduce} configuration (COMPOSE has noc
$93.48$ here and $93.46$ there).
The COBJ matrix is in Appendix~\ref{app:cross_ablation_full}.}
\label{tab:cross_ablation}
\centering
\footnotesize
\begin{tabular}{lcccc}
\toprule
\textbf{Training} & $\gamma{=}0$ & $\gamma{=}0.3$ & $\gamma{=}0.7$ & $\gamma{=}1.0$ \\
\midrule
Holistic (COMPOSE)          & 90.79 & \textbf{93.48} & 92.13 & 90.86 \\
Holistic + Chamfer (CT)     & 88.38 & 90.76 & 87.39 & 84.55 \\
Chamfer-only                & 83.15 & 83.54 & 84.80 & 80.19 \\
\midrule
$\Delta$ (CT $-$ COMPOSE)   & $-2.41$ & $\mathbf{-2.72}$ & $-4.74$ & $-6.31$ \\
\bottomrule
\end{tabular}
\hfill
\begin{tabular}{lc c}
\toprule
Chamfer weight & \multicolumn{2}{c}{read at $\gamma{=}1.0$} \\
\cmidrule(lr){2-3}
in \emph{training} & noc & $H_a$ \\
\midrule
0 (COMPOSE)     & 90.86 & 90.72 \\
0.3             & 89.44 & 90.74 \\
0.5             & 87.39 & 90.35 \\
0.7             & 84.55 & 89.49 \\
1.0 (Chamfer)   & 80.19 & 83.43 \\
\bottomrule
\end{tabular}
\end{table}

\paragraph{The noc cost of Chamfer training is not specific to one backbone.}
Table~\ref{tab:ct_gap_backbone} repeats the comparison between COMPOSE and COMPOSE-CT for four patch sources,
with the test-time computation held fixed. For \emph{every} backbone, COMPOSE-CT is better than COMPOSE on sys,
pro, sub and non but worse on noc. The trade-off therefore appears four times rather than once.
Appendix~\ref{app:shot_sensitivity} adds four shot-count comparisons with test-time computation fixed; the noc
gaps are $+4.99$ / $+3.34$ / $+2.69$ / $+2.42$\,pp at 1 / 5 / 10 / 20-shot. Two further results agree with
this pattern. First, COMPOSE-Hungarian
evaluates a holistically trained encoder with a \emph{different} structured matcher. At 10-shot,
COMPOSE-Hungarian reaches noc $92.97$. This is still $+2.19$ above COMPOSE-CT evaluated with the Chamfer
matcher that COMPOSE-CT was trained for ($90.77$, Appendix~\ref{app:alt_matchers}). Second, TransSlot, which is
trained with Chamfer only, falls to noc $69.68$. This is $13.60$\,pp \emph{below} zero-shot DINOv2-Proto
($83.28$).

\begin{table}[t]
\caption{\textbf{The noc cost of Chamfer in the training loss is backbone-invariant} (CGQA, 10-way 10-shot,
3 seeds). Test-time computation is identical for the two rows of every pair. Full per-split breakdowns,
including COBJ, in Appendix~\ref{app:full_backbone_ablation_results}.}
\label{tab:ct_gap_backbone}
\centering
\footnotesize
\begin{tabular}{lccc}
\toprule
\textbf{Patch source} & COMPOSE noc & COMPOSE-CT noc & $\Delta$ noc \\
\midrule
DINOv1 ViT-B/16~\citep{dinov1}    & $88.69_{\pm0.27}$ & $85.70_{\pm1.04}$ & $\mathbf{+2.99}$ \\
iBOT ViT-B/16~\citep{ibot}        & $90.09_{\pm0.62}$ & $87.91_{\pm0.29}$ & $\mathbf{+2.18}$ \\
DINOv2 ViT-B/14~\citep{dinov2}    & $93.46_{\pm0.33}$ & $90.77_{\pm0.53}$ & $\mathbf{+2.69}$ \\
OpenCLIP ViT-L/14~\citep{openclip}& $89.71_{\pm0.59}$ & $87.26_{\pm0.06}$ & $\mathbf{+2.45}$ \\
\bottomrule
\end{tabular}
\end{table}

\paragraph{On COBJ, the gap has the same direction but is smaller.}
On COBJ, the best $\gamma$ for COMPOSE is $0.7$, where noc is $89.29$. As on CGQA, the best $\gamma$ lies
strictly between $0$ and $1$. At $\gamma = 0.3$, $0.7$ and $1.0$, COMPOSE-CT has lower noc than COMPOSE at all
three seeds. The mean gaps at these three values are $-0.48$, $-0.98$ and $-1.56$\,pp, so the size of the gap
grows with $\gamma$. At $\gamma{=}0$, the two models are tied. The per-seed gaps at $\gamma{=}0$ are $+0.00$,
$+0.22$ and $+0.13$\,pp, and all three lie inside the $\pm 0.7$\,pp episode CI95. We therefore read the
$\gamma{=}0$ cell as noise rather than as a reversal. For this reason, we make the ``at every $\gamma$''
statement for CGQA only. For COBJ, we report that the gap has the same direction as on CGQA but a smaller
margin.

\subsection{Protocol and the COBJ Matrix}
\label{app:cross_ablation_full}

Table~\ref{tab:cross_ablation} gives the cross-ablation matrix for CGQA. This subsection describes the protocol
of the cross-ablation and gives the corresponding matrix for COBJ.

\paragraph{Protocol.} $\gamma$ appears only in Eq.~\eqref{eq:blend}, and Eq.~\eqref{eq:blend} is evaluated
after all training has finished. Each trained encoder is therefore evaluated at every value of $\gamma$, and no
encoder is retrained along a row of the matrix. We compare three training regimes: holistic only (COMPOSE),
holistic plus Chamfer (COMPOSE-CT), and Chamfer only. The Chamfer-only regime keeps the rest of the COMPOSE
architecture unchanged. Each cell uses 3 seeds (42, 123, 7) in the 10-way 10-shot setting.

\paragraph{Why the matrix reports noc rather than $H_a$.} At low $\gamma$, ranking COMPOSE and COMPOSE-CT by
$H_a$ gives the opposite order to ranking them by noc. At $\gamma{=}0$, COMPOSE-CT has $H_a$ $90.04$, against
$86.56$ for COMPOSE. As in Section~\ref{sec:ablation}, Chamfer training improves the \emph{seen}-concept
splits that dominate $H_a$ while lowering noc; $H_a$ alone would miss the damage to novel concepts, so we
report noc.

\paragraph{The dose curve evaluated at both ends of the $\gamma$ range.} The monotone decline in noc shown by the
dose curve in Table~\ref{tab:cross_ablation} is measured at $\gamma{=}1.0$, where the test-time score contains no
Chamfer term. We also evaluate the same five encoders at $\gamma{=}0$, where the test-time score is the Chamfer
score alone. At $\gamma{=}0$, their noc values are $90.79$, $90.44$, $89.79$, $88.38$ and $83.15$. These
values have the same ordering as at $\gamma{=}1.0$, so the conclusion does not depend on which end of the
$\gamma$ range is used to evaluate the encoders.

\begin{table}[ht]
\caption{\textbf{COBJ cross-ablation} (noc \%, 3 seeds, 10-way 10-shot; COBJ has no \emph{sub} split).
Both models see the same episodes per seed, so the comparison is made seed by seed.}
\label{tab:cross_ablation_cobj}
\centering
\footnotesize
\begin{tabular}{lcccc}
\toprule
\textbf{Training} & $\gamma{=}0$ & $\gamma{=}0.3$ & $\gamma{=}0.7$ & $\gamma{=}1.0$ \\
\midrule
Holistic (COMPOSE)      & 84.81 & 88.32 & \textbf{89.29} & 88.92 \\
Holistic + Chamfer (CT) & 84.92 & 87.83 & 88.31 & 87.36 \\
Chamfer-only            & 85.01 & 85.67 & 87.44 & 88.16 \\
\midrule
$\Delta$ (CT $-$ COMPOSE) & $+0.12$ & $-0.48$ & $-0.98$ & $-1.56$ \\
\bottomrule
\end{tabular}
\end{table}

On COBJ, the best $\gamma$ for COMPOSE again lies strictly between $0$ and $1$, at $\gamma{=}0.7$ (noc
$89.29$). At $\gamma = 0.3$, $0.7$ and $1.0$, COMPOSE-CT has lower noc than COMPOSE at all three seeds. The
size of the mean gap grows with $\gamma$ ($-0.48$, $-0.98$, $-1.56$\,pp), which is the same pattern as on
CGQA. At $\gamma{=}0$, the two models are tied. The per-seed gaps at $\gamma{=}0$ are $+0.00$, $+0.22$ and
$+0.13$\,pp, and all three lie inside the $\pm 0.7$\,pp episode CI95. We therefore read the $\gamma{=}0$ cell
as noise rather than as a reversal. For the same reason, we restrict the statement that the gap is
``negative at every $\gamma$'' to CGQA.

\section{What the Patch Source Must Supply}
\label{sec:patch_source}

COMPOSE builds its part representations from frozen patch tokens, so what it can represent is limited by what
the patch encoder was trained to represent. Table~\ref{tab:backbone_type_reduce} orders three self-distilled
backbones (DINOv1, iBOT, DINOv2) by noc. Because they share one family of training objectives, their ordering
measures how much patch quality matters but does not identify which property of the objective supplies that
quality. We therefore vary the objective by adding three other self-supervised or label-free sources:
EVA02~\citep{eva02}, trained by masked modelling against CLIP \emph{features}; MAE~\citep{mae}, trained by
masked modelling against \emph{pixels}; and SAM~\citep{sam}, trained with mask supervision and no class
labels. All other components remain fixed: the backbone is frozen, DINOSAUR attention maps are unchanged, and
only patch features are swapped.

\paragraph{Choosing the block from which patch features are taken.}
If patch features from every backbone were taken from its last block, MAE would be disadvantaged for reasons
unrelated to its training objective: its final blocks feed a pixel decoder and serve reconstruction rather
than semantic content. We therefore fix one layer per model \emph{before} running any CFST episode, choosing
it by slot purity on COCO val2017 ($5000$ images, $K{=}7$). For MAE, purity
peaks in the middle of the encoder: across blocks 7/15/23/31, the purity of MAE is $0.4186$ /
$\mathbf{0.6217}$ / $0.5982$ / $0.5564$. Table~\ref{tab:patch_source} reports MAE with features from both its
best layer and its last layer, so that the effect of depth can be separated from the effect of the training
objective. The DINOv2 row of Table~\ref{tab:patch_source} is the result of Table~\ref{tab:main_reduce},
obtained with the same pipeline as the remaining rows.

\begin{table}[t]
\caption{\textbf{Patch source and read depth} (frozen backbone, same attention maps, only patch features
swapped; 3 seeds). COCO purity is measured before any CFST episode is
run and ranks all four accuracy columns exactly (Spearman $\rho = +1.000$, $n = 6$); NMI does not
($\rho = +0.486$).}
\label{tab:patch_source}
\centering
\footnotesize
\setlength{\tabcolsep}{3.5pt}
\resizebox{\textwidth}{!}{%
\begin{tabular}{llccccc}
\toprule
\textbf{Backbone (layer read)} & \textbf{Pretraining target} & \textbf{COCO purity}
  & \multicolumn{2}{c}{\textbf{CGQA}} & \multicolumn{2}{c}{\textbf{COBJ}} \\
\cmidrule(lr){4-5}\cmidrule(lr){6-7}
 & & & noc & $H_a$ & noc & $H_a$ \\
\midrule
DINOv2 ViT-B/14, last  & self-distillation      & \textbf{0.9046} & \textbf{93.46}$_{\pm.33}$ & \textbf{92.42} & \textbf{88.32}$_{\pm.49}$ & \textbf{77.31} \\
EVA02 ViT-B/14, last   & masked, CLIP features  & 0.8408 & $90.31_{\pm.19}$ & 88.85 & $86.21_{\pm.20}$ & 75.28 \\
MAE ViT-H/14, block 15 & masked, pixels         & 0.6217 & $85.33_{\pm.18}$ & 88.42 & $72.28_{\pm.68}$ & 66.98 \\
MAE ViT-H/14, block 31 & masked, pixels         & 0.5564 & $79.63_{\pm.21}$ & 84.74 & $62.28_{\pm1.26}$ & 57.27 \\
SAM ViT-B, block 11    & masks, no class labels & 0.4879 & $65.46_{\pm.40}$ & 68.28 & $61.05_{\pm.65}$ & 50.49 \\
SAM ViT-B, block 8     & masks, no class labels & 0.4872 & $65.23_{\pm.37}$ & 67.78 & $59.49_{\pm.74}$ & 49.94 \\
\bottomrule
\end{tabular}%
}
\end{table}

\paragraph{Layer depth and training objective in MAE.}
The two MAE rows differ only in the block from which patch features are taken. Using the middle block rather
than the last raises noc by $+5.70$ on CGQA and $+10.00$ on COBJ. This difference is large enough to confound
comparisons between training objectives at a common depth, so we select layers by purity first. Depth alone
does not close the gap to DINOv2: even with features from its best layer, MAE remains $8.13$\,pp below DINOv2
in CGQA noc and $16.04$\,pp below in COBJ noc.

\paragraph{The difference comes from the training objective, not from supervision or scale.}
The supervised ViT-H/14 of Appendix~\ref{app:supervised_vit} matches MAE ViT-H/14 in architecture, patch size,
parameter count and attention maps. The two models differ only in training objective and in the block from
which patch features are taken. We compare both against the same DINOv2 reference. On CGQA noc, the supervised
model is $10.71$\,pp below DINOv2, and MAE with features from its last block is $13.83$\,pp below DINOv2. At
fixed architecture and scale, self-supervision by itself therefore does not account for the difference in
noc. The supervised model and MAE do differ on sys, where they are $-14.10$ and $-7.30$\,pp from DINOv2
respectively; on sys, the masked objective retains an advantage. EVA02 provides the complementary case: with
$86$\,M parameters, seven times fewer than MAE ViT-H/14, its noc is within $3.15$\,pp of DINOv2 on CGQA and
$2.11$\,pp on COBJ. Masking itself is not the defect, nor is greater capacity the remedy. What separates EVA02
from MAE is the target of masked prediction: CLIP features for EVA02, and pixels for MAE. A pixel target does
not force semantic content into the patch tokens, whereas a
feature target does.

\paragraph{SAM localizes objects but does not supply semantic content.}
SAM supports the same conclusion from the opposite direction. It is trained on $1.1$\,B masks and is strong at
segmentation, yet its purity never exceeds $0.488$ and its CGQA noc is about $28$\,pp below DINOv2, making it
the weakest patch source we measure. SAM's purity barely changes with depth ($0.4879$ at block 11 versus
$0.4872$ at block 8), as does its accuracy: moving from block $11$ to block $8$ lowers noc by only $0.23$ on
CGQA and $1.56$ on COBJ. Mask supervision teaches a model where object boundaries lie, but not what those
boundaries enclose. Slots built on SAM's patch tokens therefore carry little information for the Chamfer
branch to match.

\paragraph{Summary.}
Even at its best layer, every alternative patch source we measure is $3$ to $28$\,pp below DINOv2 on CGQA noc.
COCO purity, measured before any CFST episode is run, orders the patch sources exactly as each of the four
accuracy columns does. What COMPOSE requires of its backbone is therefore not self-supervision as a category,
nor scale, nor dense supervision, but a pretraining objective that deposits semantic content in the patch
tokens and a block where that content survives. A purity measurement can check this prerequisite on a new
backbone without relying on its training recipe.

\section{Domain-Incremental Shift: DN4IL}
\label{sec:domain_shift}

In both CFST and standard CIL, the \emph{classes} change between sessions. A frozen backbone is most exposed
when the \emph{domain} changes instead, because nothing in the model can adapt to the statistics of images
from a new domain. We therefore run COMPOSE on DN4IL~\citep{duca}, the Domain-IL split of DomainNet released
with DUCA. DN4IL has 100 fixed classes and six domain sessions, the last of which is \texttt{quickdraw}. The
backbone stays frozen, the router is trained, and one feature is stored per class per domain. Results are
averaged over three seeds.

\begin{table}[t]
\caption{\textbf{DN4IL, six domain sessions, accuracy after the final session} (\%, 3 seeds).
DUCA numbers are its published Table~10 at $|B| = 500$. Not like for like: DUCA trains a ResNet18 from
scratch with a reservoir buffer of $500$, COMPOSE keeps DINOv2 frozen with an accumulating store of $600$.
Backward transfer is the change in a domain's accuracy between the session that introduced it and the final
session, averaged over the five past domains; positive means old domains improved.}
\label{tab:dn4il}
\centering
\footnotesize
\setlength{\tabcolsep}{4pt}
\begin{tabular}{lccccccc}
\toprule
& real & clipart & infograph & painting & sketch & quickdraw & mean \\
\midrule
COMPOSE (ours)        & \textbf{85.53} & \textbf{81.40} & \textbf{52.58} & \textbf{76.32} & \textbf{78.29} & 43.11 & $\mathbf{69.54}_{\pm0.84}$ \\
DUCA~\citep{duca}     & 54.77 & 60.37 & 19.35 & 44.50 & 56.34 & \textbf{60.61} & 49.32 \\
\midrule
Backward transfer     & $+10.47$ & $+2.58$ & \multicolumn{3}{c}{$-1.32$ to $-2.38$} & n/a & $+1.46$ \\
\bottomrule
\end{tabular}
\end{table}

\textbf{COMPOSE reaches a mean accuracy of $69.54 \pm 0.84$, against DUCA's published $49.32$, and is ahead of
DUCA on five of the six domains.} Backward transfer remains positive, at $+1.46$\,pp on average. Positive
backward transfer occurs on the earliest domains, which have been in the feature bank longest. Accuracy on
\texttt{real} rises by $10.47$\,pp, and accuracy on \texttt{clipart} rises by
$2.58$. Accuracy on each of the remaining three past domains falls by between $1.32$ and $2.38$\,pp, so no
domain loses more than $2.38$\,pp. The frozen backbone explains this behaviour: stored features never become
outdated, the feature bank only grows, and the router continues improving decision boundaries for previously
seen domains.

\textbf{COMPOSE trails DUCA on only one domain, \texttt{quickdraw}, which accounts for the entire deficit.}
\texttt{quickdraw} consists of 20-second doodles with no texture. Frozen patch features learned from natural
images have nothing that carries over to these doodles. The baselines win the \texttt{quickdraw} column by
retraining the backbone, which COMPOSE never does. We tested a token adapter as a possible remedy, but it did
not help. The adapter is a 2-layer MLP with $65{,}920$ parameters, applied to each patch before the
competitive softmax of Eq.~\eqref{eq:slot_attn}; it is zero-initialised and updated every session with the
holistic objective. Without the adapter, the mean accuracy is
$69.47 \pm 0.21$ (a separate run of the Table~\ref{tab:dn4il} configuration). With the adapter and a frozen
feature bank, the mean accuracy is $68.76 \pm 0.36$. With the adapter and a feature bank re-encoded at each
session, the mean accuracy is $68.84 \pm 0.49$. The differences from the arm without the adapter are
$-0.71 \pm 0.16$ and $-0.64 \pm 0.29$ respectively, and both differences are negative in all three seeds.
Re-encoding the feature bank gives no improvement ($+0.07 \pm 0.21$). Accuracy on \texttt{quickdraw} falls the
most ($-1.80$\,pp). The adapter does learn, in the sense that its weights grow every session. What the adapter
cannot do is create patch content that the frozen encoder never produced.

\subsection{Protocol, Comparability, and Adapter Arms}
\label{app:dn4il}

\paragraph{Protocol.} DN4IL~\citep{duca} is the Domain-IL split of DomainNet released with DUCA. DN4IL presents
100 fixed classes over six domain sessions, the last of which is \texttt{quickdraw}. Only the image domain
changes between sessions. The backbone and slot attention stay frozen, the router is trained, and one feature
per class per domain yields an accumulating store of $600$ entries. Results are averaged over three seeds.
Backward transfer is the change in a domain's accuracy from the session that introduced it to the final
session, averaged over the five earlier domains; the last domain is excluded because its backward transfer is
zero by construction.

\paragraph{Comparability.} The comparison with DUCA in Table~\ref{tab:dn4il} is not like for like, and we do
not present it as one. DUCA trains a ResNet18 from scratch with a reservoir buffer of $500$. COMPOSE keeps a
pretrained DINOv2 frozen and stores $600$ features. The comparison establishes that a frozen object-centric
learner is viable under domain shift. It does not establish that the two training regimes were matched.

\paragraph{Token adapter (negative result).} The adapter is a 2-layer MLP applied to each patch token before
the competitive softmax of Eq.~\eqref{eq:slot_attn}. Zero-initialisation makes session~0 begin exactly at the
unadapted model. With $65{,}920$ parameters, the adapter is updated each session using the router's holistic
objective and sits after the encoder's output transform, immediately before slot attention's competitive
softmax. It therefore operates in the space that feeds slot attention, not in the space of $\tilde{\phi}_k$.

\begin{table}[ht]
\caption{\textbf{Token adapter on DN4IL} (mean accuracy after the final session, \%, 3 seeds).
A0 has no adapter; A1 adds the adapter with the feature bank frozen; A2 adds the adapter and re-encodes the
bank at every session. A0 is a separate run of the COMPOSE configuration of Table~\ref{tab:dn4il}
($69.54$ there), repeated here so that A1 and A2 are compared against a reference from the same run.
Differences are computed per seed.}
\label{tab:dn4il_adapter}
\centering
\footnotesize
\begin{tabular}{lccc}
\toprule
Arm & mean accuracy & $\Delta$ vs.\ A0 & sign consistent \\
\midrule
A0, no adapter                & $\mathbf{69.47}_{\pm0.21}$ & --- & --- \\
A1, adapter, bank frozen      & $68.76_{\pm0.36}$ & $-0.71_{\pm0.16}$ & 3/3 seeds \\
A2, adapter, bank re-encoded  & $68.84_{\pm0.49}$ & $-0.64_{\pm0.29}$ & 3/3 seeds \\
\bottomrule
\end{tabular}
\end{table}

Both adapter arms (A1 and A2) underperform the no-adapter arm (A0) in all three seeds, and re-encoding the
feature bank does not recover the loss ($\text{A2} - \text{A1} = +0.07 \pm 0.21$). Accuracy falls most on
\texttt{quickdraw} ($-1.80$\,pp), the domain the adapter was intended to help. The adapter is learning, in the
sense that its weights grow every session; the result shows that re-aligning the geometry of the patch tokens
cannot manufacture patch content that the frozen encoder never produced.

\section{Setting the Blend Coefficient Without a Sweep}
\label{sec:auto_gamma}

The blend coefficient $\gamma$ in Eq.~\eqref{eq:blend} is COMPOSE's only hand-set value. The main tables use a
single constant, $\gamma = 0.3$, selected on DINOv2/CGQA and shared across all datasets and backbones. Here we
derive a data-dependent alternative, state exactly what it measures, prove why the Chamfer branch cannot
supply the same signal from its output, and compare its accuracy with the fixed constant.

\subsection{The energy identity}
\label{app:gamma_identity}

For one image, write $p_k = z_k \in \mathcal{S}^{D-1}$ for the projected slots (Eq.~\eqref{eq:slot_proj}),
$m = \frac{1}{K}\sum_{k} p_k$ for the mean of these slots, and
\begin{equation}
    R^2 \;=\; \frac{1}{K}\sum_{k=1}^{K} \bigl\lVert p_k - m \bigr\rVert^2
    \label{eq:residual_energy}
\end{equation}
for the mean squared residual that remains after the per-image centering of Eq.~\eqref{eq:centering}.

\begin{proposition}[Energy split]
\label{prop:energy_split}
For any $p_1,\dots,p_K \in \mathcal{S}^{D-1}$,
\begin{equation}
    R^2 + \lVert m \rVert^2 = 1 .
    \label{eq:energy_split}
\end{equation}
\end{proposition}

\begin{proof}
Expanding one term of Eq.~\eqref{eq:residual_energy} and using $\lVert p_k\rVert = 1$,
\[
    \lVert p_k - m\rVert^2 = \lVert p_k \rVert^2 - 2\langle p_k, m\rangle + \lVert m \rVert^2
                           = 1 - 2\langle p_k, m\rangle + \lVert m \rVert^2 .
\]
Averaging over $k$ and using linearity of the inner product together with the definition of $m$,
\[
    R^2 = 1 - 2\Bigl\langle \tfrac{1}{K}\textstyle\sum_k p_k,\; m \Bigr\rangle + \lVert m \rVert^2
        = 1 - 2\langle m, m\rangle + \lVert m \rVert^2
        = 1 - \lVert m \rVert^2 . \qedhere
\]
\end{proof}

Eq.~\eqref{eq:energy_split} holds exactly for each image under only the unit-norm assumption, and requires one
mean and one norm to evaluate. It splits total slot energy into two non-negative parts: residual energy $R^2$,
which survives centering, and $\lVert m \rVert^2$, which centering removes. Since their sum is fixed, we
define
\begin{equation}
    \gamma \;=\; \lVert m \rVert^2 \;=\; 1 - R^2 .
    \label{eq:gamma_rule}
\end{equation}

\subsection{What \texorpdfstring{$\gamma$}{gamma} measures}
\label{app:gamma_meaning}

Proposition~\ref{prop:energy_split} is more than an algebraic curiosity: $\gamma$ is an affine function of
slot-set redundancy.

\begin{corollary}[Range, endpoints, and redundancy reading]
\label{cor:gamma_range}
Let $\bar{s} = \frac{1}{K(K-1)}\sum_{k \neq j} \langle p_k, p_j\rangle$ be the mean pairwise slot cosine.
Then
\begin{equation}
    \gamma \;=\; \frac{1}{K} \;+\; \frac{K-1}{K}\,\bar{s} ,
    \label{eq:gamma_cosine}
\end{equation}
and consequently $\gamma \in [0,1]$ with
\begin{enumerate}[label=(\roman*),leftmargin=2.2em,topsep=2pt,itemsep=1pt]
    \item $\gamma = 1$ if and only if $p_1 = \cdots = p_K$, in which case $R = 0$;
    \item $\gamma = 1/K$ if the slots are pairwise orthogonal;
    \item $\gamma = 0$ if and only if $\sum_k p_k = 0$.
\end{enumerate}
\end{corollary}

\begin{proof}
Expanding the square of the sum,
\[
    K^2 \lVert m \rVert^2 = \Bigl\lVert \textstyle\sum_k p_k \Bigr\rVert^2
      = \sum_{k,j} \langle p_k, p_j \rangle
      = \underbrace{\sum_k \lVert p_k\rVert^2}_{=\,K} + \sum_{k \neq j} \langle p_k, p_j \rangle
      = K + K(K-1)\,\bar{s},
\]
and dividing by $K^2$ gives Eq.~\eqref{eq:gamma_cosine}. For the range, $\gamma = \lVert m\rVert^2 \ge 0$, and
by the triangle inequality $\lVert m \rVert \le \frac{1}{K}\sum_k \lVert p_k \rVert = 1$, so $\gamma \le 1$.
Equality in the triangle inequality for unit vectors forces all $p_k$ to be equal, which gives (i); then
$R^2 = 1 - 1 = 0$ by Eq.~\eqref{eq:energy_split}. Setting $\bar{s} = 0$ in Eq.~\eqref{eq:gamma_cosine} gives
(ii). Finally $\gamma = 0$ iff $\lVert m\rVert = 0$ iff $\sum_k p_k = 0$, which gives (iii).
\end{proof}

Eq.~\eqref{eq:gamma_cosine} gives the key interpretation: \textbf{$\gamma$ is a rescaled mean pairwise slot
similarity.} It increases as the slots become more similar and decreases as they diverge. Its anchors, $1/K$
and $1$, follow from the algebra rather than a calibration set. Under Eq.~\eqref{eq:gamma_rule}, greater
redundancy leaves a part-level matcher with less distinct information, so the blend assigns more weight to the
holistic score.

Appendix~\ref{app:c_measure} defines the inverse quantity $c = 1/\lVert m \rVert^2$, the effective number of
distinct slot directions; hence $c\gamma = 1$. This relation also provides a consistency check because the two
quantities were measured and reported separately. The derived rule gives $\gamma = 0.387 \pm 0.004$ on CGQA
and $0.503 \pm 0.006$ on COBJ.
The reciprocals of these two values are $2.58$ and $1.99$, which are exactly the values of $c$ reported for the
two datasets. On real data, the identity $R^2 + \lVert m \rVert^2 = 1$ itself holds empirically to $1.0000$, as
Proposition~\ref{prop:energy_split} requires.

\subsection{Why the Chamfer branch cannot supply this signal itself}
\label{app:gamma_blindness}

A natural objection is that the rule of Eq.~\eqref{eq:gamma_rule} should be unnecessary. On this view, if the
slots have collapsed, the Chamfer score should already reflect the collapse, and the blend coefficient could
be left unchanged. The Chamfer score does not reflect the collapse, and the reason is structural rather than
empirical.

\begin{proposition}[Chamfer is blind to the residual energy]
\label{prop:chamfer_blind}
The Chamfer score $s_\text{Ch}$ of Eq.~\eqref{eq:chamfer_inf} depends on the slots of an image only through
the \emph{directions} of their centered residuals. In particular, for any $\lambda > 0$, replacing every
$p_k$ by $m + \lambda (p_k - m)$ leaves every $\hat{z}_k$ unchanged and therefore leaves $s_\text{Ch}$
exactly unchanged, while it scales the residual energy to $\lambda^2 R^2$.
\end{proposition}

\begin{proof}
Let $p'_k = m + \lambda(p_k - m)$. The mean of the perturbed slots is
$m' = \frac{1}{K}\sum_k p'_k = m + \lambda\bigl(\frac{1}{K}\sum_k p_k - m\bigr) = m$, so the centered
residual is $p'_k - m' = \lambda (p_k - m)$. Eq.~\eqref{eq:centering} normalises this residual to unit
length, and for $\lambda > 0$,
\[
    \hat{z}'_k = \frac{\lambda (p_k - m)}{\lVert \lambda (p_k - m)\rVert}
               = \frac{p_k - m}{\lVert p_k - m \rVert} = \hat{z}_k .
\]
$s_\text{Ch}$ is a function of cosines between the $\hat{z}_k$ alone, hence unchanged. The residual energy of
the perturbed set is $\frac{1}{K}\sum_k \lVert \lambda(p_k - m)\rVert^2 = \lambda^2 R^2$.
\end{proof}

\textbf{No function of the Chamfer score can recover the magnitudes of the residuals from which it was
computed}, because those magnitudes are divided out before the score is formed. As the slots of an image
converge, the residuals
$p_k - m$ shrink towards zero, and their directions are increasingly determined by whatever is left over.
Renormalisation nevertheless returns full-size unit vectors, and the resulting cosines look no less confident
than they did before the slots converged. A matcher built this way therefore degrades without any sign in its
own scores. The magnitude information that the Chamfer score discards is exactly $R^2$. By
Proposition~\ref{prop:energy_split}, $R^2$ is exactly $1 - \gamma$, and $\gamma$ can be computed at no extra
cost from the same slots.

The energy split also describes the \emph{holistic} branch. Centering subtracts $m$ before matching,
preventing Chamfer from aligning shared background instead of parts; centering is worth $1.08$ noc
(Table~\ref{tab:ablation_reduce}). Thus $m$ is the component that Chamfer discards and the holistic embedding
of Eq.~\eqref{eq:holistic} retains.
On the Chamfer side, this correspondence is exact, because the centered residuals are literally the input of
the Chamfer branch. On the holistic side, the correspondence holds up to the router weights and the order of
normalisation.

\subsection{Behaviour of the rule}
\label{app:gamma_behaviour}

Three properties of the rule follow from the statements above, and we make them explicit here.

\paragraph{The endpoint values of $\gamma$ follow from the algebra and are not chosen.} By
Corollary~\ref{cor:gamma_range}(i), a fully collapsed slot set gives $\gamma = 1$. With $\gamma = 1$, the
Chamfer term in Eq.~\eqref{eq:blend} is multiplied by zero. The Chamfer matcher thus switches \emph{itself} off
in exactly the regime where, by Proposition~\ref{prop:chamfer_blind}, the output of the Chamfer matcher has
stopped being informative. By Corollary~\ref{cor:gamma_range}(ii), a maximally spread slot set gives
$\gamma = 1/K$. This is the smallest value reachable by a non-redundant decomposition, and it corresponds to
the regime in which part matching should receive most of the weight in the blended score. Neither endpoint was
selected to produce a desired number; both endpoints are consequences of Eq.~\eqref{eq:gamma_cosine}.

\paragraph{$\gamma$ is computed once per episode.} $\gamma$ is computed once per episode from the support set,
not once per query image. This keeps the rule inductive, because no information about the query or its label
enters the computation of $\gamma$. The episode is also the level at which we claim the mechanism operates. We
do not claim that each image is individually reweighted.

\paragraph{The derived rule moves $\gamma$ in the direction a sweep would.} For $K = 7$, the reachable range
is $[1/7, 1]$. The rule gives $\gamma=0.387$ on CGQA, close to the constant $0.3$ used elsewhere in this
paper, and $\gamma=0.503$ on COBJ, well above $0.3$. The higher value on COBJ is not incidental. The published
COBJ sweep peaks at $\gamma \approx 0.6$--$0.7$, that is, at a higher
holistic weight than CGQA prefers (Table~\ref{tab:alpha_sensitivity}). The rule therefore moves the
coefficient for COBJ towards the value that the sweep found, \emph{without ever seeing the sweep}. The rule
does so using only a quantity computed from the support slots. A single fixed constant cannot do this by
construction, because a fixed constant is one number shared by both datasets.

\subsection{Results}
\label{app:gamma_results}

\begin{table}[t]
\caption{\textbf{The derived blend matches the fixed constant on CGQA and beats it on COBJ}
(frozen DINOv2 ViT-B/14, 10-way 10-shot, 3 seeds). $\gamma{=}0.3$ is the constant used everywhere else in the
paper, set on DINOv2/CGQA; $\gamma = 1-R^2$ is
read from the support set of each episode. Both rows come from one run; no other component differs between them.}
\label{tab:auto_gamma}
\centering
\footnotesize
\begin{tabular}{lcccc}
\toprule
& CGQA noc & CGQA $H_a$ & COBJ noc & COBJ $H_a$ \\
\midrule
$\gamma = 0.3$ (fixed constant) & $\mathbf{93.43}_{\pm0.18}$ & $92.47_{\pm0.52}$ & $87.93_{\pm0.71}$ & $77.25_{\pm0.55}$ \\
$\gamma = 1 - R^2$ (derived)    & $93.14_{\pm0.20}$ & $\mathbf{92.55}_{\pm0.50}$ & $\mathbf{88.80}_{\pm0.77}$ & $\mathbf{78.05}_{\pm0.29}$ \\
\bottomrule
\end{tabular}
\end{table}

On CGQA, the derived rule stays within $0.29$\,pp of the fixed constant in noc and changes $H_a$ by $+0.08$;
on COBJ, it improves noc by $+0.87$\,pp and $H_a$ by $+0.80$\,pp. This asymmetry between the two
datasets is the one predicted in Section~\ref{app:gamma_behaviour}. On CGQA, the fixed constant is already
near the derived value; on COBJ, it is farther away, and the derived rule improves accuracy. \textbf{The
benefit of the derived rule is therefore not a higher peak accuracy but the
removal of a per-dataset sweep.} Removing the sweep matters in exactly the setting a continual learner
faces, because the data of the next session is not available at the moment the coefficient would have to be
chosen.

\paragraph{Sensitivity of accuracy to the fixed constant.} Near $\gamma=0.3$, the exact value of the constant
matters little. In the sweep of Table~\ref{tab:alpha_sensitivity}, CGQA noc changes by at most $0.26$\,pp over
$\gamma \in [0.2, 0.4]$.

\paragraph{Why the backbone comparison uses the fixed constant.} Table~\ref{tab:backbone_type_reduce} uses
$\gamma = 0.3$ for every patch source. The derived rule is computed from the slots of each backbone, so the
derived $\gamma$ would change together with the patch source. Using the derived rule would therefore confound
the comparison between backbones. Because the constant $0.3$ was set on DINOv2, the constant could favour
DINOv2. However, the noc gaps between DINOv2 and the other backbones ($3.37$--$4.77$\,pp) are about four to
five times the largest gain we observed from moving $\gamma$ away from $0.3$ ($0.93$\,pp, COBJ,
Table~\ref{tab:alpha_sensitivity}).

\paragraph{What the derived rule does not do.} $\gamma$ reports whether the slot assignment of an episode
resolved into distinct directions. $\gamma$ does not report whether the resulting parts are \emph{correct}.
These two properties can differ. Our MAE backbone has by far the lowest COCO slot purity ($0.6217$,
Table~\ref{tab:patch_source}), yet it has nearly the same $c$ as DINOv2, and therefore nearly the same
$\gamma$ (Appendix~\ref{app:c_measure}). The rule guards against matching on a collapsed decomposition, but it
neither replaces a good decomposition nor ranks backbones.

\section{Reading Decomposition Quality Without Masks}
\label{app:c_measure}

Slot purity (Appendix~\ref{app:purity_def}) requires per-pixel instance masks, so it can be computed only on
COCO, not on the benchmarks where clutter actually matters. We use a complementary quantity that requires only
the slots themselves. Let $m$ be the mean projected slot of an
image, as in Eq.~\eqref{eq:energy_split}, and define
\begin{equation}
    c = 1 / \lVert m \rVert^2 .
    \label{eq:c_measure}
\end{equation}
The quantity $c$ equals $1$ when the slots coincide and $K$ when they are mutually orthogonal; when the slots
form $c$ equal, mutually orthogonal groups, it equals exactly $c$. It therefore counts the distinct slot
directions into which an image resolves. On COBJ $c$ is $1.99$ and on CGQA $2.58$, which quantifies the
observation that natural photographs yield a poorer decomposition than synthetic composites.

\paragraph{$c$ tracks how completely the slot competition has resolved.} We correlate the $c$ of each image
with two label-free indicators of decomposition. The first indicator is the assignment entropy over slots,
and the second indicator is the fraction of patches that no slot claims. The correlations are computed
separately inside each backbone and dataset, over 3 backbones $\times$ 5 datasets $\times$ 3 seeds. The mean
$r$ is $-0.47$ for the assignment entropy and $-0.45$ for the fraction of unclaimed patches. The correlation
has the expected sign in 15 of 15 backbone--dataset cells. Each cell holds at least $10^4$ images. The sign
also holds on CIFAR-100, ImageNet-R and ImageNet-A, which are not compositional benchmarks.

\paragraph{Connection to the blend rule.} Because $\gamma = \lVert m \rVert^2 = 1/c$, the coefficient in
Eq.~\eqref{eq:gamma_rule} rises as $c$ falls. When an episode's slots collapse, the blend shifts weight from
Chamfer to the holistic score without hand-tuning, allowing COMPOSE to respond to poor decomposition.

\paragraph{Limits of $c$, stated plainly.} $c$ indicates whether slot assignments resolve into distinct
directions, not whether the resulting parts are \emph{correct}. Our MAE backbone has by far
the lowest COCO purity ($0.6217$, Table~\ref{tab:patch_source}), yet it has nearly the same $c$ as DINOv2.
Where masks are available, $c$ is not a substitute for purity; we therefore do not use it to rank backbones.

%% file: sections/checklist.tex
\section*{NeurIPS Paper Checklist}

\begin{enumerate}

\item {\bf Claims}
    \item[] Question: Do the main claims made in the abstract and introduction accurately reflect the paper's contributions and scope?
    \item[] Answer: \answerYes{} %
    \item[] Justification: The abstract and introduction accurately state all three contributions: (i) replacing GRU states with attention-weighted aggregates for improved slot purity, (ii) diagnosing the compositional trap and escaping it via strict phase separation, and (iii) achieving state-of-the-art noc generalization. All claims are backed by theoretical results (Proposition~\ref{prop:gradient}, Corollary~\ref{cor:assignment_gradient}) and empirical evidence (Tables~\ref{tab:main_reduce}--\ref{tab:ablation_reduce}). The dependency on self-supervised backbone geometry is explicitly stated as a prerequisite, and COBJ limitations are honestly reported in Appendix~A.

\item {\bf Limitations}
    \item[] Question: Does the paper discuss the limitations of the work performed by the authors?
    \item[] Answer: \answerYes{} %
    \item[] Justification: Section~\ref{sec:conclusion} explicitly discusses the hard dependency on self-supervised backbone patch geometry, which degrades under supervised ViTs or significant domain shift, and identifies improving patch-level geometry as the primary direction for future work. Computational efficiency is reported in Appendix~\ref{app:computational_eff}.

\item {\bf Theory assumptions and proofs}
    \item[] Question: For each theoretical result, does the paper provide the full set of assumptions and a complete (and correct) proof?
    \item[] Answer: \answerYes{} %
    \item[] Justification: All theoretical results (Propositions~\ref{prop:gradient}--\ref{prop:var}, Lemmas~\ref{lem:sinkhorn_low_eps}--\ref{lem:ce_rotation}, Corollaries~\ref{cor:assignment_gradient}--\ref{cor:spectral_infeasible}) include full proofs in Appendices~\ref{app:gradient_extended} and~\ref{app:decor_mechanism}. Assumptions are stated at the start of each result (e.g., disjoint router and projection parameters for Proposition~\ref{prop:gradient}, $\Sigma_{\tilde{\phi}} \succ 0$ for Propositions~\ref{prop:cc_feasible}--\ref{prop:vicreg_dynamics}). Proof sketches with intuition are provided in the main text; detailed derivations are in the appendix with cross-references.

    \item {\bf Experimental result reproducibility}
    \item[] Question: Does the paper fully disclose all the information needed to reproduce the main experimental results of the paper to the extent that it affects the main claims and/or conclusions of the paper (regardless of whether the code and data are provided or not)?
    \item[] Answer: \answerYes{} %
    \item[] Justification: Section~\ref{sec:setup_exp} specifies all datasets, protocols, baselines, and seeds (42, 123, 7 for the main results and the ablations). Hyperparameters are reported in Section~\ref{sec:setup_exp} ($K=7$ slots, router hidden dimension $h=64$, $D \times D$ projection head with $D=768$, Adam with lr=$10^{-3}$, $\lambda_d=0.02$, $\gamma=0.3$). Figure~\ref{fig:pipeline} and the step-by-step descriptions in Sections~\ref{subsec:phase1}--\ref{subsec:phase2} fully describe the training and inference pipelines. Sensitivity analyses are in Appendices~\ref{app:redundancy_lamda}--\ref{app:computational_eff}.

\item {\bf Open access to data and code}
    \item[] Question: Does the paper provide open access to the data and code, with sufficient instructions to faithfully reproduce the main experimental results, as described in supplemental material?
    \item[] Answer: \answerYes{} %
    \item[] Justification: The code is publicly available at \url{https://github.com/NLPQuy/COMPOSE} under the MIT licence. Its README explains how to obtain the public datasets (CGQA, COBJ, CIFAR-100, CUB-200, miniImageNet, ImageNet-R, ImageNet-A) and the public DINOSAUR checkpoint, extract the slots, and run the COMPOSE experiments behind the main results and ablations with the seeds used in the paper. Code for the comparison baselines is not included.

\item {\bf Experimental setting/details}
    \item[] Question: Does the paper specify all the training and test details (e.g., data splits, hyperparameters, how they were chosen, type of optimizer) necessary to understand the results?
    \item[] Answer: \answerYes{} %
    \item[] Justification: Section~\ref{sec:setup_exp} reports all datasets, protocols (10-way 10-shot Task-IL, 300 episodes), optimizer (Adam, lr=$10^{-3}$), and hyperparameters ($K=7$, $\lambda_d=0.02$, $\gamma=0.3$, $\kappa=4$). Standard CIL and FSCIL protocols are specified in Appendix~\ref{app:standard_cil}. Sensitivity to all hyperparameters is reported in Appendices~\ref{app:redundancy_lamda}--\ref{app:shot_sensitivity}.

\item {\bf Experiment statistical significance}
    \item[] Question: Does the paper report error bars suitably and correctly defined or other appropriate information about the statistical significance of the experiments?
    \item[] Answer: \answerYes{} %
    \item[] Justification: All main results (Tables~\ref{tab:main_reduce}, \ref{tab:main}, \ref{tab:cil_fscil}, \ref{tab:cil_vit}) report mean $\pm$ 95\% confidence interval over 3 independent seeds (42, 123, 7). Ablation results (Table~\ref{tab:ablation_reduce}) use a single seed (42) as stated in the caption, except GRU/SCA slot representation variants which use 3-seed averages (seeds 42, 0, 1).

\item {\bf Experiments compute resources}
    \item[] Question: For each experiment, does the paper provide sufficient information on the computer resources (type of compute workers, memory, time of execution) needed to reproduce the experiments?
    \item[] Answer: \answerYes{} %
    \item[] Justification: Appendix~\ref{app:computational_eff} reports that inference runs at 66.7\,ms per episode on a single NVIDIA H100 GPU, and full continual training requires approximately 5\,minutes per dataset. The trainable MLP router contains $\approx$640k parameters ($<1\%$ of the DINOv2 ViT-B/14 backbone). This efficiency profile is compared against CoFiMA ($\approx$4\,hours per dataset) in Appendix~\ref{app:computational_eff}.
    
\item {\bf Code of ethics}
    \item[] Question: Does the research conducted in the paper conform, in every respect, with the NeurIPS Code of Ethics \url{https://neurips.cc/public/EthicsGuidelines}?
    \item[] Answer: \answerYes{} %
    \item[] Justification: This work presents foundational research on object-centric continual learning. No human subjects, personal data, or sensitive information are involved. All datasets used are standard computer vision benchmarks. The research conforms with the NeurIPS Code of Ethics in all respects.

\item {\bf Broader impacts}
    \item[] Question: Does the paper discuss both potential positive societal impacts and negative societal impacts of the work performed?
    \item[] Answer: \answerNA{} %
    \item[] Justification: This paper presents foundational algorithmic research on compositional generalization in continual learning. No direct deployment path to a specific application is identified. The method improves sample-efficient visual recognition from limited examples, which could broadly benefit data-scarce domains, but no direct path to negative societal impact is apparent for this class of algorithm.
    
\item {\bf Safeguards}
    \item[] Question: Does the paper describe safeguards that have been put in place for responsible release of data or models that have a high risk for misuse (e.g., pre-trained language models, image generators, or scraped datasets)?
    \item[] Answer: \answerNA{} %
    \item[] Justification: The paper uses publicly available pre-trained vision models (DINOv2, DINOSAUR) as frozen components without fine-tuning the backbone. No new pre-trained generative models, scraped datasets, or high-risk assets are released. The method is a lightweight router ($\approx$640k parameters) on top of an existing frozen backbone.

\item {\bf Licenses for existing assets}
    \item[] Question: Are the creators or original owners of assets (e.g., code, data, models), used in the paper, properly credited and are the license and terms of use explicitly mentioned and properly respected?
    \item[] Answer: \answerYes{} %
    \item[] Justification: All assets are properly cited: the DINOv2 ViT-B/14 backbone \citep{dinov2} (Apache 2.0 license), the publicly released DINOSAUR checkpoint \citep{dinosaur} trained on COCO 2017 \citep{coco}, and all benchmark datasets (CGQA, COBJ from \citet{cfst}, with COBJ built from Objects365 \citep{objects365}; CIFAR-100, CUB-200, miniImageNet, ImageNet-R via their respective original papers). The two COCO photographs shown in Figure~\ref{fig:purity_maps} are CC~BY~2.0 and credited to their authors in the caption. Dataset licenses are documented at \url{paperswithcode.com/datasets}.

\item {\bf New assets}
    \item[] Question: Are new assets introduced in the paper well documented and is the documentation provided alongside the assets?
    \item[] Answer: \answerYes{} %
    \item[] Justification: The only new asset is the COMPOSE code, released at \url{https://github.com/NLPQuy/COMPOSE} under the MIT licence. The repository documents installation, data preparation, usage and which configuration produces which table; the paper introduces no new dataset.

\item {\bf Crowdsourcing and research with human subjects}
    \item[] Question: For crowdsourcing experiments and research with human subjects, does the paper include the full text of instructions given to participants and screenshots, if applicable, as well as details about compensation (if any)? 
    \item[] Answer: \answerNA{} %
    \item[] Justification: This paper does not involve any crowdsourcing or experiments with human subjects. All evaluations use standard pre-existing computer vision benchmarks.

\item {\bf Institutional review board (IRB) approvals or equivalent for research with human subjects}
    \item[] Question: Does the paper describe potential risks incurred by study participants, whether such risks were disclosed to the subjects, and whether Institutional Review Board (IRB) approvals (or an equivalent approval/review based on the requirements of your country or institution) were obtained?
    \item[] Answer: \answerNA{} %
    \item[] Justification: No human subjects are involved in this research. IRB approval is not applicable.

\item {\bf Declaration of LLM usage}
    \item[] Question: Does the paper describe the usage of LLMs if it is an important, original, or non-standard component of the core methods in this research? Note that if the LLM is used only for writing, editing, or formatting purposes and does \emph{not} impact the core methodology, scientific rigor, or originality of the research, declaration is not required.
    \item[] Answer: \answerNA{} %
    \item[] Justification: LLMs are not used as any component of the core methodology. Any use of LLMs was limited to writing assistance only and does not impact the scientific rigor or originality of the research.

\end{enumerate}